\documentclass[11pt]{article}

\usepackage{fullpage}
\usepackage[round]{natbib}
\usepackage{algorithm}
\usepackage[noend]{algorithmic}
\usepackage{amsmath,amsthm,amsfonts,amssymb}
\usepackage{amsmath}
\usepackage{hyperref}
\usepackage{color}
\usepackage{mathrsfs}
\usepackage{enumitem}
\usepackage{bm}
\usepackage{multirow}
\usepackage{booktabs}
\usepackage{makecell}
\usepackage{graphicx}
\usepackage{subfigure}
\usepackage{caption}
\usepackage{thmtools}
\usepackage{thm-restate}
\usepackage{setspace}
\usepackage{makecell}
\definecolor{light-gray}{gray}{0.85}
\usepackage{colortbl}
\usepackage{xcolor}
\usepackage{hhline}

\newcommand{\paren}[1]{{\left( #1 \right)}}
\newcommand{\brac}[1]{{\left[ #1 \right]}}
\newcommand{\set}[1]{{\left\{ #1 \right\}}}

\newcommand{\defeq}{\mathrel{\mathop:}=}

\newcommand{\mat}[1]{\ensuremath{\mathbf{#1}}}

\newcommand{\argmin}{\mathop{\rm argmin}}
\newcommand{\argmax}{\mathop{\rm argmax}}

\newcommand{\trans}{^{\top}}

\newcommand{\E}{\mathbb{E}}
\newcommand{\D}{\mathbb{D}}
\renewcommand{\P}{\mathbb{P}}
\newcommand{\Var}{\text{Var}}

\renewcommand{\Pr}{\mathbb{P}}
\newcommand{\bigO}{\mathcal{O}}

\newcommand{\Phat}{\widehat{\P}}

\newcommand{\rhat}{\widehat{r}}

\newcommand{\Qhat}{\widehat{Q}}
\newcommand{\Vhat}{\widehat{V}}

\newcommand{\pihat}{{\hat{\pi}}}

\newcommand{\Qt}{{\widetilde{Q}}}
\newcommand{\Vt}{{\widetilde{V}}}

\newcommand{\eps}{\epsilon}

\newcommand{\cO}{\mathcal{O}}
\newcommand{\tlO}{\mathcal{\tilde{O}}}

\newcommand{\N}{\mathbb{N}}

\newcommand{\X}{\mat{X}}
\newcommand{\Y}{\mat{Y}}

\newcommand{\cF}{\mathcal{F}}
\newcommand{\cM}{\mathcal{M}}

\newcommand{\cD}{\mathcal{D}}

\newcommand{\cS}{\mathcal{S}}
\newcommand{\cA}{\mathcal{A}}
\newcommand{\cB}{\mathcal{B}}

\newcommand{\minone}[1]{\max\{#1,1\}}

\newcommand{\astar}{\bm{a}^\star}
\newcommand{\bstar}{\bm{b}^\star}

\newcommand{\nw}{n_{\rm wrong}}
\newcommand{\er}[1]{{\rm error}_{#1}}

\renewcommand{\circ}{\diamond}

\usepackage{times}

\newtheorem{theorem}{Theorem}
\newtheorem{lemma}[theorem]{Lemma}

\newtheorem{remark}[theorem]{Remark}
\newtheorem{claim}[theorem]{Claim}
\newtheorem{proposition}[theorem]{Proposition}
\theoremstyle{definition}
\newtheorem{definition}[theorem]{Definition}

\newcommand{\setto}{\leftarrow}
\newcommand{\up}[1]{\overline{#1}}
\newcommand{\low}[1]{\underline{#1}}
\newcommand{\Reg}{{\rm Regret}}
\newcommand{\MG}{{\rm MG}}

\newcommand{\nash}{{\star}}

\newcommand{\CE}{\textsc{CE}}
\newcommand{\CCE}{\textsc{CCE}}
\newcommand{\NASH}{\textsc{Nash}}
\newcommand{\Bonus}{\textsc{Bonus}}
\newcommand{\Eq}{\textsc{Equilibrium}}

\newcommand{\ONASHVI}{Nash-VI}
\newcommand{\OVIZERO}{VI-Zero}

\usepackage{mathtools}

\DeclarePairedDelimiter\floor{\lfloor}{\rfloor}

\hypersetup{
    colorlinks,
    linkcolor={blue!50!black},
    citecolor={blue!50!black},
}
\colorlet{linkequation}{blue}

\begin{document}

\title{\textbf{A Sharp Analysis of Model-based Reinforcement Learning\\ with Self-Play}}

\author{
  Qinghua Liu\thanks{Princeton University. Email: \texttt{qinghual@princeton.edu
}}
  \and
  Tiancheng Yu\thanks{MIT. Email: \texttt{yutc@mit.edu}}
  \and
  Yu Bai\thanks{Salesforce Research. Email: \texttt{yu.bai@salesforce.com}}
  \and
  Chi Jin\thanks{Princeton University. Email: \texttt{chij@princeton.edu}}
}
\date{\today}

\maketitle

\newcommand{\chijin}[1]{\noindent{\textcolor{magenta}{\{{\bf CJ:} \em #1\}}}}
\newcommand{\yubai}[1]{\noindent{\textcolor{blue}{\{{\bf YB:} \em #1\}}}}
\newcommand{\tiancheng}[1]{\noindent{\textcolor{red}{\{{\bf TY:} \em #1\}}}}
\newcommand{\qinghua}[1]{\noindent{\textcolor{cyan}{\{{\bf QL:} \em #1\}}}}

\allowdisplaybreaks

\begin{abstract}
  Model-based algorithms---algorithms that explore the environment through building and utilizing an estimated model---are widely used in reinforcement learning practice and theoretically shown to achieve optimal sample efficiency for single-agent reinforcement learning in Markov Decision Processes (MDPs). However, for multi-agent reinforcement learning in Markov games, the current best known sample complexity for model-based algorithms is rather suboptimal and compares unfavorably against recent model-free approaches.
  
  In this paper, we present a sharp analysis of model-based self-play algorithms for multi-agent Markov games. We design an algorithm \emph{Optimistic Nash Value Iteration} (Nash-VI) for two-player zero-sum Markov games that is able to output an $\epsilon$-approximate Nash policy in $\tlO(H^3SAB/\epsilon^2)$ episodes of game playing, where $S$ is the number of states, $A,B$ are the number of actions for the two players respectively, and $H$ is the horizon length. This significantly improves over the best known model-based guarantee of $\tlO(H^4S^2AB/\epsilon^2)$, and is the first that matches the information-theoretic lower bound $\Omega(H^3S(A+B)/\epsilon^2)$ except for a $\min\set{A,B}$ factor. In addition, our guarantee compares favorably against the best known model-free algorithm if $\min\set{A,B}=o(H^3)$, and outputs a single Markov policy while existing sample-efficient model-free algorithms output a nested mixture of Markov policies that is in general non-Markov and rather inconvenient to store and execute. We further adapt our analysis to designing a provably efficient task-agnostic algorithm for zero-sum Markov games, and designing the first line of provably sample-efficient algorithms for multi-player general-sum Markov games.

%  Sharper analyses for both settings in \cite{bai2020provable}.
\end{abstract}
%!TEX root = main.tex
\section{Introduction} \label{sec:intro}
This paper is concerned with the problem of multi-agent reinforcement learning (multi-agent RL), in which multiple agents learn to make decisions in an unknown environment in order to maximize their (own) cumulative rewards. Multi-agent RL has achieved significant recent success in traditionally hard AI challenges including large-scale strategy games (such as GO)~\citep{silver2016mastering, silver2017mastering}, real-time video games involving team play such as Starcraft and Dota2~\citep{openaidota, vinyals2019grandmaster}, as well as behavior learning in complex social scenarios~\citep{baker2020emergent}. Achieving human-like (or super-human) performance in these games using multi-agent RL typically requires a large number of samples (steps of game playing) due to the necessity of exploration, and how to improve the sample complexity of multi-agent RL has been an important research question.

% However, current understandings of the sample complexity in multi-agent settings have been relatively lacking, especially when compared with the single-player MDP setting~\citep{jaksch2010near,azar2017minimax,jin2018q}.

One prevalent approach towards solving multi-agent RL is \emph{model-based} methods, that is, to use the existing visitation data to build an estimate of the model (i.e. transition dynamics and rewards), run an offline planning algorithm on the estimated model to obtain the policy, and play the policy in the environment. Such a principle underlies some of the earliest single-agent online RL algorithms such as E3~\citep{kearns2002near} and RMax~\citep{brafman2002r}, and is conceptually appealing for multi-agent RL too since the multi-agent structure does not add complexity onto the model estimation part and only requires an appropriate multi-agent planning algorithm (such as value iteration for games~\citep{shapley1953stochastic}) in a black-box fashion. On the other hand, \emph{model-free} methods do not directly build estimates of the model, but instead directly estimate the value functions or action-value (Q) functions of the problem at the optimal/equilibrium policies, and play the greedy policies with respect to the estimated value functions. Model-free algorithms have also been well developed for multi-agent RL such as friend-or-foe Q-Learning~\citep{littman2001friend} and Nash Q-Learning~\citep{hu2003nash}. 

% \yubai{some more beginner / empiricist friendly discussions about what we are going to do. Decision making under uncertainty, exploration vs. exploitation tradeoff, ...}

While both model-based and model-free algorithms have been shown to be provably efficient in multi-agent RL in a recent line of work~\citep{bai2020provable,xie2020learning,bai2020near}, a more precise understanding of the optimal sample complexities within these two types of algorithms (respectively) is still lacking. In the specific setting of two-player zero-sum Markov games, the current best sample complexity for model-based algorithms is achieved by the VI-ULCB (Value Iteration with Upper/Lower Confidence Bounds) algorithm~\citep{bai2020provable,xie2020learning}: 
% A recent line of work makes initial progresses on the sample complexity of both model-based and model-free algorithms in the specific setting of two-player zero-sum Markov games~\citep{bai2020provable,bai2020near,xie2020learning} where the goal is to find approximate Nash equilibrium policies.
% , where the goal is to find approximate Nash equilibrium policies. Roughly speaking, the provable algorithms considered in these work can be divided into two categories. The first category is
In a tabular Markov game with $S$ states, $\set{A,B}$ actions for the two players, and horizon length $H$, % the current best sample complexity of model-based algorithms is achieved by Value Iteration with Upper/Lower Confidence Bounds (VI-ULCB)}, which is
VI-ULCB is able to find an $\epsilon$-approximate Nash equilibrium policy in $\tlO(H^4S^2AB/\epsilon^2)$ episodes of game playing.
However, compared with the information-theoretic lower bound $\Omega(H^3S(A+B)/\epsilon^2)$, this rate has suboptimal dependencies on all of $H$, $S$, and $A,B$. In contrast, the current best sample complexity for \emph{model-free} algorithms is achieved by Nash V-Learning~\citep{bai2020near}, which finds an $\epsilon$-approximate Nash policy in $\tlO(H^6S(A+B)/\epsilon^2)$ episodes. Compared with the lower bound, this is tight except for a ${\rm poly}(H)$ factor, which may seemingly suggest that model-free algorithms could be superior to model-based ones in multi-agent RL. However, such a conclusion would be in stark contrast to the single-agent MDP setting, where it is known that model-based algorithms are able to achieve minimax optimal sample complexities~\citep{jaksch2010near,azar2017minimax}. It naturally arises whether model-free algorithms are indeed superior in multi-agent settings, or whether the existing analyses of model-based algorithms are not tight. This motivates us to ask the following question: 

% However, the policy returned by this algorithm is a nested mixture of multiple Markov policies rather than a simple Markov policy, which makes it more expensive to store and execute in practice. 
% For both categories of algorithms, theoretical understandings are not as comprehensive compared with the single-agent MDP setting. ~\yubai{shorten or divide into two paragraphs?}

% Further, moving apart from the two-player setting, there is little existing work on sample-efficient learning in multi-player general-sum Markov games, a generalization of two-player games that forbids structural simplifications such as the zero-sum structure~\citep{nisan2007algorithmic}.
\begin{quote}
  \centering
  % {\bf Question}:
  \emph{How sample-efficient are model-based algorithms
  in multi-agent RL?}
\end{quote}

\begin{table*}[t]
    \renewcommand{\arraystretch}{1.6} 
    \centering
   \caption{\label{table:rate} Sample complexity (the required number of episodes) for algorithms to find $\epsilon$-approximate Nash equlibrium policies in zero-sum Markov games: VI-explore and VI-UCLB \citep{bai2020provable}, OMVI-SM \citep{xie2020learning}, and Nash Q/V-learning \citep{bai2020near}. The lower bound is proved by \citet{jin2018q,domingues2020episodic}.   }
    % The runtime of VI-ULCB is PPAD-complete while the others are polynomial-time.}
    \scalebox{0.9}{
    \begin{tabular}{|c|c|>{\centering\arraybackslash}m{0.7in}|>{\centering\arraybackslash}m{0.5in}|c|>{\centering\arraybackslash}m{0.9in}|}
      \hline
  & \textbf{Algorithm} & \textbf{Task-Agnostic} & \textbf{$\sqrt{T}$- Regret} & \textbf{Sample Complexity} & \textbf{Output Policy} \\ \hline
     \multirow{5}{*}{\shortstack{Model-based}} & VI-explore  & Yes &  & $\tlO(H^5S^2AB/\epsilon^2)$ & \multirow{5}{*}{\shortstack{a single\\Markov policy}} \\ \hhline{|~----~|}
     & VI-ULCB  & & Yes & $\tlO(H^4 S^2 AB/\epsilon^2)$  &  \\\hhline{|~----~|}
      & OMVI-SM   & & Yes & $\tlO(H^4S^3A^3B^3/\epsilon^2)$ & \\ \hhline{|~----~|}
      & \cellcolor{light-gray} Algorithm 2  & Yes &  & $\quad\tlO(H^4SAB/\epsilon^2)\quad$& \\ \hhline{|~----~|}
      & \cellcolor{light-gray} Algorithm 1  & & Yes &  $\tlO(H^3SAB/\epsilon^2)$ &\\ \hhline{|------|}
     \multirow{2}{*}{\shortstack{Model-free}} & Nash Q-learning   & & &  $\tlO(H^5SAB/\epsilon^2)$ & \multirow{2}{*}{\shortstack{a nested\\mixture of\\Markov policies}} \\ \hhline{|~----~|}
      & Nash V-learning  & & & $\quad\tlO(H^6S(A+B)/\epsilon^2)\quad$& \\ \hline
      & Lower Bound   & -  &  - & $\Omega(H^3S(A+B)/\epsilon^2)$ & - \\ \hline
		    \end{tabular}}
\end{table*}

In this paper, we advance the theoretical understandings of multi-agent RL by presenting a sharp analysis of model-based algorithms on Markov games. Our core contribution is the design of a new model-based algorithm \emph{Optimistic Nash Value Iteration} (Nash-VI) that achieves an almost optimal sample complexity for zero-sum Markov games and improves significantly over existing model-based approaches.
% The core contribution of this paper is the design of a new model-based self-play algorithm---\emph{Optimistic Nash Value Iteration} (Nash-VI)---that achieves PAC sample complexity $\tlO(H^3SAB/\epsilon^2)$. This improves upon the best known sample complexity of model-based algorithms in both $H$, $S$ dependency, and matches the lower bound except for a $\tlO(\min\set{A,B})$ factor. We utilize our analysis techniques to design new provable model-based algorithms for Markov games, including a task-agnostic algorithm for two-player zero-sum games, and the first line of provably sample-efficient algorithms for multi-player games.
We summarize our main contributions as follows. A comparison between our and prior results can be found in Table~\ref{table:rate}.

\begin{itemize}%[wide]
\item We design a new model-based algorithm \emph{Optimistic Nash Value Iteration} (Nash-VI) that provably finds $\epsilon$-approximate Nash equilibria for Markov games in $\tlO(H^3SAB/\epsilon^2)$ episodes of game playing (Section~\ref{sec:VI}). This improves over the best existing model-based algorithm by $O(HS)$ and is the first algorithm that matches the sample complexity lower bound except for a $\tlO(\min\set{A,B})$ factor, showing that model-based algorithms can indeed achieve an almost optimal sample complexity. Further, unlike state-of-the-art model-free algorithms such as Nash V-Learning~\citep{bai2020near}, this algorithm achieves in addition a $\tlO(\sqrt{T})$ regret bound, and outputs a simple Markov policy (instead of a nested mixture of Markov policies as returned by Nash V-Learning).
  
\item We design an alternative algorithm \emph{Optimistic Value Iteration with Zero Reward} (VI-Zero) that is able to perform task-agnostic (reward-free) learning for multiple Markov games sharing the same transition (Section~\ref{sec:reward-free}). For $N>1$ games with the same transition and different (known) rewards, VI-Zero can find $\epsilon$-approximate Nash policy for all games simultaneously in  $\tlO(H^4SAB\log N/\epsilon^2)$ episodes of game playing, which scales logarithmically in the number of games.
  
\item We design the first line of sample-efficient algorithms for \emph{multi-player} general-sum Markov games. In a multi-player game with $M$ players and $A_i$ actions per player, we show that an $\epsilon$ near-optimal policy can be found in $\tlO(H^4S^2\prod_{i\in[M]}A_i/\epsilon^2)$ episodes, where the desired optimality can be either one of Nash equilibrium, correlated equilibrium (CE), or coarse correlated equilibrium (CCE). We achieve this guarantee by either a multi-player version of Nash-VI or a multi-player version of reward-free value iteration (Section~\ref{section:multi-player-short}).
  % \yubai{modify to reflect CE, CCE} 
\end{itemize}

% On the technical end, this paper establishes a novel connection between model-based reinforcement learning and \emph{reward-free exploration}. At a high level, a reward-free exploration algorithm is an RL algorithm that proceesd by first learning a model through reward-free interaction with the environment, and then finding a near-optimal policy through offline planning with the estimated model and reward.~\yubai{rewrite this paragraph.}

% \yubai{let's not have a list of contributions for this paper? since our two main points are fairly straightforwardly described in the above.}

%\tiancheng{a pointer to the related work section in appendix.}

%Due to space limit, we defer a detailed survey of related works 
%to Appendix \ref{sec:related}.

%!TEX root = main.tex
\subsection{Related work} \label{sec:related}

\paragraph{Markov games.}
Markov games (or stochastic games) are proposed in the early 1950s~\citep{shapley1953stochastic}. They are widely used to model multi-agent RL. Learning the Nash equilibria of Markov games has been studied in~\citet{littman1994markov, littman2001friend, hu2003nash,hansen2013strategy, lee2020linear}, where the transition matrix and reward are assumed to be known, or in the asymptotic setting where the number of data goes to infinity.  These results do not directly apply to the non-asymptotic setting where the transition and reward are unknown and only a limited amount of data are available for estimating them.

Another line of works make certain strong reachability assumptions under which sophisticated exploration strategies are not required. A prevalent approach is to assume access to simulators (generative models) that enable the agent to directly sample transition and reward information for any state-action pair. In this setting,~\citet{jia2019feature, sidford2019solving, zhang2020model}~provide non-asymptotic bounds on the number of calls to the simulator for finding an $\epsilon$-approximate Nash equilibrium. \citet{wei2017online} study Markov games under an alternative assumption that no matter what strategy one agent sticks to, the other agent can always reach all states by playing a certain policy.

\paragraph{Non-asymptotic guarantees without reachability assumptions.}
Recent works of \citet{bai2020provable,xie2020learning} provide the first line of non-asymptotic sample complexity guarantees for learning Markov games without reachability assumptions.
%, in which exploration is essential. 
%However, both results suffer from highly suboptimal sample complexity. The results of \citet{xie2020learning} also apply to the linear function approximation setting. 
More recently, \citet{bai2020near} propose two model-free algorithms---Nash Q-Learning and Nash V-Learning with better sample complexity guarantees. In particular, the Nash V-learning algorithm achieves near-optimal dependence on $S$, $A$ and $B$. However, the dependence on $H$ is worse than our results and the output policy is a nested mixture, which is hard to implement. We compare our results with existing non-asymptotic guarantees in Table~\ref{table:rate}.
% where sample complexity itself and some other properties of the algorithms included are given.
% In particular, VI-explore \cite{bai2020provable}, OMVI-SM\cite{xie2020learning} and Algorithm~\ref{algorithm:Nash-VI} also imply regret guarantee during the learning process.   

We remark that the classic R-max algorithm~\citep{brafman2002r} also provides provable guarantees for learning Markov games. However,~\citet{brafman2002r} use a weaker definition of regret \citep[similar to the online setting in  ][]{xie2020learning}, and consequently their result does not imply any sample complexity guarantee for finding Nash equilibrium policies.

%\yubai{I am here.}

\paragraph{Adversarial MDPs.}
    Another way to model the multi-player bahavior is to use \emph{adversarial MDPs}. Most works in this line consider the setting with adversarial reward \citep{zimin2013online, rosenberg2019online, jin2019learning}, where the reward can be manipulated by an adversary arbitrarily and the goal is to compete with the optimal (stationary) policy in hindsight. Learning adversarial MDPs with changing dynamics is computationally hard even under full-information feedback~\citep{yadkori2013online}. Notice these results also do not imply provable algorithms in our setting, because the opponent in Markov games can affect both the reward and the transition.

  \paragraph{Single-agent RL.}
There is a rich literature on reinforcement learning in MDPs \citep[see e.g.,][]{jaksch2010near, osband2014generalization, azar2017minimax, dann2017unifying, strehl2006pac, jin2018q}. MDPs are special cases of Markov games, where only a single agent interacts with a stochastic environment. For the tabular episodic setting with nonstationary dynamics and no simulators, the best sample complexity is $\tilde{\mathcal{O}}(H^3SA/\epsilon^2)$, achieved by the model-based algorithm in \citet{azar2017minimax} and the model-free algorithm in \citet{zhang2020almost}.
% where $S$ is the number of states, $A$ is the number of actions, $H$ is the length of each episode. 
Both of them match the lower bound $\Omega(H^3SA/\epsilon^2)$ \citep{jin2018q}.
%~\citep{jaksch2010near, osband2016lower, jin2018q}.

\paragraph{Reward-free  learning.}
\citet{jin2019reward} study a new paradigm of learning MDPs called reward-free learning, which is also known as the task-agnostic~\citep{zhang2020task} or reward-agnostic setting.  In this setting, the agent goes through a two-stage process. In the exploration phase the agent interacts with the environment without the guidance of any reward information,  and in the planning phase the reward information is revealed and the agent computes a policy based on the transition information collected in the exploration phase and the reward information revealed in the planning phase. 

%The goal is to make the output policy near optimal for any given reward function. A closely related setting is task-agnostic learning \cite{zhang2020task}, where the reward function is determined at the very beginning but not revealed until the planning phase. Notice the algorithms for task-agnostic learning can also be applied to reward-free exploration by taking union bound over all possible reward functions. 

%In Table~\ref{table:rate}, VI-explore \citep{bai2020provable} and Algorithm~\ref{alg:VI-zero} can also be applied to this setting.

\citet{jin2019reward} also propose an algorithm, which first finds a covering policy by maximizing the probability to reach each state separately and then collects data following this policy. \cite{zhang2020task} take a different approach by first running the optimistic Q-learning algorithm \citep{jin2018q} with zero reward to explore the environment, and then utilizing the trajectories collected to compute a policy in an incremental manner. \cite{wang2020reward} follow a similar scheme, but study reward-free exploration in linear-parametrized MDPs.

%!TEX root = main.tex
\section{Preliminaries} \label{sec:prelim}
In this paper, we consider Markov Games \citep[MGs,][]{shapley1953stochastic, littman1994markov}, which are also known as stochastic games in the literature. Markov games are the generalization of standard Markov Decision Processes (MDPs) into the multi-player setting, where each player seeks to maximize her own utility. For simplicity, in this section we describe the important special case of \emph{two-player zero-sum games}, and return to the general formulation in Section \ref{section:multi-player-short}. 

Formally, we consider the tabular episodic version of two-player zero-sum Markov game, which we denote as $\MG(H, \cS, \cA, \cB, \P, r)$. Here $H$ is the number of steps in each episode, $\cS$ is the set of states with $|\cS| \le S$, $(\cA, \cB)$ are the sets of actions of the max-player and the min-player respectively with $|\cA|\le A$ and $|\cB|\le B$, $\P = \{\P_h\}_{h\in[H]}$ is a collection of transition matrices, so that $\P_h ( \cdot | s, a, b) $ gives the distribution of the next state if action pair $(a, b)$ is taken at state $s$ at step $h$, and $r = \{r_h\}_{h\in[H]}$ is a collection of reward functions, where $r_h \colon \cS \times \cA \times \cB \to [0,1]$ is the deterministic reward function at step $h$.\footnote{We assume the rewards in $[0,1]$ for normalization. Our results directly generalize to randomized reward functions, since learning the transition is more difficult than learning the reward.} This reward represents both the gain of the max-player and the loss of the min-player, making the problem a zero-sum Markov game.

In each episode of this MG, we start with a \emph{fixed initial state} $s_1$. At each step $h \in [H]$, both
players observe state $s_h \in \cS$, and pick their own actions $a_h \in \cA$ and $b_h \in \cB$ simultaneously. Then, both players observe the actions of their opponent, receive reward
$r_h(s_h, a_h, b_h)$, and then the environment transitions to the next state
$s_{h+1}\sim\P_h(\cdot | s_h, a_h, b_h)$. The episode ends when
$s_{H+1}$ is reached.

\paragraph{Policy, value function.}
A (Markov) policy $\mu$ of the max-player is a collection of $H$ functions $\{ \mu_h: \cS \rightarrow \Delta_{\cA} \}_{h\in [H]}$, each mapping from a state to a distribution over actions. (Here $\Delta_{\cA}$ is the probability simplex over action set $\cA$.) Similarly, a policy $\nu$ of the min-player is a collection of $H$ functions $\{ \nu_h: \cS \rightarrow \Delta_{\cB} \}_{h\in [H]}$. We use the notation $\mu_h(a|s)$ and $\nu_h(b|s)$ to represent the probability of taking action $a$ or $b$ for state $s$ at step $h$ under Markov policy $\mu$ or $\nu$ respectively.

We use $V^{\mu, \nu}_{h} \colon \cS \to \mathbb{R}$ to denote the value
function at step $h$ under policy $\mu$ and $\nu$, so that
$V^{\mu, \nu}_{h}(s)$ gives the expected cumulative rewards
received under policy $\mu$ and $\nu$, starting from $s$ at step $h$:
\begin{equation} \label{eq:V_value}
\textstyle V^{\mu, \nu}_{h}(s) \defeq \E_{\mu, \nu}\left[\left.\sum_{h' =
        h}^H r_{h'}(s_{h'}, a_{h'}, b_{h'}) \right| s_h = s\right].
\end{equation}
We also define $Q^{\mu, \nu}_h:\cS \times \cA \times \cB \to \mathbb{R}$ to be the $Q$-value function at step $h$ so that
$Q^{\mu, \nu}_{h}(s, a, b)$ gives the cumulative rewards received under policy $\mu$ and $\nu$, starting from
$(s, a, b)$ at step $h$:
\begin{equation} \label{eq:Q_value}
\textstyle Q^{\mu, \nu}_{h}(s, a, b) \defeq  \E_{\mu,
    \nu}\left[\left.\sum_{h' = h}^H r_{h'}(s_{h'},  a_{h'}, b_{h'})
    \right| s_h = s, a_h = a, b_h = b\right].
\end{equation}
For simplicity, we define operator $\P_h$ as
$[\P_h V](s, a, b) \defeq \E_{s' \sim \P_h(\cdot|s, a,
  b)}V(s')$ for any value function $V$. We also use notation $[\D_\pi Q](s) \defeq \E_{(a, b) \sim \pi(\cdot, \cdot|s)} Q(s, a, b)$ for any action-value function $Q$. By definition of value functions, we have the Bellman
equation
\begin{align*}
  Q^{\mu, \nu}_{h}(s, a, b) =
  (r_h + \P_h V^{\mu, \nu}_{h+1})(s, a, b), \qquad   V^{\mu, \nu}_{h}(s)
  =  (\D_{\mu_h\times\nu_h} Q^{\mu, \nu}_h)(s) 
\end{align*}
for all $(s, a, b, h) \in \cS \times \cA \times \cB \times [H]$, and at the $(H+1)^{\text{th}}$ step we have $V^{\mu, \nu}_{H+1}(s) = 0$ for all $s \in \cS$.

% The value for a pair of policy $(\mu, \nu)$ is defined as the expected cumulative rewards starting from the fixed initial state $s_1$. That is,
% $V_1^{\mu, \nu}(s_1) = \E_{\mu, \nu} []$

% \chijin{This is the definition of Markov policy, define general policy (can be function of all history and random coin here) here.}
% A policy of the max-player is defined as $\mu \defeq \big\{ \mu_h: \cS \rightarrow \Delta_{\cA} \big\}_{h\in [H]}$, and a policy  of the min-player is defined as $\nu \defeq \big\{ \nu_h: \cS \rightarrow \Delta_{\cB} \big\}_{h\in [H]}$. Value and action-value functions of policy $\mu, \nu$ are defined as:
% \begin{align*}
%    \\
% . 
% \end{align*}

\paragraph{Best response and Nash equilibrium.}
For any policy of the max-player $\mu$, there exists a \emph{best response} of the min-player, which is a policy
$\nu^\dagger(\mu)$ satisfying $V_h^{\mu, \nu^\dagger(\mu)}(s) = \inf_{\nu} V_h^{\mu, \nu}(s)$ for
any $(s, h) \in \cS \times [H]$. 
We denote $V_h^{\mu, \dagger} \defeq V_h^{\mu, \nu^\dagger(\mu)}$. By symmetry, we
can also define $\mu^\dagger(\nu)$ and $V_h^{\dagger, \nu}$.  
It is further known \citep[cf.~][]{filar2012competitive} that there exist policies $\mu^\star$, $\nu^\star$
that are optimal against the best responses of the opponents, in the sense that
\begin{equation*}
 \textstyle V^{\mu^\star, \dagger}_h(s) = \sup_{\mu}
      V^{\mu, \dagger}_h(s), 
      \qquad  V^{\dagger, \nu^\star}_h(s) = \inf_{\nu}
      V^{\dagger, \nu}_h(s),
  \qquad \textrm{for all}~(s, h).
\end{equation*}
We call these optimal strategies $(\mu^\star,\nu^\star)$ the Nash equilibrium of the Markov game, which satisfies the following minimax equation 
\footnote{The minimax theorem here is different from the one for matrix games, i.e. $\max_\phi\min_\psi \phi\trans A\psi = \min_\psi\max_\phi \phi\trans A\psi$ for any matrix $A$, since here $V^{\mu, \nu}_h(s)$ is in general not bilinear in $\mu, \nu$.}:
\begin{equation*}
\textstyle \sup_{\mu} \inf_{\nu} V^{\mu, \nu}_h(s) = V^{\mu^\star, \nu^\star}_h(s) = \inf_{\nu} \sup_{\mu} V^{\mu, \nu}_h(s).
\end{equation*}
Intuitively, a Nash equilibrium gives a solution in which no player has anything to gain by changing only her own policy. 
% It is also known that, for any $(s, h)$, the minimax
% theorem holds: 
% \begin{equation*}
% \sup_{\mu} \inf_{\nu} V^{\mu, \nu}_h(s) = V^{\mu^\star, \nu^\star}_h(s) = \inf_{\nu} \sup_{\mu} V^{\mu, \nu}_h(s).
% \end{equation*}
% Therefore, the optimal strategies $(\mu^\star,\nu^\star)$ are also
% the Nash Equalibrium for the Markov game. 
We further abbreviate the values of Nash equilibrium $V_h^{\mu^\star, \nu^\star}$ and $Q_h^{\mu^\star, \nu^\star}$ as $V_h^{\nash}$ and $Q_h^{\nash}$.
We refer readers to Appendix \ref{app:bellman} for Bellman optimality equations for (the value functions of) the best responses and the Nash equilibrium.

% \paragraph{General (non-Markovian) policy}
% In certain situations, it is beneficial to consider general, history-dependent policies that are not necessarily Markovian. A \emph{(general) policy} $\mu$ of the max-player is a set of $H$ maps $\mu \defeq \big\{ \mu_h: \R \times (\cS \times \cA \times \cB \times \R)^{h-1}\times \cS \rightarrow \Delta_{\cA} \big\}_{h\in [H]}$, from a random number $z\in\R$ and a history of length $h$---say $(s_1, a_1, b_1, r_1, \cdots, s_h)$, to a distribution over actions in $\cA$. By symmetry, we can also define the (general) policy $\nu$ of the min-player, by replacing the action set $\cA$ in the definition by set $\cB$. The random number $z$ is sampled from some underlying distribution $\mathcal{D}$, but may be shared among all steps $h \in [H]$. 
% % This often appears in a mixture of Markov policies $\{\mu_k\}_{k\in[K]}$, where a random number $k \in [K]$ is sampled in the begining, and then the Markov policy $\mu_k$ is executed. The mixture policy is no longer Markovian.

% For a pair of general policy $(\mu, \nu)$, we can still use the same definitions \eqref{eq:V_value} to define their value $V_1^{\mu, \nu}(s_1)$  at step $1$. We can also define the best response $\nu^\dagger(\mu)$ of a general policy $\mu$ as the minimizing policy so that $V_1^{\mu, \dagger}(s_1) \equiv V_1^{\mu, \nu^\dagger(\mu)}(s_1) = \inf_{\nu} V_h^{\mu, \nu}(s_1)$ at step 1. We remark that the best reponse of a general policy is not necessarily Markovian.

\paragraph{Learning Objective.}

We measure the suboptimality of any pair of general policies $(\hat{\mu}, \hat{\nu})$ using the gap between their performance and the performance of the optimal strategy (i.e., Nash equilibrium) when playing against the best responses respectively:
\begin{equation*}
% \textstyle    \textstyle
\begin{aligned}
	  V^{\dagger, \hat{\nu}}_1(s_1) - V^{\hat{\mu}, \dagger}_1(s_1)  
	  =\brac{V^{\dagger, \hat{\nu}}_1(s_1) - V^{\nash}_1(s_1)} 
	  +  \brac{V^{\nash}_1(s_1) -   V^{\hat{\mu},\dagger}_1(s_1)}.
\end{aligned}
\end{equation*}% This motivates the following definition.

\begin{definition}[$\epsilon$-approximate Nash equilibrium] \label{def:epsilon_Nash} A pair of general policies $(\hat{\mu},\hat{\nu})$ is an \textbf{$\epsilon$-approximate Nash equilibrium}, if $V^{\dagger, \hat{\nu}}_1(s_1) - V^{\hat{\mu}, \dagger}_1(s_1) \le \epsilon$.
\end{definition}
\begin{definition}[Regret]
  Let $(\mu^k$, $\nu^k)$ denote the policies deployed by the algorithm
  in the $k^{\text{th}}$ episode. After a total of $K$ episodes, the regret is defined as
  \begin{equation*}
    \Reg(K) = \sum_{k=1}^K (V^{\dagger, \nu^k}_{1} - V^{\mu^k, \dagger}_{1}) (s_1).
    % \enspace
  \end{equation*}
\end{definition}
One goal of reinforcement learning is to design algorithms for Markov games that can find an $\epsilon$-approximate Nash equilibrium using a number of episodes that is small in its dependency on $S,A,B,H$ as well as $1/\epsilon$ (PAC sample complexity bound). An alternative goal is to design algorithms for Markov games that achieves regret that is sublinear in $K$, and polynomial in $S, A, B, H$ (regret bound). We remark that any sublinear regret algorithm can be directly converted to a polynomial-sample PAC algorithm via the standard online-to-batch conversion \citep[see e.g.,][]{jin2018q}.

\newcommand{\Va}{\mathbb{V}}
\newcommand{\Vahat}{\widehat{\mathbb{V}}}

\newcommand{\Vdstar}{V^{\star}}
\newcommand{\Qdstar}{Q^{\star}}

\newcommand{\VpikA}{V^{\dagger,\nu^k}}
\newcommand{\VpikB}{V^{\mu^k,\dagger}}
\newcommand{\QpikA}{Q^{\dagger,\nu^k}}
\newcommand{\QpikB}{Q^{\mu^k,\dagger}}

\newcommand{\Vpik}{V^{\pi^k}}
\newcommand{\Qpik}{Q^{\pi^k}}

\section{Optimistic Nash Value Iteration} \label{sec:VI}

In this section, we present our main algorithm---Optimistic Nash Value Iteration (\ONASHVI), and provide its theoretical guarantee. % Our algorithm is similar to the VI-ULCB algorithm in~\citep{bai2020provable} with several important modifications, which allows our algorithm to achieve signifcant better regret and sample complexity guarantee than those in~\citep{bai2020provable}.

\subsection{Algorithm description}
\label{sec:alg_des}
We describe our Nash-VI Algorithm \ref{algorithm:Nash-VI}. In each episode, the algorithm can be decomposed into two parts. 
\begin{itemize}
\item Line \ref{line:VI_start}-\ref{line:VI_end} (Optimistic planning from the estimated model): Performs value iteration with bonus using the empirical estimate of the transition $\hat{\P}$, and computes a new (joint) policy $\pi$ which is ``greedy'' with respect to the estimated value functions;
\item Line \ref{line:policy_start}-\ref{line:policy_end} (Play the policy and update the model estimate): Executes the policy $\pi$, collects samples, and updates the estimate of the transition $\hat{\P}$.
\end{itemize}
At a high-level, this two-phase strategy is standard in the majority of model-based RL algorithms, and also underlies provably efficient model-based algorithms such as UCBVI for single-agent (MDP) setting~\citep{azar2017minimax} and VI-ULCB for the two-player Markov game setting~\citep{bai2020provable}. However, VI-ULCB has two undesirable drawbacks: the sample complexity is not tight in any of $H$, $S$, and $A,B$ dependency, and its computational complexity is PPAD-complete (a complexity class conjectured to be computationally hard~\citep{daskalakis2013complexity}).

% Specifically, VI-ULCB adapts this strategy to the setting of Markov games, and achieves a sample complexity of $\tlO(H^4 S^2 AB/\epsilon^2)$. The computational complexity of VI-ULCB is PPAD-complete, thus undesirable.

As we elaborate in the following, our Nash-VI algorithm differs from VI-ULCB in a few important technical aspects, which allows it to significantly improve the sample complexity over VI-ULCB, and ensures that our algorithm terminates in polynomial time.

Before digging into explanations of techniques, we remark that line \ref{line:term_start}-\ref{line:term_end} is only used for computing the output policies. It chooses policy $\pi^{\text{out}}$ to be the policy in the episode with minimum gap $(\up{V}_1 - \low{V}_1)(s_1)$. Our final output policies $(\mu^{\text{out}}, \nu^{\text{out}})$ are simply the \emph{marginal policies} of $\pi^{\text{out}}$. That is, for all $(s, h) \in \cS \times [H]$, $\mu_h^{\text{out}}(\cdot|s) := \sum_{b\in \cB}\pi_h^{\text{out}}(\cdot, b|s)$, and $\nu_h^{\text{out}}(\cdot|s) := \sum_{a\in \cA}\pi_h^{\text{out}}(a, \cdot|s)$.

\subsubsection{Overview of techniques}

\paragraph{Auxiliary bonus $\gamma$.} The major improvement over VI-ULCB~ \citep{bai2020provable} comes from the use of a different style of bonus term $\gamma$ (line \ref{line:gamma}), in addition to the standard bonus $\beta$ (line \ref{line:beta}), in value iteration steps (line \ref{line:Qup}-\ref{line:Qlow}). 
This is also the main technical contribution of our \ONASHVI~algorithm. This auxiliary bonus $\gamma$ is computed by applying the empirical transition matrix $\hat{\P}_h$ to the gap at the next step $\up{V}_{h+1} - \low{V}_{h+1}$, This is very different from standard bonus $\beta$, which is typically designed according to the concentration inequalities. 

The main purpose of these value iteration steps (line \ref{line:Qup}-\ref{line:Qlow}) is to ensure that the estimated values $\up{Q}_h$ and $\low{Q}_h$ are with high probability the upper bound and the lower bound of the $Q$-value of the current policy when facing best responses 
% value $Q_h^{\dagger, \nu}$ and $Q_h^{\mu, \dagger}$, where $(\mu, \nu)$ are the policies induced by policy $\pi$ computed in line \ref{line:greedy} 
(see Lemma \ref{lem:sandwichVpik_Hoeffding} and \ref{lem:sandwichVpik_Bernstein} for more details) \footnote{We remark that the current policy is stochastic. This is different from the single-agent setting, where the algorithm only seeks to provide an upper bound of the value of the optimal policy where the optimal policy is not random. Due to this difference, the techniques of  \citet{azar2017minimax} cannot be directly applied here.}. To do so, prior work \citep{bai2020provable} only adds bonus $\beta$, which needs to be as large as $\tilde{\Theta}(\sqrt{S/t})$. In contrast, the inclusion of auxiliary bonus $\gamma$ in our algorithm allows a much smaller choice for bonus $\beta$---which scales only as $\tlO(\sqrt{1/t})$---while still maintaining valid confidence bounds. This technique alone brings down the sample complexity to $\tlO(H^4 S AB/\epsilon^2)$, removing an entire $S$ factor compared to VI-ULCB. Furthermore, the coefficient in $\gamma$ is only $c/H$ for some absolute constant $c$, which ensures that the introduction of error term $\gamma$ would hurt the overall sample complexity only up to a constant factor.

\begin{algorithm}[t]
   \caption{Optimistic Nash Value Iteration (\ONASHVI)}
   \label{algorithm:Nash-VI}
\begin{algorithmic}[1]
   \STATE {\bfseries Initialize:} for any $(s, a, b, h)$,
   $\up{Q}_{h}(s,a, b)\setto H$, \\
   $\low{Q}_{h}(s,a, b)\setto 0$, $\Delta \setto H$, 
   $N_{h}(s,a, b)\setto 0$.
   \FOR{episode $k=1,\dots,K$}
   % \STATE Receive $s_1$.
   \FOR{step $h=H,H-1,\dots,1$} \label{line:VI_start}
   \FOR{$(s, a, b)\in\cS\times\cA\times \cB$}
   \STATE $t \setto N_{h}(s, a, b)$.
   \IF{$t>0$}
   \STATE $\beta \setto \Bonus(t, \widehat{\Va}_h [(\up{V}_{h+1} + \low{V}_{h+1})/2](s, a, b))$. \label{line:beta}
   \STATE $\gamma \setto (c/H)\widehat{\P}_h (\up{V}_{h+1} - \low{V}_{h+1})(s, a, b)$. \label{line:gamma}
   \STATE $\up{Q}_{h}(s, a, b)\setto \min\{(r_h +
   \widehat{\P}_{h} \up{V}_{h+1})(s, a, b)+  \gamma + \beta, H\}$. \label{line:Qup}
   \STATE $\low{Q}_{h}(s, a, b)\setto \max\{(r_h +
   \widehat{\P}_{h} \low{V}_{h+1})(s, a, b) - \gamma - \beta, 0\}$. \label{line:Qlow}
   \ENDIF
   \ENDFOR
   \FOR{$s \in \cS$}
   \STATE $\pi_h(\cdot, \cdot|s) \setto \CCE (\up{Q}_h(s, \cdot, \cdot), \low{Q}_h(s, \cdot, \cdot))$. \label{line:greedy}
   \STATE $\up{V}_h(s) \leftarrow (\D_{\pi_h}\up{Q}_h)(s)$; $\quad \low{V}_h(s) \leftarrow (\D_{\pi_h} \low{Q}_h)(s)$. \label{line:VI_end}
   \ENDFOR
   \ENDFOR
   \IF{$(\up{V}_1 - \low{V}_1)(s_1) < \Delta$} \label{line:term_start}
   \STATE $\Delta \leftarrow (\up{V}_1 - \low{V}_1)(s_1)$ and $\pi^{\text{out}} \leftarrow \pi$. \label{line:term_end}
   \ENDIF
   \FOR{step $h=1,\dots, H$} \label{line:policy_start}
   \STATE take action $(a_h, b_h) \sim  \pi_h(\cdot, \cdot| s_h)$, observe reward $r_h$ and next state $s_{h+1}$. 
   \label{line:execute-cce}
   \STATE add $1$ to $N_{h}(s_h, a_h, b_h)$ and $N_h(s_h, a_h, b_h, s_{h+1})$.
   % \STATE update $\widehat{\P}_h(\cdot|s_h, a_h, b_h)$ according to \eqref{eq:transition_estimator}. 
   \STATE $\widehat{\P}_h(\cdot|s_h, a_h, b_h)\setto {{N_h(s_h, a_h, b_h, \cdot)}/{N_h(s_h, a_h, b_h)}}$. \label{line:policy_end}
   \ENDFOR
   \ENDFOR
   \STATE {\bfseries Output}  the marginal policies of $\pi^{\text{out}}$: $(\mu^{\text{out}}, \nu^{\text{out}})$.
 \end{algorithmic}
\end{algorithm}

% \paragraph{Tighter Bonus.} To see why the new bonus can achieve better regret rate, let's compare the optimistic Hoeffding bonus in Nash-VI (Algorithm~\ref{algorithm:Hoeffding-Bonus}) $\beta_t = c\sqrt{H^2\iota/t}$ and the optimistic bonus in VI-ULCB (Algorithm 1 in \cite{bai2020provable}) $\beta_t = c\sqrt{SH^2\iota/t} $ where $\iota = \log(SABKH/p)$ is the logarithmic factor and $\widehat{\Va}_hV(s,a,b) :=\Var_{s'\sim \Phat_h(\cdot|s,a,b)}V(s') $ is the empirical variance estimator. It's clear in the new bonus design we don't have the $\sqrt{S}$ factor anymore. The $S$ factor comes from a uniform concentration in $L^1$ norm when upper bounding the best response (Lemma 12 in \cite{bai2020provable}). However, this uniform bound is too restrictive and not necessary. 

% To guarantee the $\up{Q}$ and $\low{Q}$ are still upper and lower bounds of true Q-value under best response, we add a new term $\widehat{\P}^k_h (\up{V}_{h+1} - \low{V}_{h+1})(s, a, b)/H$ in Line~\ref{line:Qup} and Line~\ref{line:Qlow}. In Lemma~\ref{lem:sandwichVstar} and Lemma~\ref{lem:sandwichVpik}, we will prove this will suffice and the new term $\widehat{\P}^k_h (\up{V}_{h+1} - \low{V}_{h+1})(s, a, b)/H$ term is of order $\cO(\sqrt{H\iota/t})$ in average. Therefore, we can remove the $\sqrt{S}$ factor in the bonus. Similar technique has been used in \cite{dann2019policy} to construct policy certificate in an MDP.

\paragraph{Bernstein concentration.} Our Nash-VI allows two choices of the bonus function $\beta=\Bonus(t, \hat{\sigma}^2)$:
\begin{equation} \label{eq:bonus_choice}
\text{Hoeffding type:}~~~c(\sqrt{H^2\iota/t}+ H^2S\iota/t), \qquad\qquad \text{Bernstein type:}~~~c(\sqrt{\hat{\sigma}^2 \iota/t} + H^2S\iota/t ).
\end{equation}
where $\hat{\sigma}^2$ is the estimated variance, $\iota$ is the logarithmic factors and $c$ is absolute constant. 
The $\hat{\Va}$ in line \ref{line:beta} is the empirical variance operator defined as $\widehat{\Va}_h V = \widehat{\P}_h V^2 - (\widehat{\P}_h V)^2$ for any $V \in [0, H]^S$.
The design of both bonuses stem from the Hoeffding and Bernstein concentration inequalities. % Although Bernstein bonus is harder to compute,
Further, the Bernstein bonus uses a sharper concentration, which saves an $H$ factor in sample complexity compared to the Hoeffding bonus
% . This improvement is similar to the one observed in
(similar to the single-agent setting \citep{azar2017minimax}). % The use of the Bernstein bonus further
This further reduces the sample complexity to $\tlO(H^3 S AB/\epsilon^2)$ which matches the lower bound in all $H, S, \epsilon$ factors.

% We also design a bonus based on Bernstein-type concentration in Algorithm~\ref{algorithm:Bernstein-Bonus}. Comparing with the Hoeffding-type bound, the Bernstein-type bonus is of order $\cO(\sqrt{H\iota/t})$ in average (See the proof of Theorem~\ref{thm:Nash-VI}), and thus spare a $\sqrt{H}$ factor. Using Bernstein concentration to improve regret rate is a popular technique in RL literature, first appears in \cite{azar2017minimax} to achieve the minimax dependence on $H$. 

\paragraph{Coarse Correlated Equalibirum (CCE).} The prior algorithm VI-ULCB~\citep{bai2020provable} computes the ``greedy'' policy with respect to the estimated value functions by directly computing the Nash equilibrium for the $Q$-value at each step $h$. However, since the algorithm maintains both the upper confidence bound and lower confidence bound of the $Q$-value, this leads to the requirement to compute the Nash equilibrium for a two-player general-sum matrix game, which is in general PPAD-complete~\citep{daskalakis2013complexity}.

To overcome this computational challenge, we compute a relaxation of the Nash equilibrium---\emph{Coarse Correlated Equalibirum (CCE)}---instead, a technique first introduced by~\citet{xie2020learning} to address reinforcement learning problems in Markov Games. Formally, for any pair of matrices $\up{Q}, \low{Q} \in [0, H]^{A\times B}$, $\textsc{CCE}(\up{Q}, \low{Q})$ returns a distribution $\pi \in \Delta_{\cA \times \cB}$ such that
\begin{equation}\label{problem:CCE}
\E_{(a,b)\sim \pi}\up{Q}(a,b) \ge \max_{a^{\star}}\E_{(a,b)\sim \pi}\up{Q}(a^{\star},b), \quad\qquad \E_{(a,b)\sim \pi}\low{Q}(a,b) \le \min_{b^{\star}}\E_{(a,b)\sim \pi}\low{Q}(a,b^{\star}).
\end{equation}
Intuitively, in a CCE the players
choose their actions in a potentially correlated way such that no one can benefit from  unilateral unconditional deviation.
A CCE always exists, since Nash equilibrium is also a CCE and a Nash equilibrium always exists. Furthermore, a CCE can be computed by linear programming in polynomial time. We remark that different from Nash equilibrium where the policies of each player are independent, the policies given by CCE are in general correlated for each player. 
Therefore, executing such a policy (line \ref{line:execute-cce}) requires the cooperation of two players.

% \begin{algorithm}[h]
%    \caption{Hoeffding-Bonus}
%    \label{algorithm:Hoeffding-Bonus}
% \begin{algorithmic}[1]
%    \STATE {\bfseries Require:} $t,S,H,\iota$.
%    \STATE $\beta \setto C\sqrt{H^2\iota/t}$
%    \RETURN $\beta$
% \end{algorithmic}
% \end{algorithm}

% \begin{algorithm}[h]
%    \caption{Bernstein-Bonus}
%    \label{algorithm:Bernstein-Bonus}
% \begin{algorithmic}[1]
%    \STATE {\bfseries Require:} $t,S,H,\widehat{\Va}_h \up{V}_{h+1}(s,a,b),\iota$.
%    \STATE $\beta \setto C(\sqrt{\iota\widehat{\Va}_h \up{V}_{h+1}(s,a,b)/t} + H^2S\iota/t )$
%    \RETURN $\beta$
% \end{algorithmic}
% \end{algorithm}

\subsection{Theoretical guarantees}

Now we are ready to present the theoretical guarantees for Algorithm \ref{algorithm:Nash-VI}. We let $\pi^k$ denote the policy computed in line \ref{line:greedy} % of Algorithm \ref{algorithm:Nash-VI} at
in the $k^{\text{th}}$ episode, and $\mu^k, \nu^k$ denote the \emph{marginal policy} of $\pi^k$ for each player.

\begin{theorem}[\ONASHVI~with Hoeffding bonus]
   \label{thm:Nash-VI}
For any $p \in (0, 1]$, letting $\iota = \log(SABT /p)$, then with probability at least $1-p$, Algorithm~\ref{algorithm:Nash-VI} with Hoeffding type bonus~\eqref{eq:bonus_choice} (with some absolute $c>0$) achieves:
\begin{itemize}[leftmargin=11mm]
\item $(V_1^{\dagger, \nu^{\text{out}}} - V_1^{\mu^{\text{out}}, \dagger})(s_1) \le \epsilon$, if the number of episodes $K \ge \Omega(H^4SAB\iota/\epsilon^2+H^3S^2AB\iota^2/\epsilon)$.
\item $\Reg(K) = \sum_{k=1}^K (\VpikA_1- \VpikB_1)(s_1) \le  \cO ( \sqrt{H^3SABT\iota} + H^3S^2AB\iota^2 )$.
\end{itemize}
\end{theorem}

Theorem~\ref{thm:Nash-VI} provides both a sample complexity bound and a regret bound for \ONASHVI~to find an $\epsilon$-approximate Nash equilibrium.
% Intuitively, regret bound provides a stronger guarantee, as it requires not only a single policy to be good, but also all policies executed through $K$ episodes to be good on average.
For small $\epsilon \le H/(S\iota)$, the sample complexity scales as $\tlO(H^4SAB/\epsilon^2)$. Similarly, for large $T \ge H^3S^3AB\iota^3$, the regret scales as $\tlO(\sqrt{H^3SABT})$, where $T = KH$ is the total number of steps played within $K$ episodes. Theorem \ref{thm:Nash-VI} is significant in that
% , for the leading order terms, 
it improves the sample complexity of the model-based algorithm in Markov games from $S^2$ to $S$ (and the regret from $S$ to $\sqrt{S}$). This is achieved by adding the new auxiliary bonus $\gamma$ in value iteration steps as explained in Section \ref{sec:alg_des}. The proof of Theorem~\ref{thm:Nash-VI} can be found in Appendix~\ref{appendix:pf_Nash-VI_Hoeffding}.

% , i.e. the summation of suboptimality of the policy executed in each episode measured in terms of the gap of the value when playing against each player's own best responses. For large $T \ge H^3S^3AB\iota^3$, the regret scales as $\cO \paren{\sqrt{H^3SABT\iota}}$, where $T = KH$ is the total number of steps played within $K$ episodes.
% Theorem \ref{thm:Nash-VI} is significant in that, for large $T$, it improves the regret of the model-based algorithm in Markov games from scaling as $S$ to scaling as $\sqrt{S}$. This is achieved by adding the new auxiliary bonus $\gamma$ in value iteration steps as explained in Section \ref{sec:alg_des}. 

Our next theorem states that when using Bernstein bonus instead of Hoeffding bonus as in \eqref{eq:bonus_choice}, the sample complexity of \ONASHVI~algorithm can be further improved by a $H$ factor in the leading order term (and the regret improved by a $\sqrt{H}$ factor).

\begin{theorem}[\ONASHVI~with the Bernstein bonus]
   \label{thm:Nash-VI-B}
For any $p \in (0, 1]$, letting $\iota = \log(SABT /p)$, then with probability at least $1-p$, Algorithm~\ref{algorithm:Nash-VI} with Bernstein type bonus~\eqref{eq:bonus_choice} (with some absolute $c>0$) achieves:
\begin{itemize}[leftmargin=11mm]
\item $(V_1^{\dagger, \nu^{\text{out}}} - V_1^{\mu^{\text{out}}, \dagger})(s_1) \le \epsilon$, if the number of episodes $K \ge \Omega(H^3SAB\iota/\epsilon^2+H^3S^2AB\iota^2/\epsilon)$.
\item $\Reg(K) = \sum_{k=1}^K (\VpikA_1- \VpikB_1)(s_1) \le  \cO (\sqrt{H^2SABT\iota} + H^3S^2AB\iota^2 )$.
\end{itemize}
\end{theorem}
Compared with the information-theoretic sample complexity lower bound $\Omega(H^3S(A+B)\iota/\epsilon^2)$ and regret lower bound $\Omega(\sqrt{H^2S(A+B)T})$~\citep{bai2020provable}, when $\epsilon$ is small, \ONASHVI~with Bernstein bonus achieves the optimal dependency on all of $H, S, \epsilon$ up to logarithmic factors in both the sample complexity and the regret, and the only gap that remains open is a $AB/(A+B)\le \min\set{A,B}$ factor. The proof of Theorem~\ref{thm:Nash-VI-B} can be found in Appendix~\ref{appendix:pf_Nash-VI_Bernstein}.

% \begin{theorem}[regret of \ONASHVI~with Bernstein bonus]\label{thm:Nash-VI-B}
% Under the same setting as Theorem \ref{thm:Nash-VI}, but with Bernstein type bonus \eqref{eq:bonus_choice}, with probability at least $1-p$, Algorithm~\ref{algorithm:Nash-VI} has the following regret bound:
% $\Reg(K) \le \cO \paren{\sqrt{H^2SABT\iota} + H^3S^2AB\iota^2 }$.
% \end{theorem}

\paragraph{Comparison with model-free approaches.}
Different from our model-based approach, a recently proposed model-free algorithm Nash V-Learning~\citep{bai2020near} achieves sample complexity $\tlO(H^6S(A+B)\iota/\epsilon^2)$, which has a tight $(A+B)$ dependency on $A,B$. However, our Nash-VI has the following important advantages over Nash V-Learning: 1. Our sample complexity has a better dependency on horizon $H$; 2. Our algorithm outputs a single pair of Markov policies $(\mu^{\text{out}}, \nu^{\text{out}})$ while their algorithm outputs a generic history-dependent policy that can be only written as a nested mixture of Markov policies; 3. The model-free algorithms in~\citet{bai2020near} cannot be directly modified to obtain a $\sqrt{T}$-regret (so that the exploration policies can be arbitrarily poor), while our model-based algorithm has the $\sqrt{T}$-regret guarantee. 
We comment that although both Nash-VI and Nash V-Learning have polynomial running time, the later enjoys a better computational complexity because Nash-VI requires to solve LPs for computing CCEs in each episode.

\section{Reward-free Learning} \label{sec:reward-free}
\newcommand{\Qexp}{{\bar{Q}^k}}
\newcommand{\Vexp}{{\bar{V}^k}}
\newcommand{\Qpih}{Q^{\pihat_k}}
\newcommand{\Vpih}{V^{\pihat_k}}
\newcommand{\pit}{{\tilde{\pi}}}
\newcommand{\Mcal}{\mathcal{M}}
\newcommand{\Acal}{\mathcal{A}}
\newcommand{\Mfrak}{\mathfrak{M}}
\newcommand{\nuhat}{\hat{\nu}}
\newcommand{\muhat}{\hat{\mu}}

In this section, we modify our model-based algorithm \ONASHVI~for the reward-free exploration setting.
%~\citep{jin2020reward}, which is also known as the task-agnostic~\citep{zhang2020task} or reward-agnostic setting. 
Formally, reward-free learning has two phases: In the exploration phase, the agent collects a dataset of transitions $\cD=\{(s_{k,h},a_{k,h},b_{k,h},s_{k,h+1})\}_{(k,h)\in[K]\times[H]}$
 from a Markov game $\cM$ without the guidance of reward information. 
After the exploration, in the planning phase, 
for each task $i\in[N]$, $\cD$ is augmented with stochastic reward information to become $\cD^i=\{(s_{k,h},a_{k,h},b_{k,h},s_{k,h+1},r_{k,h})\}_{(k,h)\in[K]\times[H]}$, where $r_{k,h}$ is sampled from some unknown reward distribution with expectation equal to $r^i_h( s_{k,h},a_{k,h},b_{k,h})$. Here, $r^i$ denotes the unknown reward function of the $i^{\rm th}$ task. The goal is to compute nearly-optimal policies for  $N$ tasks under $\cM$ simultaneously given the augmented datasets $\{\cD^i\}_{i\in[N]}$. 

% Reward-free learning can be motivated as follows.
There are strong practical motivations for considering the reward-free setting.
First, in applications such as robotics, we face multiple tasks in sequential systems with shared transition dynamics (i.e. the world) but very different rewards. There, we prefer to learn the underlying transition independent of reward information. Second, from the algorithm design perspective, decoupling exploration and planning (i.e. performing exploration without reward information) can be valuable for designing new algorithms in more challenging settings (e.g., with function approximation).

%\subsection{Optimistic Value Iteration with Zero Reward -- VI-Zero}
%\label{sec:reward-free-appendix}

\subsection{Algorithm description}
We now describe our algorithm for reward-free learning in zero-sum Markov games.
\begin{algorithm}[t]
   \caption{Optimistic Value Iteration with Zero Reward (\OVIZERO)}
   \label{alg:VI-zero}
   \begin{algorithmic}[1]
     \REQUIRE Bonus $\beta_t$.
   \STATE {\bfseries Initialize:} for any $(s, a, b, h)$,
   ${\Vt}_{h}(s,a, b)\setto H$, $\Delta \setto H$, $N_{h}(s,a, b)\setto 0$. \\
   % \STATE \qquad\qquad\qquad\qquad\qquad\qquad $\widehat{\P}_h(\cdot|s, a,b)\setto {\rm Uniform}(\cS)$.
   \FOR{episode $k=1,\dots,K$}
   \FOR{step $h=H,H-1,\dots,1$}
   \FOR{$(s, a, b)\in\cS\times\cA\times \cB$}
   \STATE $t \setto N_{h}(s, a, b)$.
   \IF{$t>0$}
   % \STATE $\beta_t \setto \sqrt{ H^2\iota /t}+ {H^2S\iota}/t$
   % and $\widehat{\P}_h^k(\cdot|s, a, b)\setto N_h(s, a, b, \cdot)/t$.
   \STATE $\Qt_{h}(s, a, b)\setto \min\{(\widehat{\P}_{h} {\Vt}_{h+1})(s, a, b) +\beta_t, H\}$. \label{line:zero_UCB}
   \ENDIF
   \ENDFOR
   \FOR{$s \in \cS$}
   \STATE $\pi_h(s) \setto \arg\max_{(a,b)\in \cA\times\cB} \Qt_{h}(s, a, b)$. \label{line:zero_greedy}
   \STATE $\Vt_h(s) \setto  (\D_{\pi_h}\Qt_{h})(s)$.
   \ENDFOR
   \ENDFOR
   \IF{$\Vt_1(s_1) < \Delta$}
   \STATE $\Delta \leftarrow \Vt_1(s_1)$ and $\Phat^{\text{out}} \leftarrow \Phat$.
   \ENDIF
   \FOR{step $h=1,\dots, H$}
   \STATE take action $(a_h, b_h) \sim  \pi_h(\cdot, \cdot| s_h)$, observe next state $s_{h+1}$. 
   \STATE add 1 to $N_{h}(s_h, a_h, b_h)$ and $N_h(s_h, a_h, b_h, s_{h+1})$.
   \STATE $\widehat{\P}_h(\cdot|s_h, a_h, b_h)\setto N_h(s_h, a_h, b_h, \cdot)/N_h(s_h, a_h, b_h)$.
   \ENDFOR
   \ENDFOR
  \STATE {\bfseries Output} $\Phat^{\text{out}}$.
\end{algorithmic}
\end{algorithm}

\paragraph{Exploration phase.} In the first phase of reward-free learning, we deploy algorithm Optimistic Value Iteration with Zero Reward (\OVIZERO, Algorithm~\ref{alg:VI-zero}). This algorithm differs from the reward-aware \ONASHVI~(Algorithm~\ref{algorithm:Nash-VI}) in two important aspects. First, we use zero reward in the exploration phase (Line~\ref{line:zero_UCB}), and only maintains an upper bound of the (reward-free) value function instead of both upper and lower bounds.
Second, our exploration policy is the maximizing (instead of CCE) policy of the value function (Line~\ref{line:zero_greedy}). We remark that the $\Qt_h(s, a, b)$ maintained in the algorithm \ref{alg:VI-zero} is no longer an upper bound for any actual value function (as it has no reward), but rather a measure of uncertainty or suboptimality that the agent may suffer---if she takes action $(a,b)$ at state $s$ and step $h$, and makes decisions by utilizing the empirical estimate $\Phat$ in the remaining steps (see a rigorous version of this statement in Lemma \ref{lem:MG-closeQpiQhat}). Finally, the empirical transition $\Phat$ of the episode that minimizes $\Vt_1(s_1)$ is outputted and passed to the planning phase.

% Our algorithm can be viewed as a zero-reward version of UCB-VI \cite{azar2017minimax}. 
% Specifically, it can be derived by setting the reward function in UCB-VI to be zero and keeping all remaining parts unchanged. 
% As a result, the exploration is purely driven by the bonus function. 

% Note that in Algorithm \ref{algorithm:Reward-free-exploration}, the value functions $\Qt$ is clearly no longer an upper bound for the optimal value function as in \cite{azar2017minimax}. Instead, one should intuitively think of it as a measure of model estimation uncertainty. To be more specific, $\Qt_h(s,a,b)$ quantifies the uncertainty (suboptimality) the agent may suffer if it starts from taking action $(a,b)$ in state $s$ at step $h$ and making decisions by utilizing the empirical estimate $\Phat$ in the remaining steps. 
% A rigorous version of this argument will be given in Theorem \ref{thm:rewardfree-explore}.

\paragraph{Planning phase.} After obtaining the estimate of tranisiton $\Phat$, our planning algorithm is rather simple. 
For the $i^{\rm th}$ task, let $\rhat^i$ be the empirical estimate of $r^i$ computed using the $i^{\rm th}$ augmented dataset $\cD^i$.
Then we compute the Nash equilibrium of the Markov game $\cM(\Phat,\rhat^i)$ with estimated transition $\Phat$ and reward $\rhat^i$. Since both $\Phat$ and $\rhat^i$ are known exactly, this is a pure computation problem without any sampling error and can be efficiently solved by simple planning algorithms such as the vanilla Nash value iteration without optimism (see Appendix \ref{algorithm:Reward-free-planning}
 for more details).
 %\chijin{fill this}

\subsection{Theoretical guarantees}

%Due to space limits, we defer the description of our algorithm Optimistic Value Iteration with Zero Reward (\OVIZERO, Algorithm~\ref{alg:VI-zero}) to Appendix \ref{sec:reward-free-appendix}
%and only state its theoretical guarantees here.

Now we are ready to state our theoretical guarantees for reward-free learning. 
 It claims that the empirical transition $\Phat^{\text{out}}$ outputted by \OVIZERO~is close to the true transition $\P$, in the sense that any Nash equilibrium of the $\cM(\Phat,\rhat^i)$ ($i\in[N]$) is also an approximate Nash equilibrium of the true underlying Markov game $\cM(\P,r^i)$,
 where $\rhat^i$ is the empirical estimate of $r^i$ computed using $\cD^i$.

\begin{theorem}[Sample complexity of \OVIZERO] \label{thm:rewardfree-explore} 
There exists an absolute constant $c$, for any $p \in (0, 1]$, $\epsilon \in (0, H]$, $N \in \N$, if we choose bonus $\beta_t = c(\sqrt{ H^2\iota /t}+ {H^2S\iota}/t)$ with $\iota = \log(NSABT /p)$ and $K \ge c(H^4 SAB\iota/\epsilon^2 + H^3 S^2 AB \iota^2/\epsilon)$, then with probability at least $1-p$, the output $\Phat^{{\rm out}}$ of Algorithm~\ref{alg:VI-zero} satisfies:
For any $N$ fixed reward functions $r^1,\ldots,r^N$, a Nash equilibrium of Markov game $\cM(\Phat^{{\rm out}}, \rhat^i)$ is also an $\epsilon$-approximate Nash equilibrium of the true Markov game $\cM(\P,r^i)$ for all $i\in [N]$.
 \end{theorem}

% \begin{remark}
% Despite Theorem~\ref{thm:rewardfree-explore} requires reward functions $r^1,\ldots,r^N$ to be given, our analysis directly extends to the setting where rewards are also unknown and stochastic, but all $N$ stochastic rewards can be collected simultaneously through bandit feedback in the exploration phase. There, we can simply use the average of the collected rewards of each task as an estimator for each reward function. We can obtain a similar guarantee as Theorem~\ref{thm:rewardfree-explore} with the same sample complexity, up to some absolute constant, since learning transition is more challenging than learning the reward.
% \end{remark}

 Theorem~\ref{thm:rewardfree-explore} shows that, when $\epsilon$ is small, \OVIZERO~only needs $\tlO(H^4 SAB/\epsilon^2)$ samples to learn an estimate of the transition $\Phat^{\text{out}}$, which is accurate enough to learn the approximate Nash equilibrium for any $N$ fixed rewards. The most important advantage of reward-free learning comes from the sample complexity only scaling polylogarithmically with respect to the number of tasks or reward functions $N$. This is in sharp contrast to the reward-aware algorithms (e.g. \ONASHVI), where  the algorithm has to be rerun for each different task, and the total sample complexity must scale linearly in $N$. 
In exchange for this benefit, compared to \ONASHVI, \OVIZERO~loses a factor of $H$ in the leading term of sample complexity since we cannot use Bernstein bonus anymore due to the lack of reward information. \OVIZERO~also does not have a regret guarantee, since again without reward information, the exploration policies are naturally sub-optimal. The proof of Theorem~\ref{thm:rewardfree-explore} can be found in Appendix~\ref{appendix:pf_rewardfree}.

\paragraph{Connections with reward-free learning in MDPs.}
Since MDPs are special cases of Markov games, our algorithm \OVIZERO~directly applies to the single-agent setting, and yields a sample complexity similar to existing results~\citep{zhang2020task,wang2020reward}.
% \cite{zhang2020task} which designs a reward-free algorithm based on optimistic Q-learning, and to \cite{wang2020reward} which designs a reward-free algorithm for linear function approximation. 
However, distinct from existing results which require both the exploration algorithm and  the planning algorithm to be specially designed to work together, our algorithm allows an arbitrary planning algorithm as long as it computes the Nash equilibrium of a Markov game with \emph{known} transition and reward. Therefore, our results completely decouple the exploration and the planning.

\paragraph{Lower bound for reward-free learning.}
Finally, we comment that despite the sample complexity in Theorem \ref{thm:rewardfree-explore} scaling as $AB$ instead of $A+B$, our next theorem states that unlike the general reward-aware setting, this $AB$ scaling is unavoidable in the reward-free setting. 
This reveals an intrinsic gap between the reward-free and reward-aware learning: An $A+B$ dependency is only achievable via sampling schemes that are reward-aware. A similar lower bound is also presented in \citet{zhang2020model} for the discounted setting with a different hard instance construction. 
%\chijin{check if this statement is correct.}

% In contrast, recent works studying reward-free setting do not enjoy such benefit.
% For example, \cite{zhang2020task} requires the samples collected in the exploration phase should be used in a specific incremental manner in the planning phase and 
%  \cite{wang2020reward} requires incorporating the exploration bonus into planning.

\begin{theorem}[Lower bound for reward-free learning of Markov games]\label{thm:MG-lowerbound}
There exists an absolute constant $c>0$ such that for any $\eps \in (0, c]$,  there exists a family of Markov games $\mathfrak{M}(\eps)$ satisfying that: for any reward-free algorithm $\mathfrak{A}$ using $K \le cH^2SAB/\eps^2$ episodes, there exists a Markov game $\Mcal\in\mathfrak{M}(\eps)$ such that if we run $\mathfrak{A}$ on $\Mcal$ and output policies $(\hat{\mu}, \hat{\nu})$, then with probability at least $1/4$, we have $(V_1^{\dagger, \hat{\nu}} - V_1^{\hat{\mu}, \dagger})(s_1) \ge \epsilon$.
\end{theorem}
This lower bound shows that the sample complexity in Theorem \ref{thm:rewardfree-explore} is optimal in $S$, $A,B$, and $\epsilon$. The proof of Theorem~\ref{thm:MG-lowerbound} can be found in Appendix~\ref{appendix:pf_lowerbound}.

\section{Multiplayer General-sum Markov Games} 
%\label{sec:multi}
\label{section:multi-player-short}

In this section, we extend both our model-based algorithms (Algorithm \ref{algorithm:Nash-VI} and Algorithm \ref{alg:VI-zero}) to the setting of multiplayer general-sum Markov games, and present corresponding theoretical guarantees.

\subsection{Problem formulation}
A general-sum Markov game (general-sum MG) with $m$ players is a tuple $\MG(H, \cS, \{\cA_i\}_{i=1}^m, \P, \{r_i\}_{i=1}^m)$, where $H$, $\cS$ denote the length of each episode and the state space. Different from the two-player zero-sum setting, we now have $m$ different action spaces, where $\cA_i$ is the action space for the $i^{\text{th}}$ player and $|\cA_i| = A_i$. We let $\bm{a}:=(a_{1},\cdots,a_{m})$ denote the (tuple of) joint actions by all $m$ players. $\P = \{\P_h\}_{h\in[H]}$ is a collection of transition matrices, so that $\P_h ( \cdot | s, \bm{a}) $ gives the distribution of the next state if actions $\bm{a}$ are taken at state $s$ at step $h$, and $r_i = \{r_{h, i}\}_{h\in[H]}$ is a collection of reward functions for the $i^{\text{th}}$ player, so that $r_{h, i}(s, \bm{a})$ gives the reward received by the $i^{\text{th}}$ player if actions $\bm{a}$ are taken at state $s$ at step $h$. 

In this section, we consider three versions of equlibrium for general-sum MGs: Nash equilibrium (NE), correlated equilibrium (CE), and coarse correlated equilibrium (CCE), all being standard solution notions in games~\citep{nisan2007algorithmic}. These three notions coincide on two-player zero-sum games, but are not equivalent to each other on multi-player general-sum games; any one of them could be desired depending on the application at hand. Below we introduce their definitions.

% One reason for considering (coarse) correlated equilibrium beyond Nash equilibrium is that  players may benefit greatly from cooperating with each other. 

\paragraph{(Approximate) Nash equilibrium in general-sum MGs.}
The policy of the $i^{\text{th}}$ player is denoted as $\pi_i \defeq \big\{ \pi_{h,i}: \cS \rightarrow \Delta_{\cA_i} \big\}_{h\in [H]}$. We denote the product policy of all the players as $\pi:=\pi_1 \times \cdots \times \pi_M$, and denote the policy of  all the players except the $i^{\text{th}}$ player as $\pi_{-i}$. We define $V^{\pi}_{h, i}(s)$ as the expected cumulative reward that will be received by the $i^{\text{th}}$ player if starting at state $s$ at step $h$ and all players follow policy $\pi$. For any strategy $\pi_{-i}$, there also exists a \emph{best response} of the $i^{\text{th}}$ player, which is a policy $\mu^\dagger(\pi_{-i})$ satisfying $V_{h, i}^{\mu^\dagger(\pi_{-i}), \pi_{-i}}(s) = \sup_{\pi_i} V_{h, i}^{\pi_i,\pi_{-i}}(s)$ for any $(s, h) \in \cS \times [H]$. We denote $V_{h, i}^{\dagger, \pi_{-i}} \defeq V_{h, i}^{\mu^\dagger(\pi_{-i}), \pi_{-i}}$. The Q-functions of the best response can be defined similarly.

Our first objective is to find an approximate Nash equilibrium of Markov games.
 % of the Markov game, or to achieve low regret which are defined as follows.
\begin{definition}[$\epsilon$-approximate Nash equilibrium in general-sum MGs] \label{def:NE_multiplayer}
A product policy $\pi$ is an \textbf{$\epsilon$-approximate Nash equilibrium} if $\max_{i\in[m]}{( V_{1,i}^{\dag, \pi_{-i}}-V_{1,i}^{\pi } )}( s_{1} ) \le \epsilon$.
\end{definition}
% According to Definition \ref{def:NE_multiplayer}, a multiplayer $\epsilon$-approximate Nash equilibrium 
The above definition requires the suboptimality gap $(V_{1,i}^{\dag, \pi_{-i}}-V_{1,i}^{\pi })(s_1)$ to be less than $\epsilon$ for all player $i$. This is consistent with the two-player case (Definition~\ref{def:epsilon_Nash}) up to a constant of 2, since in the two-player zero-sum setting, we have $V_{1,1}^\pi(s_1) = -V_{1,2}^\pi(s_1)$ for any product policy $\pi = (\mu, \nu)$, and therefore $( V_{1,1}^{\dag, \nu}-V_{1,1}^{\mu, \dag} )( s_{1} ) \le 2\max_{i\in[2]}{( V_{1,i}^{\dag, \pi_{-i}}-V_{1,i}^{\pi } )}( s_{1} ) \le 2( V_{1,1}^{\dag, \nu}-V_{1,1}^{\mu, \dag} )( s_{1} )$.We can similarly define the regret.
\begin{definition}[Nash-regret in general-sum MGs]
  Let $\pi^k$ denote the (product) policy deployed by the algorithm
  in the $k^{\text{th}}$ episode. After a total of $K$ episodes, the regret is defined as
  \begin{equation*}
    \Reg_{\sf Nash}(K) =  \sum_{k=1}^K \max_{i \in [m]}(V_{1,i}^{\dag, \pi_{-i}^k}-V_{1,i}^{\pi^k } )(s_1).
  \end{equation*}
\end{definition}

\paragraph{(Approximate) CCE in general-sum MGs.}
% Let us begin with the definition of a CCE policy of a Markov game.
The coarse correlated equilibrium (CCE) is a relaxed version of Nash equilibrium in which we consider general correlated policies instead of product policies. Let $\Acal = \Acal_1 \times \cdots\times \Acal_m$ denote the joint action space.
\begin{definition}
[CCE in general-sum MGs]\label{defn:CCEpolicy}
	A (correlated) policy $\pi:=\{ \pi_h(s)\in \Delta_{\Acal}: \ (h,s) \in[H]\times \cS\}$ is a \textbf{CCE} if $\max_{i\in[m]} V^{\dagger,\pi_{-i}}_{h,i}(s) \le V^{\pi}_{h,i}(s)$ for all $(s,h)\in\cS\times[H]$.
\end{definition}
Compared with a Nash equilibrium, a CEE is not necessarily a product policy, that is, we may not have $\pi_h(s)\in \Delta_{\Acal_1}\times \cdots\times \Delta_{\Acal_m}$. Similarly, we also define $\epsilon$-approximate CCE and CCE-regret  below.
\begin{definition}[$\epsilon$-approximate CCE in general-sum MGs] \label{def:CCE_multiplayer}
A policy $\pi:=\{ \pi_h(s)\in \Delta_{\Acal}: \ (h,s) \in[H]\times \cS\}$ is an \textbf{$\epsilon$-approximate CCE} if $\max_{i\in[m]}{( V_{1,i}^{\dag, \pi_{-i}}-V_{1,i}^{\pi } )}( s_{1} ) \le \epsilon$.
\end{definition}
\begin{definition}[CCE-regret in general-sum MGs]
  Let policy $\pi^k$ denote the (correlated) policy deployed by the algorithm
  in the $k^{\text{th}}$ episode. After a total of $K$ episodes, the regret is defined as
  \begin{equation*}
    \Reg_{\sf CCE}(K) =  \sum_{k=1}^K \max_{i \in [m]}(V_{1,i}^{\dag, \pi_{-i}^k}-V_{1,i}^{\pi^k } )(s_1).
  \end{equation*}
\end{definition}

\paragraph{(Approximate) CE  in general-sum MGs.}
The correlated equilibrium (CE) is another relaxation of the Nash equilibrium. To define CE, we first introduce the concept of strategy modification:
%\paragraph{Strategy modification.} 
A strategy modification $\phi:=\{ \phi_{h,s}\}_{(h,s)\in[H]\times \cS}$ for player $i$ is a set of $S\times H$ functions from $\Acal_i$ to itself.  Let $\Phi_i$ denote the set of all possible strategy modifications for player $i$.

One can compose a strategy modification $\phi$ with any Markov policy $\pi$ and obtain a  new policy $\phi\circ\pi$ such that when policy $\pi$ chooses to play $\bm{a}:=(a_1,\ldots,a_m)$ at state $s$ and step $h$, policy $\phi\circ\pi$will play $(a_1,\ldots,a_{i-1},\phi_{h,s}(a_i),a_{i+1},\ldots,a_m)$ instead.
\begin{definition}
[CE in general-sum MGs]\label{defn:CEpolicy}
	A policy $\pi:=\{ \pi_h(s)\in \Delta_{\Acal}: \ (h,s) \in[H]\times \cS\}$ is a \textbf{CE} if $\max_{i\in[m]}\max_{\phi\in\Phi_i}  V^{\phi\circ\pi}_{h,i}(s) \le V^{\pi}_{h,i}(s)$ holds for all $(s,h)\in\cS\times[H]$.
\end{definition}
Similarly, we have an approximate version of CE and CE-regret.
\begin{definition}[$\epsilon$-approximate CE in Markov games] \label{def:CE_multiplayer}
A policy $\pi:=\{ \pi_h(s)\in \Delta_{\Acal}: \ (h,s) \in[H]\times \cS\}$ is an \textbf{$\epsilon$-approximate CE} if $\max_{i\in[m]}\max_{\phi\in\Phi_i} (V^{\phi\circ\pi}_{1,i} - V^{\pi}_{1,i})(s_1)\le \epsilon$.
\end{definition}
\begin{definition}[CE-regret in multiplayer Markov games]
  Let policy $\pi^k$ denote the policy deployed by the algorithm
  in the $k^{\text{th}}$ episode. After a total of $K$ episodes, the regret is defined as
  \begin{equation*}
    \Reg_{\sf CE}(K) =  \sum_{k=1}^K \max_{i \in [m]}\max_{\phi\in\Phi_i}(\ V^{\phi\circ\pi^k}_{1,i} - V^{\pi^k}_{1,i})(s_1).
  \end{equation*}
\end{definition}

\paragraph{Relationship between Nash, CE, and CCE}
For general-sum MGs, we have $\set{{\rm Nash}}\subseteq \set{{\rm CE}}\subseteq \set{{\rm CCE}}$, so that they form a nested set of notions of equilibria~\citep{nisan2007algorithmic}. Indeed, one can easily verify that if we restrict the choice of strategy modification $\phi$ to those consisting of only constant functions, i.e., $\phi_{h,s}(a)$ being independent of $a$, Definition \ref{defn:CEpolicy} will reduce to the definition of CCE policy. In addition, any Nash equilibrium is a CE by definition. Finally, since a Nash equilibrium always exists, so does CE and CCE.

% \chijin{Should we comment on "strong regret" vs "weak regret", "weak regret" is not useful here?}

% \begin{remark}
% The regret defined above is termed as "strong regret" in \cite{bai2020provable}. We can also define the "weak regret" as $\max_{\mu,i}\sum_{k=1}^K{\left( V_{1,i}^{\pi _{-i}^{k},\mu_i}-V_{1,i}^{\pi ^k} \right)}\left( s_{1}^{k} \right)$. The strong regret is always larger because it is competing with the best response \emph{in each episode}. In zero-sum game, low weak regret will be enough to imply approximate NE (like in \cite{bai2020near}). However, for general-sum game, low weak regret only implies CCE. To see this more intuitively, consider in the matrix game case ($S=1$ and $H=1$). Then to achieve sublinear weak regret, we can run a no-regret algorithm for each agent, which only involves polynomial sample and computational complexity. However, finding NE for general-sum game is PPAD-hard \cite{daskalakis2013complexity}. Therefore, low weak regret cannot imply NE.
% \end{remark} 

\subsection{Multiplayer optimistic Nash value iteration}

Here we present the Multi-\ONASHVI \  algorithm, which is an extension of Algorithm~\ref{algorithm:Nash-VI} for multi-player general-sum Markov games.

\begin{algorithm}[t]
    \caption{Multiplayer Optimistic Nash Value Iteration (Multi-\ONASHVI)}
    \label{alg:Multi-Nash-VI}
 \begin{algorithmic}[1]
    \STATE {\bfseries Initialize:} for any $(s, \bm{a}, h, i)$,
    $\up{Q}_{h, i}(s, \bm{a})\setto H$, $\low{Q}_{h, i}(s, \bm{a})\setto 0$, $\Delta \setto H$, $N_{h}(s,\bm{a})\setto 0$.
    \FOR{episode $k=1,\dots,K$}
    \FOR{step $h=H,H-1,\dots,1$}
    \FOR{$(s, \bm{a})\in\cS\times \cA_1 \times \cdots \times \cA_m$}
    \STATE $t \setto N_{h}(s, \bm{a})$;
    \IF{$t>0$}
    % \STATE $\beta_t \setto C\sqrt{SH^2\iota/t}$.
    \FOR{player $i=1,2,\dots, m$}
    \STATE $\up{Q}_{h,i}(s, \bm{a})\setto \min\{(r_{h,i} +
    \widehat{\P}_{h} \up{V}_{h+1,i})(s, \bm{a}) + 
    \beta_t, H\}$.
    \STATE $\low{Q}_{h,i}(s, \bm{a})\setto \max\{(r_{h,i} +
    \widehat{\P}_{h} \low{V}_{h+1,i})(s, \bm{a}) - 
    \beta_t, 0\}$.
    \ENDFOR
    \ENDIF
    \ENDFOR
    \FOR{$s \in \cS$}
    \STATE $\pi_h(\cdot|s) \setto \Eq (\up{Q}_{h,1}(s, \cdot),\up{Q}_{h,2}(s, \cdot),\cdots,\up{Q}_{h,M}(s, \cdot))$. \label{line:general-sum}
    \FOR{player $i=1,2,\dots,m$}
    \STATE $\up{V}_{h,i}(s) \leftarrow (\D_{\pi_h}\up{Q}_{h,i})(s)$; \quad $\low{V}_{h,i}(s) \leftarrow (\D_{\pi_h}\low{Q}_{h,i})(s)$.
    \ENDFOR
    \ENDFOR
    
    \ENDFOR
    \IF{$\max_{i\in[m]}(\up{V}_{1, i} - \low{V}_{1, i})(s_1) < \Delta$} \label{line:multi_term_start}
   \STATE $\Delta \leftarrow \max_{i\in[m]}(\up{V}_{1, i} - \low{V}_{1, i})(s_1)$ and $\pi^{\text{out}} \leftarrow \pi$. \label{line:multi_term_end}
   \ENDIF
   \FOR{step $h=1,\dots, H$} \label{line:multi_policy_start}
   \STATE take action $\bm{a}_h \sim  \pi_h(\cdot| s_h)$, observe reward $r_h$ and next state $s_{h+1}$. 
   \STATE add 1 to $N_{h}(s_h, \bm{a}_h)$ and $N_h(s_h, \bm{a}_h, s_{h+1})$.
   % \STATE update $\widehat{\P}_h(\cdot|s_h, a_h, b_h)$ according to \eqref{eq:transition_estimator}. 
   \STATE $\widehat{\P}_h(\cdot|s_h, \bm{a}_h)\setto N_h(s_h, \bm{a}_h, \cdot)/N_h(s_h, \bm{a}_h)$. \label{line:multi_policy_end}
   \ENDFOR
   \ENDFOR
   \STATE {\bfseries Output} $\pi^{\text{out}}$.
 \end{algorithmic}
 \end{algorithm}

 \paragraph{The $\Eq$ subroutine.} % In Algorithm~\ref{alg:Multi-Nash-VI} and \ref{algorithm:Reward-free-planning}, we used a $\CCE$ subroutine to reduce the computational complexity. This works as well for finding Nash equilibrium because CCE in two-player zero-sum games implies Nash. However, this nice proper no longer holds in multi-player general-sum games and different equilibrium notions could be very different. To handle this issue, we take a more flexible approach. 
 Our $\Eq$ subroutine in Line~\ref{line:general-sum} could be taken from either one of the $\{\NASH, \CE, \CCE\}$ subroutines for \emph{one-step} games. When using $\NASH$, we compute the Nash equilibrium of a one-step multi-player game (see, e.g.,~\citet{berg2016exclusion} for an overview of the available algorithms); the worst-case computational complexity of such a subroutine will be PPAD-hard~\citep{daskalakis2013complexity}. When using $\CE$ or $\CCE$, we find CEs or CCEs of the one-step games respectively, which can be solved in polynomial time using linear programming. However, the policies found are not guaranteed to be a product policy. We remark that in Algorithm~\ref{algorithm:Nash-VI} we used the CCE subroutine for finding Nash in two-player zero-sum games, which seemingly contrasts the principle of using the right subroutine for finding the right equilibrium, but nevertheless works as the Nash equilibrium and CCE are equivalent in zero-sum games. 

Now we are ready to present the theoretical guarantees for Algorithm \ref{alg:Multi-Nash-VI}. We let $\pi^k$ denote the policy computed in line \ref{line:general-sum} of Algorithm \ref{alg:Multi-Nash-VI} in the $k^{\text{th}}$ episode.
\begin{theorem}[Multi-\ONASHVI]
   \label{thm:Multi-Nash-VI}
There exists an absolute constant $c$, for any $p \in (0, 1]$, let $\iota = \log(SABT /p)$, then with probability at least $1-p$, 
Algorithm~\ref{alg:Multi-Nash-VI} with bonus $\beta_t = c\sqrt{SH^2\iota/t}$ and $\Eq$ being one of $\{\NASH, \CE, \CCE\}$ satisfies (repsectively):
\begin{itemize}%[leftmargin=11mm]
\item $\pi^{\text{out}}$ is an $\epsilon$-approximate \{\NASH,\CE,\CCE\}, if the number of episodes $K \ge \Omega(H^4S^2 (\prod_{i=1}^m A_i) \iota/\epsilon^2)$.
\item $\Reg_{\{\sf Nash, CE, CCE\}}(K)  \le  \cO (\sqrt{H^3S^2(\prod_{i=1}^m A_i) T \iota} )$.
\end{itemize}
\end{theorem}
In the situation where the \Eq~subroutine is taken as \NASH, Theorem~\ref{thm:Multi-Nash-VI} provides the sample complexity bound of Multi-\ONASHVI~algorithm to find an $\epsilon$-approximate Nash equilibrium and its regret bound. Compared with our earlier result in two-player zero-sum games (Theorem~\ref{thm:Nash-VI}), here the sample complexity scales as $S^2H^4$ instead of $SH^3$. This is because the auxiliary bonus and Bernstein concentration technique do not apply here. Furthermore, the sample complexity is proportional to $\prod_{i=1}^m A_i$, which increases exponentially as the number of players increases.

\paragraph{Runtime of Algorithm~\ref{alg:Multi-Nash-VI}}
We remark that while the Nash guarantee is the strongest among the three guarantees presented in Theorem~\ref{thm:Multi-Nash-VI}, the runtime of Algorithm~\ref{alg:Multi-Nash-VI} in the Nash case is not guaranteed to be polynomial and in the worst case PPAD-hard (due to the hardness of the \NASH~subroutine). In contrast, the CE and CCE guarantees are weaker, but the corresponding algorithms are guaranteed to finish in polynomial time.
% ~\yubai{comment on Nash being strongest but computaitonally hard whereas CE, CCE are computationally efficient.}

\subsection{Multiplayer reward-free learning}

%\chijin{comment on we can give multiple Nash}
%Now we are ready to state our theoretical guarantees for reward-free learning. 

\begin{algorithm}[h]
   \caption{Multiplayer Optimistic Value Iteration with Zero Reward (Multi-VI-Zero)}
   \label{alg:multi-VI-zero}
\begin{algorithmic}[1]
   \STATE {\bfseries Initialize:} for any $(s, \bm{a}, h)$,
   ${\Vt}_{h}(s,\bm{a})\setto H$, $\Delta \setto H$, $N_{h}(s,\bm{a})\setto 0$.
   % \STATE \qquad\qquad\qquad\qquad\qquad\qquad $\widehat{\P}_h(\cdot|s, a,b)\setto {\rm Uniform}(\cS)$.
   \FOR{episode $k=1,\dots,K$}
   \FOR{step $h=H,H-1,\dots,1$}
   \FOR{$(s, \bm{a})\in\cS\times\cA_1 \times \cdots \times \cA_m$}
   \STATE $t \setto N_{h}(s, \bm{a})$.
   \IF{$t>0$}
   % \STATE $\beta_t \setto \sqrt{ H^2\iota /t}+ {H^2S\iota}/t$
   % and $\widehat{\P}_h^k(\cdot|s, \bm{a})\setto N_h(s, \bm{a}, \cdot)/t$.
   \STATE $\Qt_{h}(s, \bm{a})\setto \min\{(\widehat{\P}_{h} {\Vt}_{h+1})(s, \bm{a}) + \beta_t, H\}$. %\label{line:zero_UCB}
   \ENDIF
   \ENDFOR
   \FOR{$s \in \cS$}
   \STATE $\pi_h(s) \setto \arg\max_{\bm{a}\in \cA_1 \times \cdots \times \cA_m} \Qt_{h}(s, \bm{a})$. %\label{line:zero_greedy}
   \STATE $\Vt_h(s) \setto  (\D_{\pi_h}\Qt_{h})(s)$.
   \ENDFOR
   \ENDFOR
   \IF{$\Vt_1(s_1) < \Delta$}
   \STATE $\Delta \leftarrow \Vt_1(s_1)$ and $\Phat^{\text{out}} \leftarrow \Phat$.
   \ENDIF
   \FOR{step $h=1,\dots, H$}
   \STATE take action $\bm{a}_h \sim  \pi_h(\cdot, \cdot| s_h)$, observe next state $s_{h+1}$. 
   \STATE add 1 to $N_{h}(s_h, \bm{a}_h)$ and $N_h(s_h, \bm{a}_h, s_{h+1})$.
   \STATE $\widehat{\P}_h(\cdot|s_h, \bm{a}_h)\setto N_h(s_h, \bm{a}_h, \cdot)/N_h(s_h, \bm{a}_h)$.
   \ENDFOR
   \ENDFOR
  \STATE {\bfseries Output} $\Phat^{\text{out}}$.
\end{algorithmic}
\end{algorithm}

We can also generalize VI-Zero to the multiplayer  setting and obtain Algorithm \ref{alg:multi-VI-zero}, Multi-VI-Zero, 
which is almost the same as VI-Zero except that its exploration bonus $\beta_t$ is larger than that of VI-Zero by a $\sqrt{S}$ factor. 
%Please refer to Appendix \ref{appendix:multi-rewardfree} for  details of Algorithm \ref{alg:multi-VI-zero}.

Similar to Theorem~\ref{thm:rewardfree-explore}, we have the following theoretical guarantee claiming that  
any \{\NASH,\CCE,\CE\} of the $\cM(\Phat,\rhat^i)$ ($i\in[N]$) is also an approximate \{\NASH,\CCE,\CE\} of the true Markov game $\cM(\P,r^i)$,
where $\Phat^{\text{out}}$ is the empirical transition outputted by Algorithm \ref{alg:multi-VI-zero} and $\rhat^i$ is the empirical estimate of $r^i$.

\begin{theorem}[Multi-\OVIZERO] \label{thm:multi-rewardfree-explore} 
There exists an absolute constant $c$, for any $p \in (0, 1]$, $\epsilon \in (0, H]$, $N \in \N$, if we choose bonus $\beta_t = c\sqrt{ H^2S\iota /t}$ with $\iota = \log(NSABT /p)$ and $K \ge c(H^4 S^2 (\prod_{i=1}^m A_i) \iota/\epsilon^2)$, then with probability at least $1-p$, the output $\Phat^{\text{out}}$ of Algorithm~\ref{alg:multi-VI-zero}
 %\chijin{put the algorithm in the appendix}
  has the following property:
for any $N$ fixed reward functions $r^1,\ldots,r^N$, any \{\NASH,\CCE,\CE\} of Markov game $\cM(\Phat^{\text{out}}, \rhat^i)$ is also an $\epsilon$-approximate \{\NASH,\CCE,\CE\} of the true Markov game $\cM(\P,r^i)$ for all $i\in [N]$.
\end{theorem}
The proof of Theorem~\ref{thm:multi-rewardfree-explore} can be found in Appendix~\ref{appendix:multi-rewardfree}.
It is worth mentioning that the empirical Markov game $\cM(\Phat^{\text{out}}, \rhat^i)$ may have multiple \{Nash equilibria,CCEs,CEs\} and Theorem \ref{thm:multi-rewardfree-explore} ensures that all of them are $\epsilon$-approximate \{Nash equilibria,CCEs,CEs\} of the true Markov game.
Also, note that the sample complexity here is quadratic in the number of states
because we are using the exploration bonus $\beta_t = \sqrt{ H^2S\iota /t}$ that is  larger than usual by a $\sqrt{S}$ factor. 

\section{Conclusion} \label{sec:conclu}

In this paper, we provided a sharp analysis of model-based algorithms for Markov games. Our new algorithm Nash-VI can find an $\epsilon$-approximate Nash equilibrium of a zero-sum Markov game in $\tlO(H^3SAB/\epsilon^2)$ episodes of game playing, which almost matches the sample complexity lower bound except for the $AB$ vs. $A+B$ dependency. We also applied our analysis to derive new efficient algorithms for task-agnostic game playing, as well as the first line of multi-player general-sum Markov games. There are a number of compelling future directions to this work. For example, can we achieve $A+B$ instead of $AB$ sample complexity for zero-sum games using model-based approaches (thus closing the gap between lower and upper bounds)? How can we design more efficient algorithms for general-sum games with better sample complexity (e.g., $\cO(S)$ instead of $\cO(S^2)$)? We leave these problems as future work. 

% Discuss $AB \rightarrow A+B$ dependency, and $S^2$ dependency for multiplayer setting. 

\bibliography{ref}
\bibliographystyle{plainnat}

\clearpage

\appendix
%\input{related}
%\input{reward-free-appendix}
%!TEX root = main.tex

\section{Bellman Equations for Markov Games}
\label{app:bellman}

In this section, we present the Bellman equations for different types of values in Markov games.

\paragraph{Fixed policies.} For any pair of Markov policy $(\mu, \nu)$, by definition of their values in \eqref{eq:V_value} \eqref{eq:Q_value}, we have the following Bellman equations:
\begin{align*}
  Q^{\mu, \nu}_{h}(s, a, b) =  (r_h + \P_h V^{\mu, \nu}_{h+1})(s, a, b), \qquad   
  V^{\mu, \nu}_{h}(s)  =  (\D_{\mu_h\times\nu_h} Q^{\mu, \nu}_h)(s) 
\end{align*}
for all $(s, a, b, h) \in \cS \times \cA \times \cB \times [H]$, where $V^{\mu, \nu}_{H+1}(s) = 0$ for all $s \in \cS$.

\paragraph{Best responses.} For any Markov policy $\mu$ of the max-player, by definition, we have the following Bellman equations for values of its best response:
\begin{align*}
Q^{\mu, \dagger}_{h}(s, a, b) = (r_h + \P_h V^{\mu, \dagger}_{h+1})(s, a, b), \qquad
V^{\mu, \dagger}_{h}(s) = \inf_{\nu \in \Delta_{\cB}} (\D_{\mu_h \times \nu} Q^{\mu, \dagger}_h)(s),
\end{align*}
for all $(s, a, b, h) \in \cS \times \cA \times \cB \times [H]$, where $V^{\mu, \dagger}_{H+1}(s) = 0$ for all $s \in \cS$.

Similarly, for any Markov policy $\nu$ of the min-player, we also have the following symmetric version of Bellman equations for values of its best response:
\begin{align*}
Q^{\dagger, \nu}_{h}(s, a, b) =  (r_h + \P_h V^{\dagger, \nu}_{h+1})(s, a, b), \qquad 
V^{\dagger, \nu}_{h}(s) = \sup_{\mu \in \Delta_{\cA}} (\D_{\mu \times \nu_h} Q^{\dagger, \nu}_h)(s).
\end{align*}
for all $(s, a, b, h) \in \cS \times \cA \times \cB \times [H]$, where $V^{\dagger, \nu}_{H+1}(s) = 0$ for all $s \in \cS$.

\paragraph{Nash equilibria.} Finally, by definition of Nash equilibria in Markov games, we have the following Bellman optimality equations:
\begin{align*}
Q^{\nash}_{h}(s, a, b) = &  (r_h + \P_h V^{\nash}_{h+1})(s, a, b) \\
V^{\nash}_{h}(s) =&
  \sup_{\mu \in \Delta_{\cA}}\inf_{\nu \in \Delta_{\cB}} (\D_{\mu \times \nu} Q^{\nash}_h)(s)
  = \inf_{\nu \in \Delta_{\cB}}\sup_{\mu \in \Delta_{\cA}} (\D_{\mu \times \nu} Q^{\nash}_h)(s)
\end{align*}
for all $(s, a, b, h) \in \cS \times \cA \times \cB \times [H]$, where $V^{\nash}_{H+1}(s) = 0$ for all $s \in \cS$.

% \subsection{Q-learning}
% Q-learning \citep{watkins1989learning} is a classical model-free algorithm for learning Markov decision process. It is one of the most popular algorithms in the literature of reinforcement learning. Nash Q-learning is firstly proposed in \cite{hu2003nash}, which adapted the classical Q-learning to solve multi-agent Markov games. Nash Q-learning features two most important steps:
% \begin{enumerate}
%  \item an online incremental update of $Q$-value functions with parameter $\alpha$.
%  \begin{equation*}
%  Q_h(s_h, a_h, b_h) \setto (1-\alpha)Q_h(s_h, a_h, b_h)+ \alpha(r_h+V_{h+1}(s_{h+1}))
%  \end{equation*}
%  where tuple $(s_h, a_h, b_h, r_h, s_{h+1})$ is the observation in the current step.
%  \item a computation of greedy policies per step.
%  \begin{equation*}
%  (\mu_h(\cdot|s), \nu_h(\cdot|s)) \leftarrow \arg \max_{\mu \in \Delta_{\cA}}\min_{\nu \in \Delta_{\cB}} (\D_{\mu \times \nu} Q_h)(s)
%   = \arg \min_{\nu \in \Delta_{\cB}}\max_{\mu \in \Delta_{\cA}} (\D_{\mu \times \nu} Q_h)(s)
%  \end{equation*}
% \end{enumerate}

\section{Properties of Coarse Correlated Equilibrium}
\label{app:CCE}
% \subsection{CCE}
Recall the definition for CCE in our main paper \eqref{problem:CCE}, we restate it here after rescaling. For any pair of matrices $P, Q \in [0, 1]^{n\times m}$, the subroutine $\textsc{CCE}(P, Q)$ returns a distribution $\pi \in \Delta_{n \times m}$ that satisfies:
\begin{align}
\E_{(a, b) \sim \pi} P(a, b) \ge& \max_{a^\star} \E_{(a, b) \sim \pi} P(a^\star, b)  \label{eq:constraints_CCE}\\
\E_{(a, b) \sim \pi} Q(a, b) \le& \min_{b^\star} \E_{(a, b) \sim \pi} Q(a, b^\star) \nonumber
\end{align}
We make three remarks on CCE. First, a CCE always exists since a Nash equilibrium for a general-sum game with payoff matrices $(P, Q)$ is also a CCE defined by $(P, Q)$, and a Nash equilibrium always exists. Second, a CCE can be efficiently computed, since above constraints \eqref{eq:constraints_CCE} for CCE can be rewritten as $n+m$ linear constraints on $\pi \in \Delta_{n \times m}$, which can be efficiently resolved by standard linear programming algorithm. Third, a CCE in general-sum games needs not to be a Nash equilibrium. However, a CCE in zero-sum games is guaranteed to be a Nash equalibrium.
\begin{proposition}\label{prop:CCE}
Let $\pi = \textsc{CCE}(Q, Q)$, and $(\mu, \nu)$ be the marginal distribution over both players' actions induced by $\pi$. Then $(\mu, \nu)$ is a Nash equilibrium for payoff matrix $Q$.
\end{proposition}

\begin{proof}[Proof of Proposition \ref{prop:CCE}]
Let $N^\star$ be the value of Nash equilibrium for $Q$. Since $\pi = \textsc{CCE}(Q, Q)$, by definition, we have:
\begin{align*}
\E_{(a, b) \sim \pi} Q(a, b) \ge& \max_{a^\star} \E_{(a, b) \sim \pi} Q(a^\star, b) = \max_{a^\star} \E_{b \sim \nu} Q(a^\star, b) \ge N^\star\\
\E_{(a, b) \sim \pi} Q(a, b) \le& \min_{b^\star} \E_{(a, b) \sim \pi} Q(a, b^\star) =\min_{b^\star} \E_{a \sim \mu} Q(a, b^\star) \le N^\star
\end{align*}
This gives:
\begin{equation*}
\max_{a^\star} \E_{b \sim \nu} Q(a^\star, b) = \min_{b^\star} \E_{a \sim \mu} Q(a, b^\star) = N^\star
\end{equation*}
which finishes the proof.
\end{proof}
Intuitively, a CCE procedure can be used in Nash Q-learning for finding an approximate Nash equilibrium, because the values of upper confidence and lower confidence ($\up{Q}$ and $\low{Q}$) will be eventually very close, so that the preconditions of Proposition \ref{prop:CCE} becomes approximately satisfied.

% \chijin{discussion on CCE being collaborative}

% \subsection{Nash V-learning}

\section{Proof for Section \ref{sec:VI} -- Optimistic Nash Value Iteration}

\subsection{Proof of Theorem~\ref{thm:Nash-VI}}
\label{appendix:pf_Nash-VI_Hoeffding}

We denote $V^k$, $Q^k$, $\pi^k$, $\mu^k$ and $\nu^k$ \footnote{recall that $(\mu^k_{h},\nu^k_{h})$ are the marginal distributions of $\pi^k_{h}$.} for values and policies at the \emph{beginning} of the $k$-th episode. In particular, $N_h^k(s,a,b)$ is the number we have visited the state-action tuple $(s,a,b)$ at the $h$-th step before the $k$-th episode. $N_h^k(s,a,b,s')$ is defined by the same token. Using this notation, we can further define the empirical transition by $\widehat{\P}^k_h(s'|s, a, b):= N^k_h(s, a, b, s')/N^k_h(s, a, b)$.
If $N^k_h(s, a, b)=0$, we set $\widehat{\P}^k_h(s'|s, a, b)=1/S$.

As a result, the bonus terms can be written as
\begin{equation}
	\beta_h^k(s,a,b) := C\paren{\sqrt{\frac{\iota H^2}{\minone{N_h^k(s,a,b)}}} + \frac{H^2S\iota}{\minone{N_h^k(s,a,b)}}}
\end{equation}
\begin{equation}
	\gamma_h^k(s,a,b) := \frac{C}{H}\widehat{\P}_h (\up{V}^k_{h+1} - \low{V}^k_{h+1})(s, a, b)
\end{equation}
for some large absolute constant $C>0$.
\begin{lemma}\label{event-vi-1}
Let $c_1$  be some large absolute constant.
 Define event $E_0$ to be: for all $h,s,a,b,s'$ and $k\in[K]$, 
\begin{equation*}
\left\{
	\begin{aligned}
		&|[(\Phat^k_h - \Pr_h) V^\star_{h+1}](s,a,b)| \le {c_1}\sqrt{ \frac{H^2\iota}{\max\{N_h^k(s,a,b),1\}}},\\
		& |(\Phat^k_h - \Pr_h)(s'\mid s,a,b)| 
		\le {c_1}
	\paren{\sqrt{ \frac{\min\{\Pr_h(s'\mid s,a,b),\Phat^k_h(s'\mid s,a,b)\}\iota}{\max\{N_h^k(s,a,b),1\}}}+ \frac{\iota}{\minone{N_h^k(s,a,b)}}}.
	\end{aligned}
	\right.
\end{equation*}
We have $\Pr(E_1)\ge 1-p$.
\end{lemma}
\begin{proof}
	The proof is standard and folklore: apply standard concentration inequalities and then take a union bound. For completeness, we provide the proof of the second one here.
	
Consider a fixed $(s,a,b,h)$ tuple. 

Let's consider the following equivalent random process:
(a) before the agent starts, the environment samples 
$\{s^{(1)},s^{(2)},\ldots,s^{(K)}\}$ independently from $\Pr_h(\cdot\mid s,a,b)$; 
(b) during the interaction between the agent and environment, the $i^{\rm th}$ time the agent reaches $(s,a,b,h)$, the environment will make the agent transit to $s^{(i)}$. 
Note that the randomness induced by this interaction procedure is exactly the same as the original one, which means the probability of any event in this context is the same as in the original problem.
 Therefore, it suffices to prove the target concentration inequality in this 'easy' context.
Denote by $\Phat^{(t)}_h(\cdot\mid s,a,b)$ the empirical estimate of $\Pr_h(\cdot\mid s,a,b)$ calculated using $\{s^{(1)},s^{(2)},\ldots,s^{(t)}\}$.
 For a fixed $t$ and $s'$, by applying the  Bernstein inequality and its empirical version, we have with probability at least $1-p/S^2ABT$,
 $$
 |(\Pr_h - \Phat^{(t)}_h)(s'\mid s,a,b)| 
		\le \bigO 
	\paren{\sqrt{ \frac{\min\{\Pr_h(s'\mid s,a,b),\Phat^{(t)}_h(s'\mid s,a,b)\}\iota}{t}}+ \frac{\iota}{t}}.
 $$
 
Now we can take a union bound over all $s,a,b,h,s'$ and $t\in[K]$, and obtain 
that with probability at least $1-p$, for all $s,a,b,h,s'$ and $t\in[K]$,
 $$
 |(\Pr_h - \Phat^{(t)}_h)(s'\mid s,a,b)| 
		\le \bigO 
	\paren{\sqrt{ \frac{\min\{\Pr_h(s'\mid s,a,b),\Phat^{(t)}_h(s'\mid s,a,b)\}\iota}{t}}+ \frac{\iota}{t}}.
 $$
Note that the agent can reach each $(s,a,b,h)$ for at most $K$ times, this directly implies that the third  inequality also holds with probability at least $1-p$.
	\end{proof}

We begin with an auxiliary lemma bounding the lower-order term. 

\begin{lemma}
	\label{lem:lower-order}
	Suppose event $E_0$ holds, then there exists absolute constant $c_2$ such that:
	if function $g(s)$ satisfies $ |g|(s) \le (\up{V}^{k}_{h+1} - \low{V}^{k}_{h+1})(s) $ for all $s$, then
	\begin{align*}
	&|(\Phat_h^k- \P_h)g(s,a,b)| \\
	\le & c_2 \bigg(\frac{1}{H}\min\{\Phat_h^k (\up{V}^{k}_{h+1} - \low{V}^{k}_{h+1})(s,a,b),\P_h (\up{V}^{k}_{h+1} - \low{V}^{k}_{h+1})(s,a,b)\}
	+ \frac{H^2S\iota}{\minone{N_h^k(s,a,b)}}\bigg).
	\end{align*}
\end{lemma}
\begin{proof}
	By triangle inequality,
    \begin{align*}
		|(\Phat_h^k- \P_h)g(s,a,b)|
		\le& \sum_{s'}{|(\Phat_h^k- \P_h)(s'|s,a,b)||g|(s')}\\
    \le& \sum_{s'}{|(\Phat_h^k- \P_h)(s'|s,a,b)|(\up{V}^{k}_{h+1} - \low{V}^{k}_{h+1})(s')}\\
    \overset{\left( i \right)}{\le}& \cO\left(\sum_{s'}{(\sqrt{\frac{\iota \Phat_h^k(s'|s,a,b)}{\minone{N_h^k(s,a,b)}}}+\frac{\iota }{\minone{N_h^k(s,a,b)}})(\up{V}^{k}_{h+1} - \low{V}^{k}_{h+1})(s')}\right)\\
    \overset{\left( ii \right)}{\le}& \cO\left(\sum_{s'}{(\frac{\Phat_h^k(s'|s,a,b) }{H}+\frac{H\iota }{\minone{N_h^k(s,a,b)}})(\up{V}^{k}_{h+1} - \low{V}^{k}_{h+1})(s')}\right)\\
    \le& \cO\left(\frac{\Phat_h^k (\up{V}^{k}_{h+1} - \low{V}^{k}_{h+1})(s,a,b)}{H}
	+ \frac{H^2S\iota}{\minone{N_h^k(s,a,b)}}\right),
    \end{align*}
    where $(i)$ is by the second inequality in event $E_0$ and $(ii)$ is by AM-GM inequality. This proves the empirical version. Similarly, we can show 
    \begin{align*}
		|(\Phat_h^k- \P_h)g(s,a,b)|
	    \le \cO\left(\frac{\Pr_h (\up{V}^{k}_{h+1} - \low{V}^{k}_{h+1})(s,a,b)}{H}
	+ \frac{H^2S\iota}{\minone{N_h^k(s,a,b)}}\right),
    \end{align*}
     Combining the two bounds completes the proof.
\end{proof}

Now we can prove the upper and lower bounds are indeed upper and lower bounds of the best reponses.
\begin{lemma}\label{lem:sandwichVpik_Hoeffding}
		Suppose event $E_0$ holds. Then
 for all $h,s,a,b$ and $k\in[K]$, we have
	\begin{equation}
		\left\{
		\begin{aligned}
&		\up{Q}^{k}_h(s,a,b) \ge \QpikA_h(s,a,b) \ge \QpikB_h(s,a,b) \ge \low{Q}^{k}_h(s,a,b),\\
	&\up{V}^{k}_h(s) \ge \VpikA_h(s) \ge \VpikB_h(s) \ge \low{V}^{k}_h(s).
		\end{aligned}
\right.
	\end{equation}
	
\end{lemma}
\begin{proof}
	The proof is by backward induction. Suppose the bounds hold for the $Q$-values in the $(h+1)^{\rm th}$ step, we now establish the bounds for the $V$-values in the $(h+1)^{\rm th}$ step and $Q$-values in the $h^{\rm th}$-step. For any state $s$:
    \begin{equation}\label{boring-2}
        \begin{aligned}
            \up{V}^{k}_{h+1}(s)&=
            \D_{\pi^k_{h+1}} \up{Q}^{k}_{h+1}(s)\\
            &\ge \max_{\mu} \D_{\mu \times \nu^k_{h+1}}   \up{Q}^{k}_{h+1}(s)\\
            &\ge \max_{\mu} \D_{\mu \times \nu^k_{h+1}}   \QpikA_{h+1}(s) = \VpikA_{h+1}(s).
        \end{aligned}
        \end{equation}
	Similarly, we can show $\low{V}^{k}_{h+1}(s) \le \VpikB_{h+1}(s)$. 
	Therefore, we have: for all $s$,
	$$
	\up{V}^{k}_{h+1}(s) \ge \VpikA_{h+1}(s) \ge V^\star_{h+1}(s)\ge \VpikB_{h+1}(s) \ge \low{V}^{k}_{h+1}(s).
	$$
	Now consider an arbitrary triple $(s,a,b)$ in the $h^{\rm th}$ step. We have    
    \begin{equation}\label{eq:Q-decomposition}
        \begin{aligned}
			&(\up{Q}^{k}_h - \QpikA_h)(s,a,b) \\
			\ge &\min\bigg\{(\widehat{\P}_h^k\up{V}^{k}_{h+1} 
            - \P_h \VpikA_{h+1}  + \beta_h^k+\gamma_h^k)(s,a,b), 0\bigg\}\\
            \ge &\min\bigg\{(\widehat{\P}_h^k\VpikA_{h+1}
            - \P_h \VpikA_{h+1}  + \beta_h^k+\gamma_h^k)(s,a,b), 0\bigg\}\\
            = &\min\bigg\{ \underset{(A)}{\underbrace{(\Phat_h^k- \P_h)(\VpikA_{h+1}- \Vdstar_{h+1})(s,a,b)}}+\underset{(B)}{\underbrace{(\widehat{\P}_h^k- \P_h) \Vdstar_{h+1}(s,a,b)}}  + (\beta_h^k+\gamma_h^k)(s,a,b),0\bigg\}.
        \end{aligned}
        \end{equation}

	Invoking Lemma~\ref{lem:lower-order} with $g = \VpikA_{h+1}- \Vdstar_{h+1}$,
	$$
	|(A)| \le \cO\left(\frac{\Phat_h^k (\up{V}^{k}_{h+1} - \low{V}^{k}_{h+1})(s,a,b)}{H}
	+ \frac{H^2S\iota}{\minone{N_h^k(s,a,b)}}\right).
	$$

   By the first inequality in event $E_0$,
    $$
    |(B)| \le \cO\left(  \sqrt{\frac{H^2\iota }{\minone{N_h^k(s,a,b)}}} \right).
	$$
	Plugging the two inequalities above back into \eqref{eq:Q-decomposition} and recalling the definition of $\beta_h^k$ and $\gamma_h^k$, we obtain $\up{Q}^{k}_h(s,a,b) \ge \QpikA_h(s,a,b)$. Similarly, we can show $\low{Q}^{k}_h(s,a,b) \le \QpikB_h(s,a,b)$. 
\end{proof}

Finally we come to the proof of Theorem~\ref{thm:Nash-VI}.

\begin{proof}[Proof of Theorem~\ref{thm:Nash-VI}]
Suppose event $E_0$ holds.
We first upper bound the regret. By Lemma~\ref{lem:sandwichVpik_Hoeffding}, the regret can be upper bounded by 
$$
\sum_k (\VpikA_1(s_1^k) - \VpikB_1(s_1^k)) \le \sum_k (\up{V}^{k}_1(s_1^k) - \low{V}^{k}_1(s_1^k)). 
$$

For brevity's sake, we define the following notations:
\begin{equation}
\left\{
	\begin{aligned}
&\Delta_h^k := (\up{V}^{k}_h - \low{V}^{k}_h)(s_h^k),\\
&\zeta_h^k :=  \Delta_h^k
- (\up{Q}^{k}_h - \low{Q}^{k}_h)(s_h^k,a_h^k,b_h^k),\\
& \xi_h^k := \P_h(\up{V}^{k}_{h+1} - \low{V}^{k}_{h+1})(s_h^k,a_h^k,b_h^k)
- \Delta_{h+1}^k.
	\end{aligned}
	\right.
\end{equation}
Let $\cF_h^k$ be the $\sigma$-field generated by the following random variables: 
$$\{(s_i^j,a_i^j,b_i^j,r_i^j)\}_{(i,j)\in[H]\times[k-1]}\bigcup\{(s_i^k,a_i^k,b_i^k,r_i^k)\}_{i\in[h-1]}\bigcup\{s_h^k\}.$$
It's easy to check $\zeta_h^k$ and $\xi_h^k$ are martingale differences with respect to $\cF_h^k$. With a slight abuse of notation, we use $\beta_h^k$ to refer to $\beta_h^k(s_h^k,a_h^k,b_h^k)$  and $N_h^k$ to refer to $N_h^k(s_h^k,a_h^k,b_h^k)$ in the following proof.

 We have
 \begin{align*}
    \Delta _{h}^{k}
=&\zeta _{h}^{k}+\left( \up{Q}_{h}^{k}-\low{Q}_{h}^{k} \right) \left( s_{h}^{k},a_h^k,b_h^k \right) 
\\
\le &\zeta _{h}^{k}+2\beta _{h}^{k}+2\gamma _{h}^{k}+ \Phat_{h}^{k}(\up{V}_{h+1}^{k}-\low{V}_{h+1}^{k})  \left( s_{h}^{k},a_h^k,b_h^k \right) 
\\
\overset{\left( i \right)}{\le} &\zeta _{h}^{k}+2\beta _{h}^{k}+2\gamma _{h}^{k}+\P_{h}(\up{V}_{h+1}^{k}-\low{V}_{h+1}^{k})  \left( s_{h}^{k},a_h^k,b_h^k \right) 
+c_2\left(\frac{\P_h (\up{V}^{k}_{h+1} - \low{V}^{k}_{h+1})(s_{h}^{k},a_h^k,b_h^k)}{H}
+ \frac{H^2S\iota}{\minone{N_h^k}}\right)
\\
\overset{\left( ii \right)}{\le} &\zeta _{h}^{k}+
2\beta _{h}^{k}+
\P_{h}(\up{V}_{h+1}^{k}-\low{V}_{h+1}^{k})  \left( s_{h}^{k},a_h^k,b_h^k \right) 
+2c_2C\left(\frac{\P_h (\up{V}^{k}_{h+1} - \low{V}^{k}_{h+1})(s_{h}^{k},a_h^k,b_h^k)}{H}
+ \frac{H^2S\iota}{\minone{N_h^k}}\right)
\\
\le &\zeta _{h}^{k}+
\left(1+\frac{2c_2C}{H}\right)\P_{h}(\up{V}_{h+1}^{k}-\low{V}_{h+1}^{k})  \left( s_{h}^{k},a_h^k,b_h^k \right) 
+4c_2C\left(\sqrt{\frac{\iota H^2}{\minone{N_h^k}}}
+ \frac{H^2S\iota}{\minone{N_h^k}}\right) 
\\
= &\zeta _{h}^{k}+\left(1+\frac{2c_2C}{H}\right)\xi_{h}^k+\left(1+\frac{2c_2C}{H}\right)\Delta_{h+1}^k +
4c_2C\left(\sqrt{\frac{\iota H^2}{\minone{N_h^k}}}
+ \frac{H^2S\iota}{\minone{N_h^k}}\right) 
\end{align*}
where $(i)$ and $(ii)$ follow from Lemma~\ref{lem:lower-order}.

Define $c_3 := 1+2c_2C$ and $\kappa := 1+c_3/H$.
Recursing this argument for $h \in [H]$ and summing over $k$,
\begin{align*}
\sum_{k=1}^{K}{\Delta_1^k} \le  
\sum_{k=1}^{K}\sum_{h=1}^{H}\left[
\kappa^{h-1}\zeta_h^k + \kappa^{h}\xi_h^k+
{\cO\left(\sqrt{\frac{\iota H^2}{\minone{N_h^k}}}
+ \frac{H^2S\iota}{\minone{N_h^k}}\right)}\right].
\end{align*}

By Azuma-Hoeffding inequality, with probability at least $1-p$,
\begin{equation}\label{boring}
\left\{
	\begin{aligned}
&\sum_{k=1}^{K}\sum_{h=1}^{H}{\kappa^{h-1}\zeta_h^k} \le \cO \paren{H\sqrt{HK\iota}} = \cO \paren{\sqrt{H^2T\iota}},\\
&\sum_{k=1}^{K}\sum_{h=1}^{H}{\kappa^{h}\xi_h^k} \le \cO \paren{H\sqrt{HK\iota}} = \cO \paren{\sqrt{H^2T\iota}}.
\end{aligned}\right.
\end{equation}

By pigeon-hole argument,
\begin{align*}
	\sum_{k=1}^{K}\sum_{h=1}^{H}{\frac{1}{\sqrt{\minone{N_h^k}}}} \le &\sum_{s,a,b,h:\ N_h^K(s,a,b)>0}\sum_{n=1}^{N_h^K(s,a,b)}{\frac{1}{\sqrt{n}}}+HSAB 
	\le \cO \paren{\sqrt{HSABT}+HSAB},
\end{align*}
\begin{align*}
\sum_{k=1}^{K}\sum_{h=1}^{H}{\frac{1}{\minone{N_h^k}}} \le &\sum_{s,a,b,h:\ N_h^K(s,a,b)>0}\sum_{n=1}^{N_h^K(s,a,b)}{\frac{1}{n}}+HSAB 
	\le\cO \paren{HSAB\iota}.
\end{align*}

Put everything together, with probability at least $1-2p$ (one $p$ comes from $\Pr(E_0)\ge1-p$ and the other is for equation \eqref{boring}),
$$
\sum_{k=1}^K (\VpikA_1(s_1^k) - \VpikB_1(s_1^k)) \le  \cO \paren{\sqrt{H^3SABT\iota} + H^3S^2AB\iota^2 }
$$

For the PAC guarantee, recall that we choose   $\pi^{\text{out}}= \pi^{k^{\star}}$ such that $k^{\star} = \argmin_{k} \left( \up{V}_{1}^{k}-\low{V}_{1}^{k} \right)\left( s_{1} \right)$. As a result, 
$$
 (V^{\dagger,\nu^{k^{\star}}}_1 - V^{\mu^{k^{\star}},\dagger}_1)(s_1) \le  (\up{V}^{k^{\star}}_1 - \low{V}^{k^{\star}}_1) (s_1) \le \frac{1}{K} \cO \paren{\sqrt{H^3SABT\iota} + H^3S^2AB\iota^2 },
$$
which concludes the proof.
\end{proof}
\subsection{Proof of Theorem~\ref{thm:Nash-VI-B}}
\label{appendix:pf_Nash-VI_Bernstein}
We use the same notation as in Appendix~\ref{appendix:pf_Nash-VI_Hoeffding} except the form of bonus. Besides, we define the empirical variance operator
$$
\widehat{\Va}^k_hV(s,a,b) :=\Var_{s'\sim \Phat^k_h(\cdot|s,a,b)}V(s')
$$ 
and the true (population) variance operator 
$$
\Va_hV(s,a,b) :=\Var_{s'\sim \P_h(\cdot|s,a,b)}V(s')
$$ 
for any function $V \in \Delta^{S}$.
If $N^k_h(s, a, b)=0$, we simply set $\widehat{\Va}^k_hV(s,a,b):=H^2$ regardless of the choice of $V$.

As a result, the bonus terms can be written as
\begin{equation}
	\beta_h^k(s,a,b) := C\paren{\sqrt{\frac{\iota\widehat{\Va}^k_h [(\up{V}^k_{h+1} + \low{V}^k_{h+1})/2](s,a,b)}{\minone{N_h^k(s,a,b)}}} + \frac{H^2S\iota}{\minone{N_h^k(s,a,b)}}}
\end{equation}
for some absolute constant $C>0$.

\begin{lemma}\label{event-vi-2}
Let $c_1$  be some large absolute constant.
 Define event $E_1$ to be: for all $h,s,a,b,s'$ and $k\in[K]$, 
\begin{equation*}
\left\{
	\begin{aligned}
		&|[(\Phat^k_h - \Pr_h) V^\star_{h+1}](s,a,b)| \le {c_1} \paren{\sqrt{\frac{\widehat{\Va}_h^k \Vdstar_{h+1}(s,a,b) \iota}{\minone{N_h^k(s,a,b)}}} + \frac{H\iota }{\minone{N_h^k(s,a,b)}}},\\
		& |(\Phat^k_h - \Pr_h)(s'\mid s,a,b)| 
		\le {c_1}
	\paren{\sqrt{ \frac{\min\{\Pr_h(s'\mid s,a,b),\Phat^k_h(s'\mid s,a,b)\}\iota}{\max\{N_h^k(s,a,b),1\}}}+ \frac{\iota}{\minone{N_h^k(s,a,b)}}},
	\\
&	\|(\Phat^k_h - \Pr_h)(\cdot\mid s,a,b)\|_1 
		\le {c_1}
	{\sqrt{ \frac{S\iota}{\max\{N_h^k(s,a,b),1\}}}}.
	\end{aligned}
	\right.
\end{equation*}
We have $\Pr(E_1)\ge 1-p$.
\end{lemma}
The proof of Lemma \ref{event-vi-2} is highly similar to that of Lemma \ref{event-vi-1}. Specifically, the first two can be proved by following basically the same argument in Lemma \ref{event-vi-1}; the third one is standard (e.g., equation (12) in \cite{azar2017minimax}). We omit the proof here.

Since the proof of Lemma~\ref{lem:lower-order} does not depend on the form of the bonus, it can also be applied in this section. As in Appendix~\ref{appendix:pf_Nash-VI_Hoeffding}, we will prove the upper and lower bounds are indeed upper and lower bounds of the best reponses.
	
\begin{lemma}\label{lem:sandwichVpik_Bernstein}
	Suppose event $E_1$ holds. Then for all $h,s,a,b$ and $k\in[K]$, we have
	\begin{equation}
		\left\{
		\begin{aligned}
&		\up{Q}^{k}_h(s,a,b) \ge \QpikA_h(s,a,b) \ge \QpikB_h(s,a,b) \ge \low{Q}^{k}_h(s,a,b),\\
	&\up{V}^{k}_h(s) \ge \VpikA_h(s) \ge \VpikB_h(s) \ge \low{V}^{k}_h(s).
		\end{aligned}
\right.
	\end{equation}
\end{lemma}
\begin{proof}
The proof is by backward induction and very similar to that of Lemma~\ref{lem:sandwichVpik_Hoeffding}. 
Suppose the bounds hold for the $Q$-values in the $(h+1)^{\rm th}$ step, we now establish the bounds for the $V$-values in the $(h+1)^{\rm th}$ step and $Q$-values in the $h^{\rm th}$-step.

The proof for the $V$-values is the same as 
\eqref{boring-2}.

For the $Q$-values, the decomposition \eqref{eq:Q-decomposition} still holds and $(A)$ is bounded using Lemma~\ref{lem:lower-order} as before. The only difference is that we need to bound $(B)$ more carefully.

First, by the first inequality in event $E_1$,
$$
    |(B)| \le \cO\left(  \sqrt{\frac{\widehat{\Va}_h^k \Vdstar_{h+1}(s,a,b) \iota}{\minone{N_h^k(s,a,b)}}} + \frac{H\iota }{\minone{N_h^k(s,a,b)}}\right).
    $$

 By the relation of $V$-values in the $(h+1)^{\rm th}$ step,
\begin{equation}\label{eq:july25-2}
		\begin{aligned}
			&|[\Vahat_h^k (\up{V}^k_{h+1} + \low{V}^k_{h+1})/2] - \widehat{\Va}_h^k\Vdstar_{h+1}|(s,a,b)\\
\le & |[\Phat_h^k (\up{V}^k_{h+1} + \low{V}^k_{h+1})/2]^2-(\Phat_h^k\Vdstar_{h+1})^2|(s,a,b)+|\Phat_h^k [(\up{V}^k_{h+1} + \low{V}^k_{h+1})/2]^2-\Phat_h^k(\Vdstar_{h+1})^2|(s,a,b)
			\\
			  \le&
			4H\widehat{\P}_h^k |(\up{V}^k_{h+1} + \low{V}^k_{h+1})/2 - \Vdstar_{h+1}|(s,a,b)\\
			 \le &
			4H\widehat{\P}_h^k (\up{V}^{k}_{h+1} - \low{V}^{k}_{h+1})(s,a,b),
		\end{aligned}
		\end{equation}
		which implies
		\begin{equation}\label{eq:july25-3}
		\begin{aligned}
&\qquad  \sqrt{\frac{\iota\widehat{\Va}_h^k \Vdstar_{h+1}(s,a,b) }{\minone{N_h^k(s,a,b)}}}  \\
			&\le 
		\sqrt{\frac{\iota[\Vahat_h^k [(\up{V}^k_{h+1} + \low{V}^k_{h+1})/2] + 4 H\widehat{\P}_h^k (\up{V}^{k}_{h+1} - \low{V}^{k}_{h+1})](s,a,b)}{\minone{N_h^k(s,a,b)}}}\\
		&	\le    \sqrt{\frac{\iota\widehat{\Va}_h^k [(\up{V}^k_{h+1} + \low{V}^k_{h+1})/2](s,a,b) }{\minone{N_h^k(s,a,b)}}}+ \sqrt{\frac{4\iota H\widehat{\P}_h^k (\up{V}^{k}_{h+1} - \low{V}^{k}_{h+1})](s,a,b) }{\minone{N_h^k(s,a,b)}}}\\
	&	\overset{\left( i \right)}{\le} \sqrt{\frac{\iota\widehat{\Va}_h^k [(\up{V}^k_{h+1} + \low{V}^k_{h+1})/2](s,a,b) }{\minone{N_h^k(s,a,b)}}}
		+ \frac{\widehat{\P}_h^k (\up{V}^{k}_{h+1} - \low{V}^{k}_{h+1})}{H}
		+ \frac{4H^2\iota}{\minone{N_h^k(s,a,b)}},
		\end{aligned}
		\end{equation}
		where $(i)$ is by AM-GM inequality.

	Plugging the above inequalities back into \eqref{eq:Q-decomposition} and recalling the definition of $\beta_h^k$ and $\gamma_h^k$ completes the proof.
\end{proof}

We need one more lemma to control the error of the empirical variance estimator:
\begin{lemma}
	\label{lem:bound_variance}
	Suppose event $E_1$ holds. Then for all $h,s,a,b$ and $k\in[K]$, we have
	\begin{align*}
&	| \Vahat_h^k [(\up{V}^k_{h+1} + \low{V}^k_{h+1})/2] - \Va_h\Vpik_{h+1}|(s,a,b) \\
	\le & 4H\Pr_h(\up{V}^{k}_{h+1} - \low{V}^{k}_{h+1})(s,a,b) + 
	\bigO\paren{1 + \frac{H^4S\iota}{\minone{N_h^k(s,a,b)}}}.
	\end{align*}
\end{lemma}
\begin{proof}
By Lemma~\ref{lem:sandwichVpik_Bernstein}, we have $\up{V}^{k}_h(s) \ge V_h^{\pi^k}(s) \ge \low{V}^{k}_h(s)$. As a result,
	\begin{align*}
	&	 |\Vahat_h^k [(\up{V}^k_{h+1} + \low{V}^k_{h+1})/2] - \Va_h\Vpik_{h+1}|(s,a,b) \\
		= & | [\Phat_h^k( \up{V}^k_{h+1} + \low{V}^k_{h+1})^2/4 - \P_h(\Vpik_{h+1})^2](s,a,b)
- [(\Phat_h^k (\up{V}^k_{h+1} + \low{V}^k_{h+1}))^2/4 - (\P_h\Vpik_{h+1})^2](s,a,b)| \\
\le & [\Phat_h^k( \up{V}^k_{h+1})^2 - \P_h(\low{V}^k_{h+1})^2
- (\Phat_h^k \low{V}^k_{h+1})^2 + (\P_h\up{V}^k_{h+1})^2](s,a,b)
\\
\le &[|(\Phat_h^k-\P_h)( \up{V}^k_{h+1})^2|+|\P_h[( \up{V}^k_{h+1})^2-(\low{V}^k_{h+1})^2]|\\
&+|(\Phat_h^k \low{V}^k_{h+1})^2-(\P_h \low{V}^k_{h+1})^2|+|(\P_h \low{V}^k_{h+1})^2-(\P_h\up{V}^k_{h+1})^2|](s,a,b)
	\end{align*}
These terms can be bounded separately by using event $E_1$:
\begin{align*}
&	|(\Phat_h^k-\P_h)( \up{V}^k_{h+1})^2|(s,a,b) \le 
	H^2\|(\Phat^k_h-\Pr_h)(\cdot\mid s,a,b)\|_1
	\le
	\cO(H^2\sqrt{\frac{S\iota}{\minone{N_h^k(s,a,b)}}}),
	\\
&	|\P_h[( \up{V}^k_{h+1})^2-(\low{V}^k_{h+1})^2]|(s,a,b) \le 2H[\P_h( \up{V}^k_{h+1}-\low{V}^k_{h+1})](s,a,b),
	\\
&	|(\Phat_h^k \low{V}^k_{h+1})^2-(\P_h \low{V}^k_{h+1})^2|(s,a,b) \le 2H[(\Phat_h^k-\P_h)\low{V}^k_{h+1}](s,a,b) \le \cO(H^2\sqrt{\frac{S\iota}{\minone{N_h^k(s,a,b)}}}),
	\\
&	|(\P_h \low{V}^k_{h+1})^2-(\P_h\up{V}^k_{h+1})^2|(s,a,b) \le 2H[\P_h( \up{V}^k_{h+1}-\low{V}^k_{h+1})](s,a,b).	
\end{align*}

Combining with $H^2\sqrt{\frac{S\iota}{\minone{N_h^k(s,a,b)}}} \le 1 + \frac{H^4S\iota}{\minone{N_h^k(s,a,b)}}$ completes the proof.
\end{proof} 

Finally we come to the proof of Theorem~\ref{thm:Nash-VI-B}.

\begin{proof}[Proof of Theorem~\ref{thm:Nash-VI-B}]
Suppose event $E_1$ holds.
We define $\Delta_h^k$, $\zeta_h^k$ abd $\xi_h^k$ as in the proof of Theorem~\ref{thm:Nash-VI}. As before we have
\begin{equation}
	\label{equ:single-step}
\begin{aligned}
    \Delta _{h}^{k}
\le &\zeta _{h}^{k}+
\left(1+\frac{c_3}{H}\right)\P_{h}(\up{V}_{h+1}^{k}-\low{V}_{h+1}^{k})  \left( s_{h}^{k},a_h^k,b_h^k \right) \\
&
+4c_2C\left(\sqrt{\frac{\iota\Vahat_h^k[(\up{V}^k_{h+1} + \low{V}^k_{h+1})/2]( s_{h}^{k},a_h^k,b_h^k)}{\minone{N_h^k( s_{h}^{k},a_h^k,b_h^k)}}}
+ \frac{H^2S\iota}{\minone{N_h^k( s_{h}^{k},a_h^k,b_h^k)}}\right).
\end{aligned}
\end{equation}

By Lemma~\ref{lem:bound_variance},
	\begin{equation}\label{eq:july25-6}
	\begin{aligned}
		&\sqrt{\frac{\iota\Vahat_h^k[(\up{V}^k_{h+1} + \low{V}^k_{h+1})/2](s,a,b)}{\minone{N_h^k(s,a,b)}}}\\
		 \le&
\bigO\paren{\sqrt{\frac{\iota\Va_h\Vpik_{h+1}(s,a,b) + \iota}{\minone{N_h^k(s,a,b)}}} + \sqrt{\frac{H\iota\Pr_h (\up{V}^{k}_{h+1} - \low{V}^{k}_{h+1})(s,a,b)}{\minone{N_h^k(s,a,b)}}}+ \frac{H^2\sqrt{S}\iota}{\minone{N_h^k(s,a,b)}}}\\
		\le& c_4\paren{\sqrt{\frac{\iota\Va_h\Vpik_{h+1}(s,a,b) + \iota}{\minone{N_h^k(s,a,b)}}} + 
		\frac{\Pr_h (\up{V}^{k}_{h+1} - \low{V}^{k}_{h+1})(s,a,b)}{H}
	+ \frac{H^2\sqrt{S}\iota}{\minone{N_h^k(s,a,b)}}},
	\end{aligned}
	\end{equation}
	where $c_4$ is some absolute constant.
	Define $c_5:=4c_2c_4C+c_3$ and $\kappa:=1+c_5/H$. 
	Plugging \eqref{eq:july25-6} back into \eqref{equ:single-step}, we have 
	\begin{equation}
	\begin{aligned}
		\Delta_h^k\le \kappa \Delta_{h+1}^k + \kappa   \xi_h^k +\zeta_h^k +\cO\bigg(\sqrt{\frac{\iota\Va_h\Vpik_{h+1}(s_h^k,a_h^k,b_h^k)}{N_h^k(s_h^k,a_h^k,b_h^k)}} + \sqrt{\frac{\iota}{N_h^k(s_h^k,a_h^k,b_h^k)}}	+ \frac{H^2S\iota}{N_h^k(s_h^k,a_h^k,b_h^k)}\bigg)\Bigg\}.
		\end{aligned}
	\end{equation}
Recursing this argument for $h \in [H]$ and summing over $k$,
\begin{align*}
\sum_{k=1}^{K}{\Delta_1^k} \le  
\sum_{k=1}^{K}\sum_{h=1}^{H}\bigg[
\kappa^{h-1}\zeta_h^k + \kappa^{h}\xi_h^k
+
\cO
\left(
\sqrt{\frac{\iota\Va_h\Vpik_{h+1}(s_h^k,a_h^k,b_h^k)}{\minone{N_h^k}}}
+\sqrt{\frac{\iota}{\minone{N_h^k}}}
+ \frac{H^2S\iota}{\minone{N_h^k}}
\right)\bigg].
\end{align*}

The remaining steps are the same as that in the proof of Theorem~\ref{thm:Nash-VI} except that we need to bound the sum of variance term.

 By Cauchy-Schwarz,
\begin{align*}
	\sum_{k=1}^{K}\sum_{h=1}^{H}{\sqrt{\frac{\Va_h\Vpik_{h+1}(s_h^k,a_h^k,b_h^k)}{\minone{N_h^k(s_h^k,a_h^k,b_h^k)}}}}
	\le  \sqrt{\sum_{k=1}^{K}\sum_{h=1}^{H}{\Va_h\Vpik_{h+1}(s_h^k,a_h^k,b_h^k)}\cdot \sum_{k=1}^{K}\sum_{h=1}^{H}{\frac{1}{\minone{N_h^k(s_h^k,a_h^k,b_h^k)}}}}.\end{align*}

By the Law of total variation and standard martingale concentration (see Lemma C.5 in \citet{jin2018q} for a formal proof), with probability at least $1-p$, we have
\begin{align*}
	\sum_{k=1}^{K}\sum_{h=1}^{H}{\Va_h\Vpik_{h+1}(s_h^k,a_h^k,b_h^k)}{\le}  \cO \paren{HT+H^3\iota}.
\end{align*}
Putting all relations together, we obtain that with probability at least $1-2p$ (one $p$ comes from $\Pr(E_1)\ge 1-p$ and the other comes from the inequality for bounding the variance term), 
$$\Reg(K) = \sum_{k=1}^K (\VpikA_1- \VpikB_1)(s_1) \le  \cO (\sqrt{H^2SABT\iota} + H^3S^2AB\iota^2 ).$$
Rescaling $p$ completes the proof.
\end{proof}

\section{Proof for Section \ref{sec:reward-free} -- Reward-Free Learning}
\subsection{Proof of Theorem \ref{thm:rewardfree-explore}}
\label{appendix:pf_rewardfree}

In this section, we prove Theorem \ref{thm:rewardfree-explore} for the single reward function case, i.e., $N=1$. 
The proof for multiple reward functions ($N>1$) simply follows  from taking a union bound, that is, replacing the failure probability $p$ by $Np$.

Let $(\mu^k,\nu^k)$ be an arbitrary Nash-equilibrium policy of $\widehat{\cM}^k := (\Phat^k,\rhat^k)$, where $\Phat^k$ and $\rhat^k$ are our empirical estimate of the transition and the reward at the beginning of the $k$'th episode in Algorithm \ref{alg:VI-zero}, respectively.
We use $N_h^k(s,a,b)$ to denote the number we have visited the state-action tuple $(s,a,b)$ at the $h$'=th step before the $k$'th episode. And the bonus used in the $k$'th episode can be written as 
\begin{equation}\label{eq:bonus}
	\beta_h^k(s,a,b) := C\paren{\sqrt{ \frac{H^2\iota}{\max\{N_h^k(s,a,b),1\}}}+ \frac{H^2S\iota}{\minone{N_h^k(s,a,b)}}},\end{equation}
where $\iota = \log(SABT /p)$ and  $C$ is some large absolute constant.

%Note that the updating rule above is slightly different from the original one in \cite{azar2017minimax}: 1. we truncate the state value function instead of the state-action value function by $H$; 2. our bonus term has an extra high order term. 
%Nonetheless, one can verify easily that the analysis of \cite{azar2017minimax} still applies with minor changes. 
%
%\qinghua{explain that the bonus here has an extra high-order term compared with that in UCB-VI. Nevertheless, one can easily verify this high order term will only incur a neglectable low order term in the final regret.}

We use $\Qhat^k$ and $\Vhat^k$ to denote the empirical optimal value functions of  $\widehat{\cM}^k$ as following.

\begin{equation}
\left\{
\begin{aligned}
&	\Qhat^k_h(s,a,b) = (\Phat^k_h\Vhat_{h+1})(s,a,b)+\rhat_h^k(s,a,b),     \\
&	\Vhat^k_h(s) = \max_{\mu}\min_{\nu} \D_{\mu\times\nu}  \Qhat^k_h(s).
\end{aligned}	
\right.
\end{equation}
Since $(\mu^k,\nu^k)$ is a Nash-equilibrium policy of $\widehat{\cM}^k$, we also have $\Vhat^k_h(s)  =\D_{\mu^k\times\nu^k} \Qhat^k_h(s) $.

%For technical reasons, we further define the following auxiliary optimistic value functions, which will be used as \emph{bridging} quantities in our analysis.
%\begin{equation}
%\left\{
%\begin{aligned}
%&	Q^k_h(s,a,b) = \min \left\{(\Phat^k_h V_{h+1}^k)(s,a,b)+r_h(s,a,b) + \beta_h^k(s,a,b),H\right\},     \\
%&	V^k_h(s) =  \min_{\mu\in\Delta_A}\max_{\nu\in\Delta_B} (\D_{\mu \times \nu}Q^k_h)(s).
%\end{aligned}	
%\right.
%\end{equation}

We begin with stating a useful property of matrix game that will be frequently used in our analysis. Since its proof is quite simple, we omit it here.

\begin{lemma} \label{lem:lipschitz}
Let $\X,\Y,\mat{Z}\in\mathbb{R}^{A\times B}$ and $\Delta_d$ be the $d$-dimensional simplex.
Suppose $\left|\X-\Y \right|\le \mathbf{Z}$, where the inequality is entry-wise. Then
	\begin{equation}
		\left|\max_{\mu\in\triangle_A}\min_{\nu\in\triangle_B}
		\mu\trans \X \nu 
		-\max_{\mu\in\triangle_A}\min_{\nu\in\triangle_B}
		\mu\trans \Y \nu \right|
		\le  \max_{i,j} \mathbf{Z}_{ij}.
	\end{equation}
\end{lemma}

\begin{lemma}\label{event-rf-1}
Let $c_1$  be some large absolute constant such that $c_1^2+c_1\le C$.
 Define event $E_1$ to be: for all $h,s,a,b,s'$ and $k\in[K]$, 
\begin{equation}
\left\{
	\begin{aligned}
		&|[(\Phat^k_h - \Pr_h) V^\star_{h+1}](s,a,b)| \le \frac{c_1}{10}\sqrt{ \frac{H^2\iota}{\max\{N_h^k(s,a,b),1\}}},\\
		&|(\rhat^k_h - r_h)(s,a,b)| \le \frac{c_1}{10}\sqrt{ \frac{H^2\iota}{\max\{N_h^k(s,a,b),1\}}}, \\
		& |(\Phat^k_h - \Pr_h)(s'\mid s,a,b)| 
		\le \frac{c_1}{10}
	\paren{\sqrt{ \frac{\Phat^k_h(s'\mid s,a,b)\iota}{\max\{N_h^k(s,a,b),1\}}}+ \frac{\iota}{\minone{N_h^k(s,a,b)}}}.
	\end{aligned}
	\right.
\end{equation}
We have $\Pr(E_1)\ge 1-p$.
\end{lemma}
\begin{proof}
	The proof is standard: apply concentration inequalities and then take a union bound. For completeness, we provide the proof of the third one here.
	
Consider a fixed $(s,a,b,h)$ tuple. 

Let's consider the following equivalent random process:
(a) before the agent starts, the environment samples 
$\{s^{(1)},s^{(2)},\ldots,s^{(K)}\}$ independently from $\Pr_h(\cdot\mid s,a,b)$; 
(b) during the interaction between the agent and the environment, the $i^{\rm th}$ time the agent reaches $(s,a,b,h)$, the environment will make the agent transit to $s^{(i)}$. 
Note that the randomness induced by this interaction procedure is exactly the same as the original one, which means the probability of any event in this context is the same as in the original problem.
 Therefore, it suffices to prove the target concentration inequality in this 'easy' context.
Denote by $\Phat^{(t)}_h(\cdot\mid s,a,b)$ the empirical estimate of $\Pr_h(\cdot\mid s,a,b)$ calculated using $\{s^{(1)},s^{(2)},\ldots,s^{(t)}\}$.
 For a fixed $t$ and $s'$, by the empirical Bernstein inequality, we have with probability at least $1-p/S^2ABT$,
 $$
 |(\Pr_h - \Phat^{(t)}_h)(s'\mid s,a,b)| 
		\le \bigO 
	\paren{\sqrt{ \frac{\Phat^{(t)}_h(s'\mid s,a,b)\iota}{t}}+ \frac{\iota}{t}}.
 $$
 
Now we can take a union bound over all $s,a,b,h,s'$ and $t\in[K]$, and obtain 
that with probability at least $1-p$, for all $s,a,b,h,s'$ and $t\in[K]$,
 $$
 |(\Pr_h - \Phat^{(t)}_h)(s'\mid s,a,b)| 
		\le \bigO 
	\paren{\sqrt{ \frac{\Phat^{(t)}_h(s'\mid s,a,b)\iota}{t}}+ \frac{\iota}{t}}.
 $$
Note that the agent can reach each $(s,a,b,h)$ for at most $K$ times, so we conclude the third  inequality also holds with probability at least $1-p$.
	\end{proof}

The following lemma states that the empirical optimal value functions are close to the true optimal ones, and their difference is controlled by the exploration value functions 
 calculated in Algorithm \ref{alg:VI-zero}.

\begin{lemma}\label{lem:MG-closeQstarQk}
Suppose event $E_1$ (defined in Lemma \ref{event-rf-1}) holds. Then for all $h,s,a,b$ and $k\in[K]$, we have, 
\begin{equation}
\left\{
\begin{aligned}
	&	\left|\Qhat_h^k(s,a,b) - Q_h^\star(s,a,b) \right|\le \Qt_h^k(s,a,b), \\
&	\left| \Vhat_h^k(s) - V_h^\star(s) \right|\le \Vt_h^k(s).
\end{aligned}
\right.
\end{equation}
\end{lemma}

\begin{proof}
	Let's prove by backward induction on $h$. The case of $h=H+1$ holds trivially.
	
	Assume the conclusion hold for $(h+1)$'th step. For $h$'th step, 
	\begin{equation}
		\begin{aligned}
			&\quad \ \left|\Qhat_h^k(s,a,b) - Q_h^\star(s,a,b)\right|\\
				 & \le \min\left\{ \left| [(\Phat^k_h - \Pr_h) V^\star_{h+1}](s,a,b) \right|
				 + 
				 |(\rhat^k_h - r_h)(s,a,b)|
			 + \left|[\Phat_h^k (\Vhat^k_{h+1} -V^\star_{h+1})](s,a,b)\right|, H\right\}
			\\
			 &  \overset{(i)}{\le}\min\left\{\beta_h^k(s,a,b)+
			  (\Phat_h^k \Vt^k_{h+1})(s,a,b),H\right\}
			  \stackrel{(ii)}{=}  \Qt^k_h(s,a,b),
		\end{aligned}
	\end{equation}
	where $(i)$ follows from the induction hypothesis and event $E_1$, and $(ii)$ follows from the definition of $\Qt^k_h$. 
	By Lemma \ref{lem:lipschitz}, we immediately obtain $| \Vhat_h^k(s) - V_h^\star(s) |\le \Vt_h^k(s)$.
\end{proof}

%The following lemma states that the optimistic value function $Q^k$ ($V^k$) is also an upper bound for $\Qhat^k$ ($\Vhat^k$). Moreover, their difference is controlled by $\Qt^k$ plus the planning error $\epsplan$.
%\begin{lemma}\label{lem:MG-closeQhatQk}
%With probability $1-p$, for all $h,s,a,b$ and $k\in[K]$, we have, 
%\begin{equation}
%\left\{
%\begin{aligned}
%	&	|Q_h^k(s,a,b) - \Qhat_h^k(s,a,b) |\le \Qt^k_h(s,a,b) + (H-h)\epsplan, \\
%&	| V_h^k(s) - \Vhat_h^k(s) |\le \Vt^k_h(s)+ (H+1-h)\epsplan.
%\end{aligned}
%\right.
%\end{equation}
%\end{lemma}
%\begin{proof}
%Let's prove by doing backward induction on $h$.
%	
%	Assume the conclusion hold for $h+1,\ldots,H$. For $h$'th step, 
%	\begin{equation}
%		\begin{aligned}
%		|Q_h^k(s,a,b) - \Qhat_h^k(s,a,b)|
%				 & \le \minH{ |[\Phat_h^k (V^k_{h+1} -\Vhat^k_{h+1})](s,a,b)
%			 + \beta_h^k(s,a,b)|}\\
%			 &  \stackrel{(i)}{\le} \minH{|(\Phat_h^k \Vt^k_{h+1})(s,a,b)
%			 + \beta_h^k(s,a,b)+(H-h)\epsplan|}\\
%			 & \stackrel{(ii)}{\le}  \Qt^k_h(s,a,b)+(H-h)\epsplan,	
%		\end{aligned}
%\end{equation}
%where $(i)$ follows from the induction hypothesis and $(ii)$ follows from the definition of $\Qt^k_h$.
%
%Furthermore, by Lemma \ref{lem:lipschitz}, have
%\begin{equation}
%	\begin{aligned}
%		&| V_h^k(s) - \Vhat_h^k(s) | \\
%		\le	 & | \min_\mu\max_\nu (\D_{\mu\times \nu}Q_h^k)(s) - \min_\mu\max_\nu (\D_{\mu\times \nu}\Qhat_h^k)(s) |+ \epsplan\\
%		\le & \max_{a,b} \Qt^k_h(s,a,b) + (H+1-h)\epsplan \\
%		=& 		\Vt^k_h(s)+ (H+1-h)\epsplan.
%	\end{aligned}
%\end{equation}
%\end{proof}

Now, we are ready to establish the key lemma in our analysis using Lemma \ref{lem:MG-closeQstarQk}.

\begin{lemma}\label{lem:MG-closeQpiQhat}
Suppose event $E_1$ (defined in Lemma \ref{event-rf-1}) holds. Then for all $h,s,a,b$ and $k\in[K]$, we have
\begin{equation}
\left\{
\begin{aligned}
	&	|\Qhat_h^k(s,a,b) - Q_h^{\dagger,\nu^k}(s,a,b)| \le  \alpha_{h} \Qt^k_h(s,a,b), \\
&	|\Vhat_h^k(s) - V_h^{\dagger,\nu^k}(s)| \le  \alpha_{h} \Vt^k_h(s) ,
\end{aligned}
\right.
\end{equation}
and 
\begin{equation}
\left\{
\begin{aligned}
	&	|\Qhat_h^k(s,a,b) - Q_h^{\mu^k,\dagger}(s,a,b)| \le   \alpha_{h} \Qt^k_h(s,a,b), \\
&	|\Vhat_h^k(s) - V_h^{\mu^k,\dagger}(s)| \le  \alpha_{h} \Vt^k_h(s),
\end{aligned}
\right.
\end{equation}
where $\alpha_{H+1}=0$ and $\alpha_h =[(1+\frac{1}{H})\alpha_{h+1} + \frac{1}{H}]\le 4$.
\end{lemma}

\begin{proof}
We only prove the first set of inequalities. The second one follows exactly the same.
Again, the proof is by performing backward induction on $h$.
	It is trivial to see the conclusion holds for $(H+1)$'th step with $\alpha_{H+1}=0$.
Now, assume the conclusion holds for $(h+1)$'th step. For $h$'th step, 
	\begin{equation}\label{eq:july2-1}
		\begin{aligned}
&			|\Qhat_h^k(s,a,b) - Q_h^{\dagger,\nu^k}(s,a,b)|\\
\le &\min\bigg\{ |[(\Phat_h^k-\Pr_h)(V_{h+1}^{\dagger,\nu^k} - V_{h+1}^\star)](s,a,b)|
+|(\Phat_h^k-\Pr_h)V_{h+1}^\star(s,a,b)| \\&
+ |(\rhat^k_h - r_h)(s,a,b)|
+ |[\Phat_h(\Vhat_{h+1}^k-V^{\dagger,\nu^k}_{h+1})](s,a,b)|, H\bigg\}\\
\le & \min\bigg\{\underbrace{|[(\Phat_h^k-\Pr_h)(V_{h+1}^{\dagger,\nu^k} - V_{h+1}^\star)](s,a,b)|}_{(T_1)}
+c_1\sqrt{\frac{H^2\iota}{\minone{N_h^k(s,a,b)} }}+ \underbrace{|[\Phat_h(\Vhat_{h+1}^k-V^{\dagger,\nu^k}_{h+1})](s,a,b)|}_{(T_2)},H\bigg\},
		\end{aligned}
\end{equation}
where the second inequality follows from the definition of event $E_1$.

We can control the term $(T_1)$ by combining Lemma \ref{lem:MG-closeQstarQk} and the induction hypothesis to bound $|V_{h+1}^{\dagger,\nu^k}-V^\star_{h+1}|$, and then applying the third inequality in event $E_1$:
\begin{equation}\label{eq:july2-1-1}
\begin{aligned}	
(T_1)\le & \sum_{s'} | \Phat_h^k(s'\mid s,a,b) -\Pr_h(s'\mid s,a,b)| |V_{h+1}^{\dagger,\nu^k} - V_{h+1}^\star(s')| \\
\le & \sum_{s'} | \Phat_h^k(s'\mid s,a,b) -\Pr_h(s'\mid s,a,b)| \paren{|V_{h+1}^{\dagger,\nu^k} - \Vhat_{h+1}^k(s')| 
+ |\Vhat_{h+1}^k - V_{h+1}^\star(s')|}\\
\le & \sum_{s'} | \Phat_h^k(s'\mid s,a,b) -\Pr_h(s'\mid s,a,b)| \paren{\alpha_{h+1}+1} \Vt_{h+1}^k\\
\le & \frac{\paren{\alpha_{h+1}+1}}{H}(\Phat_h^k\Vt^k_{h+1})(s,a,b) + \frac{c_1^2\paren{\alpha_{h+1}+1}H^2S\iota}{\minone{N_h^k(s,a,b)}}.
	\end{aligned}
\end{equation}

The term $(T_2)$ is bounded by directly applying the induction hypothesis
\begin{equation}\label{eq:july2-2}
		\begin{aligned}
		|[\Phat_h(\Vhat_{h+1}^k-V^{\dagger,\nu^k}_{h+1})](s,a,b)|
		 \le \alpha_{h+1} [\Phat_h\Vt^k_{h+1}](s,a,b).
		\end{aligned}
		\end{equation}
		
Plugging \eqref{eq:july2-1-1} and \eqref{eq:july2-2} into \eqref{eq:july2-1}, we obtain
\begin{equation}\label{eq:july2-3}
		\begin{aligned}
&			\left|\Qhat_h^k(s,a,b) - Q_h^{\dagger,\nu^k}(s,a,b)\right|\\
\le & \min \bigg\{(1+\frac{1}{H})\alpha_{h+1}+ \frac{1}{H}[\Phat_h^k\Vt^k_{h+1}](s,a,b) + 
c_1\sqrt{\frac{H^2\iota}{\minone{N_h^k(s,a,b)} }}+\frac{c_1^2\paren{\alpha_{h+1}+1}H^2S\iota}{\minone{N_h^k(s,a,b)}},H\bigg\} \\
\stackrel{(i)}{\le} &  \min\left\{
\paren{(1+\frac{1}{H})\alpha_{h+1} + \frac{1}{H}}[\Phat_h^k\Vt^k_{h+1}](s,a,b) 
+  \beta_h^k(s,a,b) ,H\right\} \\
\stackrel{(ii)}{\le} &  \paren{(1+\frac{1}{H})\alpha_{h+1} + \frac{1}{H}}\Qt^k_{h}(s,a,b)  ,
\end{aligned}
		\end{equation}
where  
$(i)$ follows from the  definition of $\beta_h^k$,
and $(ii)$ follows from the definition of $\Qt^k_h$.
Therefore, by \eqref{eq:july2-3}, choosing $\alpha_h =[(1+\frac{1}{H})\alpha_{h+1} + \frac{1}{H}]$ suffices for the purpose of induction.

Now, let's prove the inequality for $V$ functions. \begin{equation}
		\begin{aligned}
		\qquad |(\Vhat_{h}^k-V^{\dagger,\nu^k}_{h}	)(s)| 
		&\overset{(i)}{=} |\max_{\mu\in\triangle_A}(\D_{\mu,\nu^k}\Qhat_{h}^k)(s)
		- \max_{\mu\in\triangle_A}(\D_{\mu,\nu^k}Q^{\dagger,\nu^k}_{h})(s) |  \\
		& \overset{(ii)}{\le} \max_{a,b}\left[ \alpha_{h}\Qt^k_{h}(s,a,b)\right]=  \alpha_{h} \Vt^k_h(s) ,
		\end{aligned}
		\end{equation}
		where $(i)$ follows from the definition of $\Vhat_{h}^k$ and $V^{\dagger,\nu^k}_{h}$, and $(ii)$ uses \eqref{eq:july2-3} and Lemma \ref{lem:lipschitz}.
\end{proof}

\begin{theorem}[Guarantee for UCB-VI from \citet{azar2017minimax}] \label{thm:UCB-VI}
For any $p \in (0,1]$, choose the exploration bonus $\beta_t$ in Algrothm~\ref{alg:VI-zero} as \eqref{eq:bonus}.
 Then, with probability at least $1-p$, 
$$ \sum_{k=1}^{K} \Vt_1^k(s_1) \le \cO(\sqrt{H^4SAK\iota}+H^3S^2A\iota^2).$$
\end{theorem}

\begin{proof}[Proof of Theorem \ref{thm:rewardfree-explore}]
Recall that $\text{out} = \arg\min_{k\in[K]} \Vt^k_h(s)$.
By Lemma \ref{lem:MG-closeQpiQhat} and Theorem \ref{thm:UCB-VI}, with probability at least $1-2p$, 
	\begin{equation}
		\begin{aligned}
V_h^{\dagger,\nu^\text{out}}(s)
- 	V_h^{\mu^\text{out},\dagger}(s)
\le  &
| V_h^{\dagger,\nu^\text{out}}(s) - \Vhat_h^{\text{out}}(s)|
+|\Vhat_h^{\text{out}}(s)    - 	V_h^{\mu^\text{out},\dagger}(s)|\\
\le & 8 \Vt^\text{out}_h(s) \le 
\cO(\sqrt{\frac{H^4SA\iota}{K}}+\frac{H^3S^2A\iota^2}{K}).
		\end{aligned}
\end{equation}
Rescaling $p$ completes the proof.
\end{proof}

\subsection{Vanilla Nash Value Iteration}\label{algorithm:Reward-free-planning}

Here, we provide one optional algorithm, Vanilla Nash VI, for computing the Nash equilibrium policy for a \emph{known} model. 
Its only difference from the value iteration algorithm for MDPs is that the maximum operator is replaced by the minimax operator in Line \ref{line:minimax}.
We remark that the Nash equilibrium for a two-player zero-sum game can be computed in polynomial time.

 \begin{algorithm}[h]
    \caption{Vanilla Nash Value Iteration}
    
 \begin{algorithmic}[1]
    \STATE {\bfseries Input}: model $\widehat{\cM}= (\Phat,\rhat)$.
     \STATE {\bfseries Initialize:} for all $(s, a, b)$,
    ${V}_{H+1}(s,a, b)\setto 0$.
  \FOR{step $h=H,H-1,\dots,1$}
    \FOR{$(s, a, b)\in\cS\times\cA\times \cB$}
    \STATE $Q_{h}(s, a, b)\setto 
    [\widehat{\P}_{h} {V}_{h+1}](s, a, b) + \rhat_h(s,a,b)$.
    \ENDFOR
    \FOR{$s \in \cS$}
    \STATE $(\muhat_h(\cdot\mid s),\nuhat_h(\cdot\mid s)) \setto \mbox{NASH-ZERO-SUM}(Q_{h}(s, \cdot, \cdot))$.\label{line:minimax}
    \STATE $V_h(s) \setto \muhat_h(\cdot \mid s)^\top Q_{h}(s, \cdot, \cdot) \nuhat_h(\cdot \mid s)$.
    \ENDFOR
    \ENDFOR
    \STATE {\bfseries Output} $(\muhat,\nuhat)\setto \{(\muhat_h(\cdot\mid s),\nuhat_h(\cdot\mid s))\}_{(h,s)\in[H]\times\cS}$.
 \end{algorithmic}
 \end{algorithm}

By recalling the definition of best responses in Appendix \ref{app:bellman}, one can directly see that the output policy $(\muhat,\nuhat)$ is a Nash equilibrium for  $\widehat{\cM}$.

\subsection{Proof of Theorem \ref{thm:MG-lowerbound}}
\label{appendix:pf_lowerbound}

\newcommand{\paction}{{\color{red} + }}
\newcommand{\naction}{{\color{blue} - }}

In this section, we first prove a  $\Theta(AB/\epsilon^2)$ lower bound for reward-free learning of matrix games, i.e., $S=H=1$, and then  generalize it to $\Theta(SABH^2/\epsilon^2)$ for reward-free learning of  Markov games.

\subsubsection{Reward-free learning of matrix games}
In the matrix game, let the max-player pick row and the min-player pick column.  
We consider the following family of Bernoulli matrix games:
\begin{equation}\label{eq:matrix_game}
	\Mfrak(\eps) = \left\{ \cM^{a^\star b^\star}\in \mathbb{R}^{A\times B} \mbox{ with }
	\cM_{ab}^{a^\star b^\star}= \frac12+(1-2\cdot\mathbf{1} \{a\neq a^\star \&b= b^\star \})\eps: \ \ (a^\star,b^\star)\in[A]\times[B]\right\},
\end{equation} 
where in matrix game $\Mcal^{a^\star b^\star}$, the reward is sampled from ${\rm Bernoulli}(\cM^{a^\star b^\star}_{ab})$ if the max-player picks the $a$'th row and the min-player picks the $b$'th column.

\begin{equation}\label{eq:figure}
	\begin{aligned}
\text{Min-player}\qquad \qquad  \qquad\qquad \ \qquad \\
	\text{Max-player}\quad	\begin{matrix}
	\text{action}& 1 & \dots & b^\star-1 & b^\star & b^\star +1 &\dots & B \\
			1 \quad & \paction & \dots & \paction & \naction & \paction & \dots & \paction \\
			\vdots \quad & \vdots & \ddots & \vdots & \vdots & \vdots & \ddots & \vdots\\
		a^\star-1 \quad	& \paction & \dots & \paction & \naction & \paction & \dots & \paction \\
		a^\star \quad	& \paction & \dots & \paction & \paction & \paction & \dots & \paction \\
		a^\star+1 \quad	& \paction & \dots & \paction & \naction & \paction & \dots & \paction \\
		\vdots \quad	& \vdots & \ddots & \vdots & \vdots & \vdots & \ddots & \vdots\\
		A \quad	& \paction & \dots & \paction & \naction & \paction & \dots & \paction 
		\end{matrix}\qquad \ \qquad
	\end{aligned}
\end{equation}

Above, we visualize $\cM^{a^\star b^\star}$ by using $\paction$ and $\naction$ to represent $1/2 + \eps$ and  $1/2 - \eps$, respectively.
It is direct to see that  the optimal (Nash equilibrium) policy for the max-player is always picking the $a^\star$'th row.
If the max-player picks the $a^\star$'th row with probability smaller than $2/3$, it is at least $\eps/10$ suboptimal.

\begin{lemma}\label{lem:MG-lower-lemma}
Consider an arbitrary \textbf{fixed} matrix game $\Mcal^{a^\star b^\star}$ from $\Mfrak(\eps)$ and $N\in \mathbb{N}$.
 If there exists an algorithm $\Acal$ such that when running on $\Mcal^{a^\star b^\star}$, it uses at most $N$ samples and outputs an $\eps/10$-optimal policy  with probability at least $p$, then there exists an algorithm $\hat{\Acal}$ that can identify $a^\star$  with probability at least $p$ using at most $N$ samples.
\end{lemma}
\begin{proof}
We simply define $\hat{\Acal}$ as running algorithm $\Acal$ and choosing the most played row by its output policy as the guess for $a^\star$. 
Because any $\eps/10$-optimal policy must play $a^\star$ with probability  at least $2/3$, we obtain 
 $\hat{\Acal}$ will correctly identify $a^\star$ with probability at least $p$.
\end{proof}

Lemma \ref{lem:MG-lower-lemma} directly implies that in order to prove the desired lower bound for reward-free matrix games:
\begin{claim}\label{claim:0}
for \textbf{any} reward-free algorithm $\Acal$ using at most $N=AB/(10^3\epsilon^2)$ samples, there exists a matrix game $\Mcal^{a^\star b^\star}$ in $\Mfrak(\eps)$ such that  when running $\Acal$ on $\Mcal^{a^\star b^\star}$, it will output a policy that is \textbf{at least} $\eps/10$ suboptimal for the max-player with probability at least $1/4$,
\end{claim}
 it suffices to prove the following claim:
\begin{claim}\label{claim:2}
for \textbf{any} reward-free algorithm $\hat{\Acal}$ using at most $N=AB/(10^3\epsilon^2)$ samples, there exists a matrix game $\Mcal^{a^\star b^\star}$ in $\Mfrak(\eps)$ such that  when running $\hat{\Acal}$ on $\Mcal$, it will fail to identify the optimal row  with probability at least $1/4$.
\end{claim}
\begin{remark}
	By Lemma \ref{lem:MG-lower-lemma}, the existence of such 'ideal' $\Acal$ implies the existence of an 'ideal' $\hat{\Acal}$, so to  prove such 'ideal' $\Acal$ does not exist (Claim \ref{claim:0}), it  suffices to show  such 'ideal' $\hat{\Acal}$ does not exist (Claim \ref{claim:2}).
\end{remark}

\begin{proof}[Proof of Claim \ref{claim:2}]
WLOG, we assume $\hat{\Acal}$ is deterministic.
Since $\hat{\Acal}$ is \emph{reward-free}, being deterministic means that in the exploration phase algorithm $\hat{\Acal}$ always pulls each arm $(a,b)$ for some \emph{fixed} $n(a,b)$ times (because there is no information revealed in this phase), and in the planning phase it outputs a guess for $a^\star$, which is a deterministic function of the reward information revealed.

We define the following notations:
\begin{itemize}
	\item $L$: the stochastic reward information revealed after algorithm $\hat{\Acal}$'s pulling.
	\item  $\Pr_\star$: the probability measure induced by picking $\mathcal{M}^{a^\star b^\star}$ uniformly at random from $\Mfrak(\epsilon)$ and then running $\hat{\Acal}$ on $\Mcal$.
	\item  $\Pr_{ab}$: the probability  measure induced by running $\hat{\Acal}$ on $\Mcal^{ab}$.
	\item   $\Pr_{0b}$: the probability  measure induced by running $\Acal$ on matrix $\Mcal^{0b}$, whose $b$'th column are all $(1/2-\eps)$'s and other columns are all 
       $(1/2+\eps)$'s.\footnote{We comment that matrix $\Mcal^{0b}$ does not belong to  $\Mfrak(\epsilon)$.}
       \item $f(L)$: the output of $\hat{\cA}$ based on the stochastic reward information $L$ revealed. More precisely, $f$ is function mapping from $[0,1]^N$ to $[A]$.
\end{itemize}
%Denote by $L$ the stochastic reward information revealed after algorithm $\hat{\Acal}$'s pulling.
%   Denote by $\Pr_\star$ the probability measure induced by picking $\mathcal{M}^{a^\star b^\star}$ uniformly at random from $\Mfrak(\epsilon)$ and then running $\hat{\Acal}$ on $\Mcal$.
%        Denote by $\Pr_{ab}$ the probability  measure induced by running $\hat{\Acal}$ on $\Mcal^{ab}$.
%   Denote by $\Pr_{0,b}$ the probability  measure induced by running $\Acal$ on matrix $\Mcal^{0b}$, whose $b$'th column are all $(1/2-\eps)$'s and other columns are all 
%       $(1/2+\eps)$'s. We comment that matrix $\Mcal^{0b}$ does not belong to  $\Mfrak(\epsilon)$.
       
We have 
\begin{equation}\label{eq:feb3}
\begin{aligned}
	\Pr_\star(f(L)\neq a^\star) 
	&\ge \frac{1}{AB} \sum_{a,b} \Pr_{0b}(f(L)\neq a) -\frac{1}{AB}\sum_{a,b} \| \Pr_{ab}(L=\cdot)- \Pr_{0b}(L=\cdot)\|_{1} \\
	& \ge 1- \frac{1}{A} - \frac{1}{AB}\sum_{a,b} \sqrt{2{\rm KL}(\Pr_{0b}\| \Pr_{ab})} \\
	& = 1 - \frac{1}{A} - \frac{1}{AB}\sum_{a,b} 
	\sqrt{2 n(a,b) [ (\frac12 - \eps)\log\frac{\frac12 - \eps}{\frac12 +\eps} + (\frac12 + \eps)\log\frac{\frac12 + \eps}{\frac12 -\eps}]}\\
	& \ge 1- \frac{1}{A} -\frac{10}{AB}\sum_{a,b} 
	\sqrt{n(a,b)\eps^2}\\
  & \ge 1- \frac{1}{A} - \sqrt{\frac{100N\eps^2}{AB}},
	\end{aligned}
\end{equation}
where the second inequality follows from 
$\sum_{a,b} \Pr_{0b}(f(L)\neq a)= \sum_{a,b} [1-\Pr_{0b}(f(L)= a)] = B(A-1)$ and Pinsker's inequality. 
Finally, plugging in $N=AB/(10^3\epsilon^2)$ concludes the proof.
\end{proof}
\begin{remark}
	The arguments in proving Claim \ref{claim:2} basically follows the same line in proving lower bounds for multi-arm bandits \citep[e.g., see][]{lattimore2018bandit}. 
\end{remark}

\subsubsection{Reward-free learning of Markov games}

Now let's generalize the $\Theta(AB/\epsilon^2)$ lower bound for reward-free learning of matrix games  
to $\Theta(SABH^2/\epsilon^2)$
for reward-free learning of Markov games.
We can follow the same way of generalizing a lower bound for multi-arm bandits to a lower bound for MDPs \citep[see e.g.,][]{dann2015sample,lattimore2018bandit,zhang2020task}.

\textbf{Proof sketch.}\quad Given the family of Bernoulli matrix games $\Mfrak(\cdot)$ defined in \eqref{eq:matrix_game},  we simply construct a Markov game to consist of $SH$ Bernoulli matrix games $\{M^{s,h}\}_{(s,h)\in[S]\times[H]}$ where $M^{s,h}$'s are sampled independently and identically from the uniform distribution over $\Mfrak(\epsilon/H)$. We will define the transition measure to be totally 'uniform at random' so that in each episode the agent will always reach each $M^{s,h}$ with probability $1/S$ (it is not $1/(SH)$ because in each episode the agent can visit $H$ matrix games). As a result, to guarantee $\epsilon$-optimality, the output policy must be at least $2\epsilon/H$-optimal for at least $SH/2$ different $M^{s,h}$'s. Recall Claim \ref{claim:0} shows learning a $2\epsilon/H$-optimal policy  for a single $M^{s,h}$ requires $\Omega(H^2AB/\epsilon^2)$ samples. Therefore, we need 
$\Omega(H^3AB/\epsilon^2)$ samples in total for learning $SH/2$ different $M^{s,h}$'s. 

Below, we provide a rigorous proof where the constants may be slightly different from those in our sketch. We remark that although the notations we will use are involved, they are only introduced for rigorousness and there is no real technical difficulty or new informative idea in the following proof.

\paragraph{Construction} 
We define the following family of Markov games:
\newcommand{\Jfrak}{\mathfrak{J}}
\newcommand{\Jcal}{\mathcal{J}}
\begin{equation}
	\Jfrak(\eps) := \left\{ \Jcal(\astar,\bstar): 
	\ (\astar,\bstar)\in [A]^{H\times S}\times[B]^{H\times S} \right\},
\end{equation} 
where MG $\Jcal(\astar,\bstar)$ is defined as
\begin{itemize}
	\item {\bf States and actions:} $\Jcal(\astar,\bstar)$ is a finite-horizon MG with $S+1$ states and of length $H+1$. 
	There is a fixed initial state $s_0$ in the first step, $S$ states $\{s_1,\ldots,s_S\}$ in the  remaining steps. The two players have $A$ and $B$ actions, respectively.
	
\item {\bf Rewards:}	  there is no reward in the first step. For the remaining steps $h\in\{2,\ldots,H+1\}$, if the agent takes action $(a,b)$ at state $s_i$ in the $h^{\rm th}$ step, it will receive a binary reward sampled from 
$$
{\rm Bernoulli} \Big(\frac12+(1-2\cdot\mathbf{1} \{a\neq \astar_{h-1,i} \&b= \bstar_{h-1,i} \})\frac{\eps}{H}\Big)
$$
	\item {\bf Transitions:}
	The agent always starts at a fixed initial state $s_0$ in the first step
	Regardless of the current state, actions and index of steps, the agent will always transit to one of $s_1,\ldots,s_S$ uniformly at random.
	\end{itemize}

It is direct to see that $\Jcal(\astar,\bstar)$ is a collection of $SH$ matrix games from $\Mfrak(\epsilon/H)$.
Therefore, the optimal policy for the max-player is to 
always pick action $\astar_{h-1,i}$ whenever it reaches state $s_i$ at step $h$ ($h\ge 2$). 
%In other words, $\astar_{h-1,i}$ is the unique optimal action for the step-state pair $(h,i)$.

%At a high level, in order to find an $\eps$-optimal policy for the above Markov game, we need to identify at least half of the entries of $\astar$. 
% Therefore, the number of episodes should be at least 
%$$
%\Theta\paren{\frac{AB}{(\epsilon/H)^2}} \times \frac{S}{2}
%= \Theta\paren{\frac{ABSH^2}{\epsilon^2}}.
%$$
%Below we provide a formal proof of this argument, which is almost the same as that for the setting of reward-free matrix games.

\paragraph{Formal proof of Theorem \ref{thm:MG-lowerbound}.}
Now, let's use $\Jfrak(\epsilon)$ to prove the $\Theta(SABH^2/\epsilon^2)$  lower bound (in terms of number of episodes)
for reward-free learning of Markov games.
We start by proving an analogue of Lemma \ref{lem:MG-lower-lemma}.
\begin{lemma}\label{lem:MarkovG-lower-lemma}
Consider an arbitrary fixed matrix game $\Jcal(\astar,\bstar)$ from $\Jfrak(\eps)$ and $N\in \mathbb{N}$. 
If an algorithm $\Acal$ can output a policy that is at most $\eps/10^3$ suboptimal with probability at least $p$ using at most $N$ samples, then there exists an algorithm $\hat{\Acal}$ that can correctly identify  at least $SH- \floor*{SH/500}$ entries of $\astar$ with probability at least $p$ using at most $N$ samples.
\end{lemma}
\begin{proof}
Denote by $\pi$ the output policy for the max player.
Denote by $Z$ the collection of $(h,i)$'s in $[H]\times[S]$ such that $\pi_{h+1}(\astar_{h,i} \mid s_{i})\le 2/3$. 

Observe that each time the max player picks a suboptimal action, it will incur an $2\epsilon/H$ suboptimality in expectation. 
As a result, if $\pi$ is at most $\eps/10^3$-suboptimal, we must have
$$
\frac{1}{S}\sum_{(h,i)\in Z}
(1-\pi_{h+1}(\astar_{h,i} \mid s_{i}))
  \times  \frac{2\epsilon}{H} \le \frac{\eps}{10^3},
$$
which implies $|Z| \le SH/500$, that is, for at most $\floor*{SH/500}$ different $(h,i)$'s, $\pi_{h+1}(\astar_{h,i} \mid s_{i})\le 2/3$. 
Therefore, we can simply pick $\argmax_{a}\pi_{h+1}(a \mid s_{i})$ as the guess for $\astar_{h,i}$. 
Since policy $\pi$ is at most $\eps/10^3$ suboptimal with probability at least $p$, 
our guess will be correct for at least $SH- \floor*{SH/500}$ different $(s,h)$ pairs  also with probability no smaller than $p$.
\end{proof}

Similar to the funtion of Lemma \ref{lem:MG-lower-lemma},
Lemma \ref{lem:MarkovG-lower-lemma} directly implies that in order to prove the desired lower bound for reward-free learning of Markov games:
\begin{claim}
for any reward-free algorithm $\Acal$ that interacts with the environment for at most $K=SABH^2/(10^4\epsilon^2)$ episodes, there exists $\Jcal\in\Jfrak(\eps)$ such that  when running $\Acal$ on $\Jcal$, it will output a policy that is at least $\eps/10^3$ suboptimal for the max-player with probability at least $1/4$,
\end{claim}
 it suffices to prove the following claim:
\begin{claim}\label{claim:3}
for any reward-free learning algorithm $\hat{\Acal}$ that interacts with the environment for at most $K=ABSH^2/(10^4\epsilon^2)$ episodes, there exists $\Jcal\in\Jfrak(\eps)$  such that  when running $\hat{\Acal}$ on $\Jcal$, it will fail to correctly identify $\astar_{h,i}$ for at least $\floor*{SH/500}+1$ different $(h,i)$ pairs with probability at least $1/4$.
\end{claim}

\begin{proof}[Proof of Claim \ref{claim:3}]
Denote by $\Pr_\star$ ($\E_\star$) the probability measure (expectation) induced by picking $\Jcal$ uniformly at random from $\Jfrak(\epsilon)$ and then running $\hat{\Acal}$ on $\Jcal$.
Denote by $\nw$ the number of $(s,h)$ pairs for which $\hat{\Acal}$ fails to identify the optimal actions.
Denote by $\er{h,i}$ the indicator function of the event that $\hat{\Acal}$ fails to identify the optimal action for $(h+1,i)$.

We prove by contradiction. 
Suppose for any $\Jcal\in\Jfrak(\epsilon)$, $\hat{\Acal}$ can  identify the optimal actions for at least $SH-\floor*{SH/500}$ different $(s,h)$ pairs with probability larger than $3/4$. Then we have
$$
\E_\star[\nw] \le \frac{1}{4}\times SH
+ \frac{3}{4}\times \floor*{\frac{SH}{500}} \le \frac{101SH}{400}.
$$
Since $\sum_{(h,i)\in[H]\times[S]} \E_\star[\er{h,i}] = \E_\star[\nw]$, there must exists $(h',i')\in[H]\times[S]$ such that $\E_\star[\er{h',i'}]\le 101/400$. 
However, in the following, we show that for every $(h,i)\in[H]\times[S]$, 
$\E_\star[\er{h,i}]\ge 1/3$.
%
%$\hat{\Acal}$ fails to identify the optimal action for the step-state pair $(h+1,i)$ with probability at least $1/3$, which directly implies 
%$\E_\star[\er{h,i}]\ge 1/3$ for all $(h,i)\in[H]\times[S]$.
As a result, we obtain a contraction and Claim \ref{claim:3} holds.
In the remainder of this section, we will prove for every $(h,i)\in[H]\times[S]$, $\E_\star[\er{h,i}]\ge 1/3$. 

WLOG, we assume $\hat{\Acal}$ is deterministic. 
%and it runs for exactly $K=ABSH^2/(10^3\epsilon^2)$ episodes.
 It suffices to consider an arbitrary \emph{fixed} $(h',i')$ pair and prove $\E_\star[\er{h',i'}]\ge 1/3$.

For technical reason, we introduce a new MG $\Jcal_{-(h',i')}(\astar,\bstar)$ as below:
\begin{itemize}
	\item {\bf States, actions and transitions:} same as $\Jcal(\astar,\bstar)$.	
\item {\bf Rewards:}	  there is no reward in the first step. For the remaining steps $h\in\{2,\ldots,H+1\}$, if the agent takes action $(a,b)$ at state $s_i$ in the $h^{\rm th}$ step such that $(h-1,i)\neq(h',i')$, it will receive a binary reward sampled from 
$$
{\rm Bernoulli} \Big(\frac12+(1-2\cdot\mathbf{1} \{a\neq \astar_{h-1,i} \&b= \bstar_{h-1,i} \})\frac{\eps}{H}\Big),
$$
otherwise it will receive 
a binary reward sampled from 
$$
{\rm Bernoulli} \Big(\frac12+(1-2\cdot\mathbf{1} \{b= \bstar_{h-1,i} \})\frac{\eps}{H}\Big).
$$
	\end{itemize}
	\begin{remark}
	Briefly speaking, $\Jcal_{-(h',i')}(\astar,\bstar)$ is the same as $\Jcal(\astar,\bstar)$ except the matrix game embedded at state $s_{i'}$ at step $h'+1$, where for the max player all its actions  are equivalently bad \footnote{A graphic illustration based on \eqref{eq:figure} would be replacing the column $[\naction,\ldots,\naction,\paction,\naction,\ldots,\naction]\trans$ with a column of all $\naction$'s in the matrix game embedded at state $s_{i'}$ at step $h'+1$.}.    
	Finally, we remark that $\Jcal_{-(h',i')}(\astar,\bstar)$ is \textbf{independent} of $\astar_{h',i'}$.
	\end{remark}

To proceed, we introduce (and recall) the following notations: 
\begin{itemize}
	\item $n(a,b)$: the number of times $\hat{\Acal}$ picks action $(a,b)$ at state $s_{i'}$ at step $(h'+1)$ within $K$ episode.
	\item $\Pr_{\astar\bstar}$ ($\E_{\astar\bstar}$): the probability measure (expectation) induced  
by running algorithm $\hat{\Acal}$ on $\Jcal(\astar,\bstar)$.
\item $\Pr_{\astar\bstar}^-$ ($\E_{\astar\bstar}^-$): the probability measure (expectation) induced  
by running algorithm $\hat{\Acal}$ on  $\Jcal_{-(h',i')}(\astar,\bstar)$ .
\item $\Pr_\star$ ($\E_\star$) the probability measure (expectation) induced by picking $\Jcal(\astar,\bstar)$ uniformly at random from $\Jfrak(\epsilon)$ and running $\hat{\Acal}$ on $\Jcal(\astar,\bstar)$.
\item $L$: the whole interaction trajectory of states, actions and rewards produced by algorithm $\hat{\Acal}$ within $K$ episodes.
\item $f(L)$: the guess of $\hat{\Acal}$ for $\astar_{h',i'}$ based on $L$.
\end{itemize}
%denote by $n(a,b)$ the number of times $\hat{\Acal}$ picks action $(a,b)$ at state $s_{i'}$ in the ${(h'+1)}^{\rm th}$ step;
%denote by $\Pr(\cdot\mid \Jcal(\astar,\bstar))$ ($\E[\cdot\mid \Jcal(\astar,\bstar)]$) the probability (expectation) induced  
%by running algorithm $\hat{\Acal}$ on $\Jcal(\astar,\bstar)$;
%similarly, we define $\Pr(\cdot\mid \Jcal_{-(h',i')}(\astar,\bstar))$ ($\E[\cdot\mid \Jcal_{-(h',i')}(\astar,\bstar)]$);
%also recall that we denote by $\Pr_\star$ ($\E_\star$) the probability (expectation) induced by picking $\Jcal(\astar,\bstar)$ uniformly at random from $\Jfrak(\epsilon)$ and running $\hat{\Acal}$ on $\Jcal$;
%denote by $L$ the whole trajectory of states, actions and rewards produced by algorithm $\hat{\Acal}$ in $N$ episodes; 
%with slight abuse of notation, denote by $\hat{\Acal}(L)$ the guess of $\hat{\Acal}$ for $\astar_{h',i'}$ based on $L$.
%denote by $\astar_{-}$ the collection of entries of $\astar$ except for $\astar_{h',i'}$ and similarly define 
%$\bstar_{-}$. 

The key observation here is that  for any $(a,b)\in[A]\times[B]$ and $(\astar,\bstar)\in [A]^{H\times S}\times[B]^{H\times S}$,
the expectation $\E_{\astar\bstar}^-[n(a,b) ]$ is independent of $(\astar,\bstar)$ because the agent does not receive any reward information  when interacting with the environment
 and the transition dynamics of different $ \Jcal_{-(h',i')}(\astar,\bstar)$'s are exactly the same. For simplicity of notation, we denote this expectation by $m(a,b)$. 
 Moreover, note that $\sum_{a,b}m(a,b) = K/S$ because the agent always reach state $s_{i'}$ in  step $(h'+1)$ with probability $1/S$ regardless of the actions taken.

By mimicking the arguments in \eqref{eq:feb3}, we have 
\begin{equation}
\begin{aligned}
&\qquad\E_\star[\er{h',i'}] = \Pr_\star(f(L)\neq \astar_{h',i'})\\
	&= \frac{1}{(AB)^{SH}} 
	\sum_{(\astar,\bstar)\in[A]^{H\times S}\times[B]^{H\times S}} \Pr_{\astar\bstar}(f(L)\neq \astar_{h',i'})\\
	&\ge \frac{1}{(AB)^{SH}} 
	\sum_{\astar,\bstar}
	\bigg(\Pr_{\astar\bstar}^-(f(L)\neq \astar_{h',i'})
	 -\left\|\Pr_{\astar\bstar}^-(L=\cdot)-
	 \Pr_{\astar\bstar}(L=\cdot)\right\|_1\bigg)\\
	&=1-\frac{1}{A}- \frac{1}{(AB)^{SH}} 
	\sum_{\astar,\bstar}
	\left\|\Pr_{\astar\bstar}^-(L=\cdot)-
	 \Pr_{\astar\bstar}(L=\cdot)\right\|\\
	 &\ge1-\frac{1}{A}- \frac{1}{(AB)^{SH}} 
	\sum_{\astar,\bstar}
	\sqrt{ 2{\rm KL}(\Pr_{\astar\bstar}^-(L=\cdot),
	 \Pr_{\astar\bstar}(L=\cdot))}\\
		& = 1-\frac{1}{A}- \frac{1}{(AB)^{SH}} 
	\sum_{\astar,\bstar} 
	\sqrt{2 m(\astar_{h',i'},\bstar_{h',i'}) \big[ (\frac12 - \frac{\eps}{H})\log\frac{\frac12 - \frac{\eps}{H}}{\frac12 +\frac{\eps}{H}} + (\frac12 + \frac{\eps}{H})\log\frac{\frac12 + \frac{\eps}{H}}{\frac12 -\frac{\eps}{H}}\big]}\\
	& = 1-\frac{1}{A}- \frac{1}{AB} 
	\sum_{(a,b)\in[A]\times[B]} 
	\sqrt{2 m(a,b) \big[ (\frac12 - \frac{\eps}{H})\log\frac{\frac12 - \frac{\eps}{H}}{\frac12 +\frac{\eps}{H}} + (\frac12 + \frac{\eps}{H})\log\frac{\frac12 + \frac{\eps}{H}}{\frac12 -\frac{\eps}{H}}\big]}\\
	& \ge 1- \frac{1}{A} -\frac{10}{AB}\sum_{a,b} 
	\sqrt{m(a,b)\frac{\eps^2}{H^2}}\\
  &
  \ge   1- \frac{1}{A} -\frac{10\epsilon}{ABH}\sqrt{AB\sum_{a,b} 
	m(a,b)}
  = 1- \frac{1}{A} - \sqrt{\frac{100K\eps^2}{SABH^2}}.
	\end{aligned}
\end{equation}
Plugging in $K=SABH^2/(10^4\epsilon^2)$ completes the proof.
\end{proof}

\section{Proof for Section \ref{section:multi-player-short} -- Multi-player General-sum Markov Games}

\subsection{Proof of Theorem~\ref{thm:Multi-Nash-VI}}
\label{sec:proof_general_ovi}

\subsubsection{NE Version}

In this section, we prove Theorem~\ref{thm:Multi-Nash-VI} (NE version). As before, we begin with proving the optimistic estimations are indeed upper bounds of the corresponding V-value and Q-value functions.

\begin{lemma}
    \label{lem:general-UCB}
    With probability $1-p$, for any $(s, \bm{a}, h, i)$ and $k\in[K]$:
    \begin{equation}
        \label{eq:general-Q-UCB}
    \up{Q}_{h,i}^{k}\left( s,\bm{a} \right) \ge Q_{h,i}^{\dagger,\pi _{-i}^{k}}\left( s,\bm{a} \right) ,   \,\,\,\,  \low{Q}_{h,i}^{k}\left( s,\bm{a} \right) \le Q_{h,i}^{\pi ^{k}}\left( s,\bm{a} \right),
    \end{equation}
    \begin{equation}
        \label{eq:general-V-UCB}
        \up{V}_{h,i}^{k}\left( s \right) \ge V_{h,i}^{\dagger,\pi _{-i}^{k}}\left( s \right), \,\,\,\,  \low{V}_{h,i}^{k}\left( s \right) \le V_{h,i}^{\pi ^{k}}\left( s \right).
    \end{equation}
 \end{lemma}

 \begin{proof}
For each fixed $k$, we prove this by induction
from $h = H + 1$ to $h = 1$. 
For the base case, we know at the $(H + 1)$-th step,$\up{V}_{H+1,i}^{k}\left( s \right) ={V}_{H+1,i}^{\dagger,\pi_{-i}^{k}}\left( s \right) =0$. Now, assume the inequality~\eqref{eq:general-V-UCB} holds for the $(h + 1)$-th step, for the $h$-th step, by the definition of $Q$-functions,
\begin{align*}
    \up{Q}_{h,i}^{k}\left( s,\bm{a} \right) -Q_{h,i}^{\dagger,\pi _{-i}^{k}}\left( s,\bm{a} \right) =&\left[ \Phat_{h}^{k}\up{V}_{h+1,i}^{k} \right] \left( s,\bm{a} \right) -\left[ \P_{h}V_{h+1,i}^{\dagger,\pi _{-i}^{k}} \right] \left( s,\bm{a} \right) +\beta _t
\\
=&\underset{\left( A \right)}{\underbrace{\Phat_{h}^{k}\left( \up{V}_{h+1,i}^{k}-V_{h+1,i}^{\dagger,\pi _{-i}^{k}} \right) \left( s,\bm{a} \right) }}+\underset{\left( B \right)}{\underbrace{\left( \Phat_{h}^{k}-\P_{h} \right)V_{h+1,i}^{\dagger,\pi _{-i}^{k}}\left( s,\bm{a} \right) }}+\beta_t .
\end{align*}

By induction hypothesis, for any $s'$, $\left(\up{V}_{h+1,i}^{k}-V_{h+1,i}^{\dagger,\pi _{-i}^{k}}\right)(s') \ge 0$, and thus $(A) \ge 0$. By uniform concentration \citep[e.g., Lemma 12 in][]{bai2020provable}, $(B) \le C\sqrt{SH^2\iota/N_h^k(s,\bm{a})} = \beta_t$. Putting everything together we have $\up{Q}_{h,i}^{k}\left( s,\bm{a} \right) -Q_{h,i}^{\dagger,\pi _{-i}^{k}}\left( s,\bm{a} \right) \ge 0$. The second inequality can be proved similarly.

Now assume inequality~\eqref{eq:general-Q-UCB} holds for the $h$-th step, by the definition of $V$-functions and Nash equilibrium,
$$
\up{V}_{h,i}^{k}\left( s \right) =\D_{\pi ^k}\up{Q}_{h,i}^{k}\left( s \right) =\underset{\mu}{\max}\D_{\mu \times \pi _{-i}^{k}}\up{Q}_{h,i}^{k}\left( s \right). 
$$

By Bellman equation,
$$
V_{h,i}^{\dagger,\pi _{-i}^{k}}\left( s \right) =\underset{\mu}{\max}\D_{\mu \times \pi _{-i}^{k}}Q_{h,i}^{\dagger,\pi _{-i}^{k}}\left( s \right) .
$$

Since by induction hypothesis, for any $(s,\bm{a})$, $\up{Q}_{h,i}^{k}\left( s,\bm{a} \right) \ge Q_{h,i}^{\dagger,\pi _{-i}^{k}}\left( s,\bm{a} \right)$. As a result, we also have $\up{V}_{h,i}^{k}\left( s \right) \ge V_{h,i}^{\dagger,\pi _{-i}^{k}}\left( s \right)$, which is exactly inequality~\eqref{eq:general-V-UCB} for the $h$-th step. The second inequality can be proved similarly.
 \end{proof}

\begin{proof}[Proof of Theorem~\ref{thm:Multi-Nash-VI}]
    Let us focus on the $i$-th player and ignore the subscript when there is no confusion. To bound 
    $$
    \max_{i}\left( V_{1,i}^{\dagger,\pi _{-i}^{k}}-V_{1,i}^{\pi ^k} \right) \left( s_{h}^{k} \right) \le \max_{i}\left( \up{V}_{1,i}^{k}-\low{V}_{1,i}^{k} \right) \left( s_{h}^{k} \right),
    $$
we notice the following propogation:
\begin{equation}
    \left\{
        \begin{aligned}
    &(\up{Q}^k_{h,i}-\low{Q}^k_{h,i})(s,\bm{a}) \le \Phat_h^k(\up{V}^k_{h+1,i}-\low{V}^k_{h+1,i})(s,\bm{a})+2\beta^k_h(s,\bm{a}),\\
    &(\up{V}_{h,i}-\low{V}_{h,i})(s) =[\D_{\pi_h}(\up{Q}^k_{h,i}-\low{Q}^k_{h,i})](s).
        \end{aligned}
        \right.
    \end{equation}

    We can define $\widetilde{Q}^k_{h}$ and $\widetilde{V}^k_{h}$ recursively by $\widetilde{V}^k_{H+1} = 0$ and 
    \begin{equation}
        \left\{
            \begin{aligned}
        &\widetilde{Q}^k_{h}(s,\bm{a}) = \Phat_h^k\widetilde{V}^k_{h+1}(s,\bm{a})+2\beta^k_h(s,\bm{a}),\\
        &\widetilde{V}^k_{h}(s) =[\D_{\pi_h}\widetilde{Q}^k_{h}](s).
            \end{aligned}
            \right.
        \end{equation}

    Then we can prove inductively that for any $k$, $h$, $s$ and $\bm{a}$ we have
    \begin{equation}
        \left\{
            \begin{aligned}
        &\max_{i}(\up{Q}^k_{h,i}-\low{Q}^k_{h,i})(s,\bm{a}) \le \widetilde{Q}^k_{h}(s,\bm{a}),\\
        &\max_{i}(\up{V}_{h,i}-\low{V}_{h,i})(s) \le \widetilde{V}^k_{h}(s).
            \end{aligned}
            \right.
        \end{equation}
    Thus we only need to bound $\sum_{k=1}^K\widetilde{V}^k_{1}(s)$. Define the shorthand notation
    \begin{equation}
        \left\{
            \begin{aligned}
        &\beta _{h}^{k} := \beta _{h}^{k}(s_h^k,\bm{a}_h^k),\\
        &\Delta_{h}^k := \widetilde{V}^k_{h}(s_h^k),\\
        &\zeta_{h}^k :=  [\D_{\pi ^k}\widetilde{Q}^k_{h}] \left( s_{h}^{k} \right) 
        - \widetilde{Q}^k_{h}(s_h^k,\bm{a}_h^k),\\
        & \xi_{h}^k := [\P_h\widetilde{V}^k_{h+1}](s_h^k,\bm{a}_h^k)
        - \Delta_{h+1}^k.
            \end{aligned}
            \right.
        \end{equation}
        We can check $\zeta_h^k$ and $\xi_h^k$ are martingale difference sequences. As a result,
\begin{align*}
    \Delta _{h}^{k}=&\D_{\pi ^k}\widetilde{Q}^k_{h} \left( s_{h}^{k} \right) 
\\
=&\zeta _{h}^{k}+\widetilde{Q}^k_{h} \left( s_{h}^{k},\bm{a}_h^k \right) 
\\
=&\zeta _{h}^{k}+2\beta _{h}^{k}+ [\Phat_{h}^{k}\widetilde{V}^k_{h+1}]  \left( s_{h}^{k},\bm{a}_h^k \right) 
\\
\le &\zeta _{h}^{k}+3\beta _{h}^{k}+[\P_{h}\widetilde{V}^k_{h+1}]  \left( s_{h}^{k},\bm{a}_h^k \right) 
\\
=&\zeta _{h}^{k}+3\beta _{h}^{k}+\xi _{h}^{k}+\Delta _{h+1}^{k}.
\end{align*}

Recursing this argument for $h \in [H]$ and taking the sum,
$$
\sum_{k=1}^K{\Delta _{1}^{k}}\le \sum_{k=1}^K{\left( \zeta _{h}^{k}+3\beta _{h}^{k}+\xi _{h}^{k} \right)}\le O\left( S\sqrt{H^3T\iota \prod_{i=1}^M{A_i}} \right).
$$
\end{proof}

\subsubsection{CCE Version}
\label{sec:proof-Multi-Nash-VI-CCE}
The proof is very similar to the NE version. Specifically, the only part that uses the properties of NE there is Lemma~\ref{lem:general-UCB}. We prove a counterpart here.

\begin{lemma}
    \label{lem:general-UCB-CCE}
    With probability $1-p$, for any $(s, \bm{a}, h, i)$ and $k\in[K]$:
    \begin{equation}
        \label{eq:general-Q-UCB-CCE}
    \up{Q}_{h,i}^{k}\left( s,\bm{a} \right) \ge Q_{h,i}^{\dagger,\pi _{-i}^{k}}\left( s,\bm{a} \right) ,   \,\,\,\,  \low{Q}_{h,i}^{k}\left( s,\bm{a} \right) \le Q_{h,i}^{\pi ^{k}}\left( s,\bm{a} \right),
    \end{equation}
    \begin{equation}
        \label{eq:general-V-UCB-CCE}
        \up{V}_{h,i}^{k}\left( s \right) \ge V_{h,i}^{\dagger,\pi _{-i}^{k}}\left( s \right), \,\,\,\,  \low{V}_{h,i}^{k}\left( s \right) \le V_{h,i}^{\pi ^{k}}\left( s \right).
    \end{equation}
 \end{lemma}

 \begin{proof}
For each fixed $k$, we prove this by induction
from $h = H + 1$ to $h = 1$. For the base case, we know at the $(H + 1)$-th step, $\up{V}_{H+1,i}^{k}\left( s \right) ={V}_{H+1,i}^{\dagger,\pi _{-i}^{k}}\left( s \right) =0$. Now, assume the inequality~\eqref{eq:general-V-UCB} holds for the $(h + 1)$-th step, for the $h$-th step, by the definition of $Q$-functions,
\begin{align*}
    \up{Q}_{h,i}^{k}\left( s,\bm{a} \right) -Q_{h,i}^{\dagger,\pi _{-i}^{k}}\left( s,\bm{a} \right) =&\left[ \Phat_{h}^{k}\up{V}_{h+1,i}^{k} \right] \left( s,\bm{a} \right) -\left[ \P_{h}V_{h+1,i}^{\dagger,\pi _{-i}^{k}} \right] \left( s,\bm{a} \right) +\beta _t
\\
=&\underset{\left( A \right)}{\underbrace{\Phat_{h}^{k}\left( \up{V}_{h+1,i}^{k}-V_{h+1,i}^{\dagger,\pi _{-i}^{k}} \right) \left( s,\bm{a} \right) }}+\underset{\left( B \right)}{\underbrace{\left( \Phat_{h}^{k}-\P_{h} \right)V_{h+1,i}^{\dagger,\pi _{-i}^{k}}\left( s,\bm{a} \right) }}+\beta _t.
\end{align*}

By induction hypothesis, for any $s'$, $\left(\up{V}_{h+1,i}^{k}-V_{h+1,i}^{\dagger,\pi _{-i}^{k}}\right)(s') \ge 0$, and thus $(A) \ge 0$. By uniform concentration, $(B) \le C\sqrt{SH^2\iota/N_h^k(s,\bm{a})} = \beta_t$. Putting everything together we have $\up{Q}_{h,i}^{k}\left( s,\bm{a} \right) -Q_{h,i}^{\dagger,\pi _{-i}^{k}}\left( s,\bm{a} \right) \ge 0$. The second inequality can be proved similarly.

Now assume inequality~\eqref{eq:general-Q-UCB-CCE} holds for the $h$-th step, by the definition of $V$-functions and CCE,
$$
\up{V}_{h,i}^{k}\left( s \right) =\D_{\pi ^k}\up{Q}_{h,i}^{k}\left( s \right) \ge\underset{\mu}{\max}\D_{\mu \times \pi _{-i}^{k}}\up{Q}_{h,i}^{k}\left( s \right). 
$$

By Bellman equation,
$$
V_{h,i}^{\dagger,\pi _{-i}^{k}}\left( s \right) =\underset{\mu}{\max}\D_{\mu \times \pi _{-i}^{k}}Q_{h,i}^{\dagger,\pi _{-i}^{k}}\left( s \right) .
$$

Since by induction hypothesis, for any $(s,\bm{a})$, $\up{Q}_{h,i}^{k}\left( s,\bm{a} \right) \ge Q_{h,i}^{\dagger,\pi _{-i}^{k}}\left( s,\bm{a} \right)$. As a result, we also have $\up{V}_{h,i}^{k}\left( s \right) \ge V_{h,i}^{\dagger,\pi _{-i}^{k}}\left( s \right)$, which is exactly inequality~\eqref{eq:general-V-UCB} for the $h$-th step. The second inequality can be proved similarly.
 \end{proof}

 \subsubsection{CE Version}
\label{sec:proof-Multi-Nash-VI-CE}

The proof is very similar to the NE version. Specifically, the only part that uses the properties of NE there is Lemma~\ref{lem:general-UCB}. We prove a counterpart here.

\begin{lemma}
    \label{lem:general-UCB-CE}
    With probability $1-p$, for any $(s, \bm{a}, h,i)$ and $k\in[K]$:
    \begin{equation}
        \label{eq:general-Q-UCB-CE}
    \up{Q}_{h,i}^{k}\left( s,\bm{a} \right) \ge \underset{\phi\in\Phi_i}{\max}Q_{h,i}^{\phi \circ \pi^{k}}\left( s,\bm{a} \right) ,   \,\,\,\,  \low{Q}_{h,i}^{k}\left( s,\bm{a} \right) \le Q_{h,i}^{\pi^{k}}\left( s,\bm{a} \right),
    \end{equation}
    \begin{equation}
        \label{eq:general-V-UCB-CE}
        \up{V}_{h,i}^{k}\left( s \right) \ge\underset{\phi\in\Phi_i}{\max}V_{h,i}^{\phi \circ \pi^{k}}\left( s \right), \,\,\,\,  \low{V}_{h,i}^{k}\left( s \right) \le V_{h,i}^{\pi ^{k}}\left( s \right).
    \end{equation}
 \end{lemma}

 \begin{proof}
For each fixed $k$, we prove this by induction
from $h = H + 1$ to $h = 1$. For the base case, we know at the $(H + 1)$-th step, $\up{V}_{H+1,i}^{k}\left( s \right) =\underset{\phi}{\max}{V}_{H+1,i}^{\phi \circ\pi^{k}}\left( s \right) =0$. Now, assume the inequality~\eqref{eq:general-V-UCB} holds for the $(h + 1)$-th step, for the $h$-th step, by definition of $Q$-functions,
\begin{align*}
    &\up{Q}_{h,i}^{k}\left( s,\bm{a} \right) -\underset{\phi}{\max}Q_{h,i}^{\phi \circ\pi^{k}}\left( s,\bm{a} \right) \\
    =&\left[ \Phat_{h}^{k}\up{V}_{h+1,i}^{k} \right] \left( s,\bm{a} \right) -\left[ \P_{h}\underset{\phi}{\max}V_{h+1,i}^{\phi \circ\pi^{k}} \right] \left( s,\bm{a} \right) +\beta _t
\\
=&\underset{\left( A \right)}{\underbrace{\Phat_{h}^{k}\left( \up{V}_{h+1,i}^{k}-\underset{\phi}{\max}V_{h+1,i}^{\phi \circ\pi^{k}} \right) \left( s,\bm{a} \right) }}+\underset{\left( B \right)}{\underbrace{\left( \Phat_{h}^{k}-\P_{h} \right)\underset{\phi}{\max}V_{h+1,i}^{\phi \circ\pi^{k}}\left( s,\bm{a} \right) }}+\beta _t.
\end{align*}

By induction hypothesis, for any $s'$, $\left(\up{V}_{h+1,i}^{k}-\underset{\phi}{\max}V_{h+1,i}^{\phi \circ\pi^{k}}\right)(s') \ge 0$, and thus $(A) \ge 0$. By uniform concentration, $(B) \le C\sqrt{SH^2\iota/N_h^k(s,\bm{a})} = \beta_t$. Putting everything together we have $\up{Q}_{h,i}^{k}\left( s,\bm{a} \right) -\underset{\phi}{\max}Q_{h,i}^{\phi \circ\pi^{k}}\left( s,\bm{a} \right) \ge 0$. The second inequality can be proved similarly.

Now assume inequality~\eqref{eq:general-Q-UCB-CE} holds for the $h$-th step, by the definition of $V$-functions and CE,
$$
\up{V}_{h,i}^{k}\left( s \right) =\D_{\pi ^k}\up{Q}_{h,i}^{k}\left( s \right) =\underset{\phi}{\max}\D_{\phi \circ \pi^{k}}\up{Q}_{h,i}^{k}\left( s \right). 
$$

By Bellman equation,
$$
\underset{\phi}{\max}V_{h,i}^{\phi \circ \pi^{k}}\left( s \right) =\underset{\phi}{\max}\D_{\phi \circ \pi^{k}}\underset{\phi'}{\max}Q_{h,i}^{\phi' \circ \pi^{k}}\left( s \right) .
$$

Since by induction hypothesis, for any $(s,\bm{a})$, $\up{Q}_{h,i}^{k}\left( s,\bm{a} \right) \ge\underset{\phi}{\max}Q_{h,i}^{\phi \circ\pi^{k}}\left( s,\bm{a} \right)$. As a result, we also have $\up{V}_{h,i}^{k}\left( s \right) \ge \underset{\phi}{\max}V_{h,i}^{\phi \circ \pi^{k}}\left( s \right)$, which is exactly inequality~\eqref{eq:general-V-UCB} for the $h$-th step. The second inequality can be proved similarly.
 \end{proof}

\subsection{Proof of Theorem \ref{thm:multi-rewardfree-explore}}\label{appendix:multi-rewardfree}

In this section, we prove each theorem
 for the single reward function case, i.e., $N=1$. 
The proof for the case of multiple reward functions ($N>1$) simply follows  from taking a union bound, that is, replacing the failure probability $p$ by $Np$.

\subsubsection{NE version}

Let $(\mu^k,\nu^k)$ be an arbitrary Nash-equilibrium policy of $\widehat{\cM}^k := (\Phat^k,\rhat^k)$, where $\Phat^k$ and $\rhat^k$ are our empirical estimate of the transition and the reward at the beginning of the $k$'th episode in Algorithm \ref{alg:multi-VI-zero}.
Given an arbitrary Nash equilibrium $\pi^k$ of $\widehat{\cM}^k$, we 
 use $\Qhat^k_{h,i}$ and $\Vhat^k_{h,i}$ to denote its value functions of the $i$'th player at the $h$'th step in $\widehat{\cM}^k$.
 % and use $Q^k_{h,i}$ and $V^k_{h,i}$ to denote its value functions in $\Mcal=(\Pr,r)$, where $\Pr$ is the true transition.

We prove the following two lemmas, which together imply the conclusion about Nash equilibriums in Theorem \ref{thm:multi-rewardfree-explore} 
 as in the proof of Theorem \ref{thm:rewardfree-explore}.
 
 \begin{lemma}\label{lem:general-rewardfree-1}
With probability $1-p$, for any $(h,s,\bm{a},i)$ and $k\in[K]$, we have
\begin{equation}\label{eq:aug15-11}
\left\{
\begin{aligned}
	&	|\Qhat_{h,i}^k(s,\bm{a}) - Q_{h,i}^{\pi^k}(s,\bm{a})| \le\Qt^k_h(s,\bm{a})  , \\
&	|\Vhat_{h,i}^k(s) - V_{h,i}^{\pi^k}(s)| \le\Vt^k_h(s)  .
\end{aligned}
\right.
\end{equation}
\end{lemma}

\begin{proof}
For each fixed $k$, we prove this by induction
from $h = H + 1$ to $h = 1$. For base case, we know at the $(H + 1)$-th step,$\Vhat_{H+1,i}^k =V_{H+1,i}^{\pi^{k}} = \Qhat_{H+1,i}^{k}=Q_{H+1,i}^{\pi^{k}}=0$. Now, assume the conclusion holds for the $(h + 1)$'th step, for the $h$'th step, by definition of $Q$- functions,
\begin{align*}
&  \left|  \Qhat_{h,i}^k\left( s,\bm{a} \right) -Q_{h,i}^{\pi^{k}}\left( s,\bm{a} \right) \right|\\
  \le &\left|\left[ \Phat_{h}^{k}\Vhat_{h+1,i}^k \right] \left( s,\bm{a} \right) -\left[ \P_{h}V_{h+1,i}^{\pi^{k}} \right] \left( s,\bm{a} \right) \right|+\left| r_h(s,a)-\rhat_h^k(s,a)\right|
\\
\le &\underset{\left( A \right)}{\underbrace{\left|\Phat_{h}^{k}\left( \Vhat_{h+1,i}^k-V_{h+1,i}^{\pi^{k}} \right) \left( s,\bm{a} \right) \right|}}+
\underset{\left( B \right)}{\underbrace{\left|\left( \Phat_{h}^{k}-\P_{h} \right)V_{h+1,i}^{\pi^{k}}\left( s,\bm{a} \right) \right|+\left| r_h(s,a)-\rhat_h^k(s,a)\right|}}
\end{align*}

By the induction hypothesis, 
$$(A) \le \Phat_{h}^{k} \left|\Vhat_{h+1,i}^k-V_{h+1,i}^{\pi^{k}}\right|(s,\bm{a})\le  
(\Phat_{h}^{k} \Vt^k_{h+1})(s,\bm{a})   .$$

By uniform concentration \citep[e.g., Lemma 12 in ][]{bai2020provable}, $(B) \le \sqrt{SH^2\iota/N_h^k(s,\bm{a})} = \beta_t$. Putting everything together we have
$$\left| Q_{h,i}^{\pi^{k}}\left( s,\bm{a} \right)  -\Qhat_{h,i}^k\left( s,\bm{a} \right)\right| \le \min\left\{(\Phat_{h}^{k} \Vt^k_{h+1})(s,\bm{a})   + \beta_t, H \right\} =\Qt^k_h(s,\bm{a})  ,$$
which proves the first inequality in \eqref{eq:aug15-11}. The inequality for $V$ functions follows directly by noting that the value functions are computed using the same policy $\pi^k$.
 \end{proof}

\begin{lemma}\label{lem:general-rewardfree-2}
With probability $1-p$, for any $(h,s,\bm{a},i,k)$, we have
\begin{equation}\label{eq:aug15-12}
\left\{
\begin{aligned}
	&	|\Qhat_{h,i}^k(s,\bm{a}) - Q_{h,i}^{\pi^k_{-i},\dagger}(s,\bm{a})| \le\Qt^k_h(s,\bm{a})  , \\
&	|\Vhat_{h,i}^k(s) - V_{h,i}^{\pi^k_{-i},\dagger}(s)| \le\Vt^k_h(s)  .
\end{aligned}
\right.
\end{equation}
\end{lemma}

\begin{proof}
For each fixed $k$, we prove this by induction
from $h = H + 1$ to $h = 1$. For the base case, we know at the $(H + 1)$-th step,$\Vhat_{H+1,i}^k =V_{H+1,i}^{\pi _{-i}^{k},\dagger} =\Qhat_{H+1,i}^{k}=Q_{H+1,i}^{\pi _{-i}^{k},\dagger}=0$. Now, assume the conclusion holds for the $(h + 1)$'th step, for the $h$'th step, by definition of the $Q$ functions,
\begin{align*}
  &\left|  \Qhat_{h,i}^k\left( s,\bm{a} \right) -Q_{h,i}^{\pi _{-i}^{k},\dagger}\left( s,\bm{a} \right) \right|\\
  =&\left|\left[ \Phat_{h}^{k}\Vhat_{h+1,i}^k \right] \left( s,\bm{a} \right) -\left[ \P_{h}V_{h+1,i}^{\pi _{-i}^{k},\dagger} \right] \left( s,\bm{a} \right) \right|+\left| r_h(s,a)-\rhat_h^k(s,a)\right|
\\
\le &\underset{\left( A \right)}{\underbrace{\left|\Phat_{h}^{k}\left( \Vhat_{h+1,i}^k-V_{h+1,i}^{\pi _{-i}^{k},\dagger} \right) \left( s,\bm{a} \right) \right|}}+
\underset{\left( B \right)}{\underbrace{\left|\left( \Phat_{h}^{k}-\P_{h} \right)V_{h+1,i}^{\pi _{-i}^{k},\dagger}\left( s,\bm{a} \right) \right|+\left| r_h(s,a)-\rhat_h^k(s,a)\right|}}
\end{align*}

By the induction hypothesis, 
$$(A) \le \Phat_{h}^{k} \left|\Vhat_{h+1,i}^k-V_{h+1,i}^{\pi _{-i}^{k},\dagger}\right|(s,\bm{a})\le  
(\Phat_{h}^{k} \Vt^k_{h+1})(s,\bm{a})   .$$

By uniform concentration, $(B) \le \sqrt{SH^2\iota/N_h^k(s,\bm{a})} = \beta_t$. Putting everything together we have
$$\left| Q_{h,i}^{\pi _{-i}^{k},\dagger}\left( s,\bm{a} \right)  -\Qhat_{h,i}^k\left( s,\bm{a} \right)\right| \le \min\left\{(\Phat_{h}^{k} \Vt^k_{h+1})(s,\bm{a})   + \beta_t, H \right\} =\Qt^k_h(s,\bm{a})  ,$$
which proves the first inequality in \eqref{eq:aug15-12}. It remains to show the inequality for $V$-functions also hold in the $h$'th step.

Since $\pi^k$ is a Nash-equilibrium policy, we have 
$$
\Vhat_{h,i}^k\left( s \right) = \underset{\mu}{\max}\D_{\mu \times \pi _{-i}^{k}}\Qhat_{h,i}^k\left( s \right).
$$

By Bellman equation,
$$
V_{h,i}^{\pi _{-i}^{k},\dagger}\left( s \right) =\underset{\mu}{\max}\D_{\mu \times \pi _{-i}^{k}}Q_{h,i}^{\pi _{-i}^{k},\dagger}\left( s \right) .
$$

Combining the two equations above, and utilizing the bound we just proved for $Q$ functions, 
we obtain
\begin{align*}
\left|\Vhat_{h,i}^{k}\left( s \right) - V_{h,i}^{\pi _{-i}^{k},\dagger}\left( s \right)\right|
\le 
\left| \underset{\mu}{\max}\D_{\mu \times \pi _{-i}^{k}}\Qhat_{h,i}^k\left( s \right) - \underset{\mu}{\max}\D_{\mu \times \pi _{-i}^{k}}Q_{h,i}^{\pi _{-i}^{k},\dagger}\left( s \right)\right|
\le \max_{\bm{a}}\Qt^k_h(s,\bm{a})  
= \Vt^k_h(s)  ,
\end{align*}
which completes the whole proof.
 \end{proof}

\subsubsection{CCE version}
\label{appendix:proof-multi-rewardfree-CCE}
 
 The proof is almost the same as that for Nash equilibriums. 
 We will reuse Lemma \ref{lem:general-rewardfree-1} and prove an analogue of Lemma \ref{lem:general-rewardfree-2}. 
 The conclusion for CCEs will follow directly by combining the two lemmas as in the proof of Theorem  \ref{thm:rewardfree-explore}.

 \begin{lemma}\label{lem:general-rewardfree-3}
With probability $1-p$, for any $(h,s,\bm{a},i)$ and $k\in[K]$, we have
\begin{equation}\label{eq:aug15-1}
\left\{
\begin{aligned}
	&	  Q_{h,i}^{\pi^k_{-i},\dagger}(s,\bm{a})-\Qhat_{h,i}^k(s,\bm{a}) \le\Qt^k_h(s,\bm{a})  , \\
&	V_{h,i}^{\pi^k_{-i},\dagger}(s)-\Vhat_{h,i}^k(s) \le\Vt^k_h(s)  .
\end{aligned}
\right.
\end{equation}
\end{lemma}

\begin{proof}
For each fixed $k$, we prove this by induction
from $h = H + 1$ to $h = 1$. For base case, we know at the $(H + 1)$-th step,$\Vhat_{H+1,i}^k =V_{H+1,i}^{\pi _{-i}^{k},\dagger} =\Qhat_{H+1,i}^{k}=Q_{H+1,i}^{\pi _{-i}^{k},\dagger}=0$. Now, assume the conclusion holds for the $(h + 1)$'th step, for the $h$'th step, by definition of $Q$ -functions,
\begin{align*}
&Q_{h,i}^{\pi _{-i}^{k},\dagger}\left( s,\bm{a} \right)-   \Qhat_{h,i}^k\left( s,\bm{a} \right) \\
\le &\left[ \P_{h}V_{h+1,i}^{\pi _{-i}^{k},\dagger} \right] \left( s,\bm{a} \right) 
 -\left[ \Phat_{h}^{k}\Vhat_{h+1,i}^k \right] \left( s,\bm{a} \right) +\left| r_h(s,a)-\rhat_h^k(s,a)\right|\\
\le &\underset{\left( A \right)}{\underbrace{\Phat_{h}^{k}\left( V_{h+1,i}^{\pi _{-i}^{k},\dagger}-\Vhat_{h+1,i}^k \right) \left( s,\bm{a} \right) }}+
\underset{\left( B \right)}{\underbrace{\left( \P_{h} -\Phat_{h}^{k}\right)V_{h+1,i}^{\pi _{-i}^{k},\dagger}\left( s,\bm{a} \right) +\left| r_h(s,a)-\rhat_h^k(s,a)\right|}}.
\end{align*}

By the induction hypothesis, 
$(A) \le 
(\Phat_{h}^{k} \Vt^k_{h+1})(s,\bm{a})$.

By uniform concentration, $(B) \le \sqrt{SH^2\iota/N_h^k(s,\bm{a})} = \beta_t$. Putting everything together we have
$$ Q_{h,i}^{\pi _{-i}^{k},\dagger}\left( s,\bm{a} \right)  -\Qhat_{h,i}^k\left( s,\bm{a} \right) \le \min\left\{(\Phat_{h}^{k} \Vt^k_{h+1})(s,\bm{a})   + \beta_t, H \right\} =\Qt^k_h(s,\bm{a})  ,$$
which proves the first inequality in \eqref{eq:aug15-1}. It remains to show the inequality for $V$-functions also hold in the $h$'th step.

Since $\pi^k$ is a CCE, we have 
$$
\Vhat_{h,i}^k\left( s \right) \ge \underset{\mu}{\max}\D_{\mu \times \pi _{-i}^{k}}\Qhat_{h,i}^k\left( s \right).
$$

Observe that $V_{h,i}^{\pi _{-i}^{k},\dagger}$ obeys the Bellman optimality equation, so we have
$$
V_{h,i}^{\pi _{-i}^{k},\dagger}\left( s \right) =\underset{\mu}{\max}\D_{\mu \times \pi _{-i}^{k}}Q_{h,i}^{\pi _{-i}^{k},\dagger}\left( s \right) .
$$

Combining the two equations above, and utilizing the bound we just proved for $Q$-functions, 
we obtain
\begin{align*}
V_{h,i}^{\pi _{-i}^{k},\dagger}\left( s \right)-\Vhat_{h,i}^{k}\left( s \right)
\le \underset{\mu}{\max}\D_{\mu \times \pi _{-i}^{k}}Q_{h,i}^{\pi _{-i}^{k},\dagger}\left( s \right)
-
\underset{\mu}{\max}\D_{\mu \times \pi _{-i}^{k}}\Qhat_{h,i}^k\left( s \right)\le \max_{\bm{a}}\Qt^k_h(s,\bm{a})  
= \Vt^k_h(s)  ,
\end{align*}
which completes the whole proof.
 \end{proof}

\subsubsection{CE version} 
 \label{appendix:proof-multi-rewardfree-CE}

  The proof is almost the same as that for Nash equilibriums. 
 We will reuse Lemma \ref{lem:general-rewardfree-1} and prove an analogue of Lemma \ref{lem:general-rewardfree-2}. 
 The conclusion for CEs will follow directly by combining the two lemmas as in the proof of Theorem  \ref{thm:rewardfree-explore}.
 
\begin{lemma}\label{lem:general-rewardfree-4}
With probability $1-p$, for any $(h,s,\bm{a},i)$, $k\in[K]$ and strategy modification $\phi$ for player $i$, we have
\begin{equation}\label{eq:aug15-1-1}
\left\{
\begin{aligned}
	&	  Q_{h,i}^{\phi\circ \pi^k}(s,\bm{a})-\Qhat_{h,i}^k(s,\bm{a}) \le\Qt^k_h(s,\bm{a})  , \\
&	V_{h,i}^{\phi\circ \pi^k}(s)-\Vhat_{h,i}^k(s) \le\Vt^k_h(s)  .
\end{aligned}
\right.
\end{equation}
\end{lemma}

\begin{proof}
For each fixed $k$, we prove this by induction
from $h = H + 1$ to $h = 1$. For the base case, we know at the $(H + 1)$-th step, $\Vhat_{H+1,i}^k =V_{H+1,i}^{\phi\circ \pi^k} =\Qhat_{H+1,i}^{k}=Q_{H+1,i}^{\phi\circ \pi^k}=0$. Now, assume the conclusion holds for the $(h + 1)$'th step, for the $h$'th step, following exactly the same argument as Lemma \ref{lem:general-rewardfree-3}, we can show
$$ Q_{h,i}^{\phi\circ \pi^k}\left( s,\bm{a} \right)  -\Qhat_{h,i}^k\left( s,\bm{a} \right) \le \min\left\{(\Phat_{h}^{k} \Vt^k_{h+1})(s,\bm{a})   + \beta_t, H \right\} =\Qt^k_h(s,\bm{a})  ,$$
which proves the first inequality in \eqref{eq:aug15-1-1}. It remains to show the inequality for $V$-functions also hold in the $h$'th step.

Since $\pi^k$ is a CE, we have 
$$
\Vhat_{h,i}^k\left( s \right) = \underset{\tilde{\phi}_{h,s}}{\max}\D_{\tilde{\phi}_{h,s} \circ\pi ^{k}}\Qhat_{h,i}^k\left( s \right),
$$
where the maximum is take over all possible functions from $\Acal_i$ to itself. 

Observe that $V_{h,i}^{\phi\circ \pi^k}$ obeys the Bellman optimality equation, so we have
$$
V_{h,i}^{\phi\circ \pi^k}\left( s \right) =\underset{\tilde{\phi}_{h,s}}{\max} \D_{\tilde{\phi}_{h,s} \circ\pi ^{k}}Q_{h,i}^{\phi\circ \pi^k}\left( s \right).
$$

Combining the two equations above, and utilizing the bound we just proved for $Q$-functions, 
we obtain
\begin{align*}
V_{h,i}^{\phi\circ \pi^k}\left( s \right)-\Vhat_{h,i}^{k}\left( s \right)
&= \underset{\tilde{\phi}_{h,s}}{\max} \D_{\tilde{\phi}_{h,s} \circ\pi ^{k}}Q_{h,i}^{\phi\circ \pi^k}\left( s \right)
-
\underset{\tilde{\phi}_{h,s}}{\max}\D_{\tilde{\phi}_{h,s} \circ\pi ^{k}}\Qhat_{h,i}^k\left( s \right)\\
&\le \max_{\bm{a}}\Qt^k_h(s,\bm{a})  
= \Vt^k_h(s)  ,
\end{align*}
which completes the whole proof.
 \end{proof}

\end{document}